 \documentclass[accepted]{uai2022} 

\usepackage[american]{babel}

\usepackage{natbib} 
\usepackage{mathtools} 
\usepackage{booktabs} 
\usepackage{tikz} 

\usepackage{microtype}
\usepackage{graphicx}
\usepackage{graphics}
\usepackage{subfigure}
\usepackage{booktabs} 
\usepackage{amsmath,amsthm,amssymb}
\usepackage{amsfonts}
\usepackage{dsfont}

\newtheorem{theorem}{Theorem}[section]
\newtheorem{lemma}[theorem]{Lemma}
\newtheorem{proposition}[theorem]{Proposition}

\newtheorem{definition}[theorem]{Definition}

\newtheorem{remark}[theorem]{Remark}
\newtheorem{assumption}[theorem]{Assumption}

\newcommand\norm[1]{\left\lVert#1\right\rVert}

\newcommand{\argmin}{\mathop{\mathrm{argmin}}}

\newcommand{\blue}[1]{\textcolor{blue}{#1}}

\newcommand\explaineq[2]{\stackrel{\mathclap{\normalfont\mbox{#1}}}{#2}}
\newcommand{\sumSA}{\sum_{h=0}^{\infty}\sum_{\substack{s,a \\ s\neq g}}}
\newcommand{\sumS}{\sum_{h=0}^{\infty}\sum_{\substack{s \\ s\neq g}}}

\def\ini{\mathrm{init}}

\def\hP{\widehat{P}}
\def\tP{\widetilde{P}}
\def\tB{\widetilde{B}}
\def\hT{\widehat{\mathcal{T}}}
\def\hV{\widehat{V}}
\def\Vb{V^{\bar{\pi}}}
\def\tT{\widetilde{\mathcal{T}}}

\def\E{\mathbb{E}}
\def\P{\mathbb{P}}

\def\Var{\mathrm{Var}}
\def\half{\frac{1}{2}}

\def\R{\mathbb{R}}

\def\cT{\mathcal{T}}

\usepackage{booktabs}       
\usepackage{amsfonts}       
\usepackage{nicefrac}       
\usepackage{microtype}      
\usepackage{xcolor}         

\usepackage{amsmath,amsthm,amssymb,bbm}
\usepackage{mathtools}
\usepackage{cases}
\usepackage{dsfont}
\usepackage{microtype}
\usepackage{tablefootnote}

\allowdisplaybreaks

\usepackage{subfigure}
\usepackage{algorithm,algorithmic}
\usepackage{color}
\usepackage{appendix}

\usepackage{bm}
\usepackage{subfigure}
\usepackage{algorithm,algorithmic}
\usepackage{color}
\usepackage{booktabs}       
\usepackage{appendix}
\usepackage{authblk}
\usepackage{comment}

\usepackage{hyperref}
\usepackage{xcolor}
\hypersetup{
	colorlinks,
	linkcolor={blue!50!black},
	citecolor={blue!50!black},
}
\colorlet{linkequation}{blue}

\definecolor{maroon}{RGB}{192,80,77}
\definecolor{mypink3}{cmyk}{0, 0.7808, 0.4429, 0.1412}

\title{Offline Stochastic Shortest Path: Learning, Evaluation and Towards Optimality}

\author[1,2]{Ming Yin\thanks{Equal contribution.}}
\author[3]{Wenjing Chen$^*$}
\author[4]{Mengdi Wang}
\author[1]{Yu-Xiang Wang}
\affil[1]{%
	Department of Computer Science\\
	UC Santa Barbara
}
\affil[2]{%
	Department of Statistics and Applied Probability\\
	UC Santa Barbara
}
\affil[3]{%
	Electrical and Electronics Engineering Department\\
	Texas A\&M University
}

\affil[4]{%
	Department of Electrical and Computer Engineering\\
	Princeton University
}
  
\begin{document}
\onecolumn 	
  	
\maketitle

\begin{abstract}
	Goal-oriented Reinforcement Learning, where the agent needs to reach the goal state while simultaneously minimizing the cost, has received significant attention in real-world applications. Its theoretical formulation, \emph{stochastic shortest path} (SSP), has been intensively researched in the online setting. Nevertheless, it remains understudied when such an online interaction is prohibited and only historical data is provided. In this paper, we consider the \emph{offline stochastic shortest path} problem when the state space and the action space are finite. We design the simple \emph{value iteration}-based algorithms for tackling both \emph{offline policy evaluation (OPE)} and \emph{offline policy learning} tasks. Notably, our analysis of these simple algorithms yields strong instance-dependent bounds which can imply worst-case bounds that are near-minimax optimal. We hope our study could help illuminate the fundamental statistical limits of the offline SSP problem and motivate further studies beyond the scope of current consideration.
	
\end{abstract}


\section{Introduction}\label{sec:introduction}

Goal-oriented reinforcement learning aims at entering a goal state while minimizing its expected cumulative cost. The interplay between the agent and the environment keeps continuing when the target/goal state is not reached and this causes trajectories to have variable lengths among different trials, which makes it different from (or arguably more challenging than) the finite-horizon RL. In particular, this setting naturally subsumes the \emph{infinite-horizon $\gamma$-discounted} case as one can make up a ``ghost'' goal state $g$ and set $1-\gamma$ probability to enter $g$ at each timestep for the latter. 

The goal-oriented RL covers many popular reinforcement learning tasks, such as navigation problems (e.g., Mujoco mazes), Atari games (\emph{e.g.} breakout) and Solving Rubik's cube \citep{akkaya2019solving} (also see Figure~\ref{fig:main} for more examples). Parallel to its empirical popularity, the theoretical formulation, \emph{stochastic shortest path} (SSP), has been studied from the control perspective (\emph{i.e.} with known transition) since \cite{bertsekas1991analysis}. Recently, there is a surge of studying SSP from the data-driven aspects (\emph{i.e.} with unknown transition) and existing literatures formulate SSP into the \emph{online reinforcement learning} framework \citep{tarbouriech2020no,rosenberg2020near,cohen2021minimax,chen2021finding,tarbouriech2021stochastic}. On the other hand, there exists no literature (to the best of our knowledge) formally study the \emph{offline} behavior of stochastic shortest path problem. 

In this paper, we study the offline counterpart of the stochastic shortest path (SSP) problem. Unlike its online version, we have no access to further explore new strategies (policies) and the data provided are historical trajectories. The goal is to come up with a cost-minimizing policy that can enter the goal state (\emph{policy learning}) or to evaluate the performance of a target policy (\emph{policy evaluation}).

\begin{figure}
	\centering     
	\subfigure{\label{fig:different_n}\includegraphics[width=30mm]{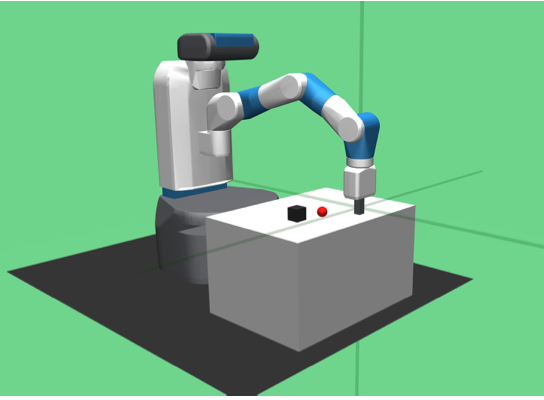}}
	\subfigure{\label{fig:different_H}\includegraphics[width=30mm]{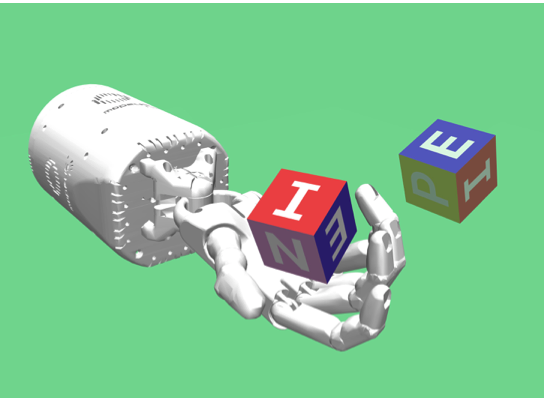}}
	\subfigure{\label{fig:different_H}\includegraphics[width=30mm]{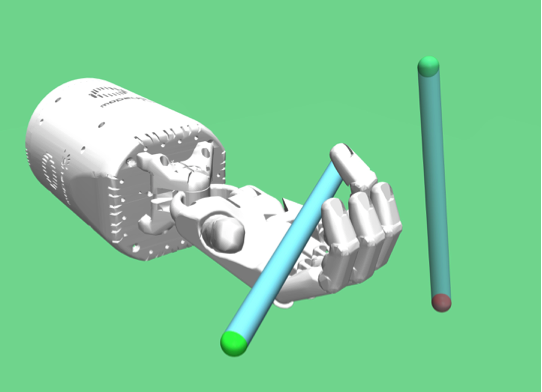}}
	\caption{Examples of Goal-oriented RL tasks in OpenAI-Gym environment. The robot can be asked to move-fetch to a position, orient a block or play with a pen.}
	\label{fig:main}
\end{figure}

\paragraph{Why should we study offline SSP?}Online SSP provides a suitable learning framework for goal-oriented tasks with cheap experiments (\emph{e.g.} Atari games). However, real-world applications usually have high-stake experiments which makes online interactions infeasible. For instance, in the application of logistic transportation, goods need to be delivered to their destinations. How to minimize the transportation cost should be decided/learned beforehand using the logged data. In the aircraft planning, changing flight routes instantaneously could be dangerous and designing routes based on history records is more appropriate for optimizing flying operation budget. In those scenarios, \emph{offline SSP} suffices for treating the practical challenges as it only learns from historical data.

\paragraph{Our contributions.} In this paper, we provide the first systematic study of the offline stochastic shortest path problem, and consider both \emph{offline policy evaluation} (OPE) and \emph{offline policy learning} tasks. As an initial attempt, we design the simple \emph{value iteration}-based algorithms to tackle the problems and obtain strong statistical guarantees. Concretely, our contributions are four folds.

\begin{itemize}
	\item For the offline policy evaluation task, we design VI-OPE algorithm (Algorithm~\ref{alg:VI_OPE}) under the coverage Assumption~\ref{assum:ope}. In particular, our algorithm is \emph{parameter-free} (requires no knowledge about $T^\pi$/$B^\pi$) and fully executed by the offline data. Theorem~\ref{thm:ope} provides the first statistical guarantee for offline SSP evaluation and nearly matches the statistical efficiency of its finite horizon counterpart (see discussion in Section~\ref{sec:discussion});
	
	\item For the offline learning task, we propose \emph{pessimism}-based algorithm PVI-SSP (Algorithm~\ref{alg:OPO}) under the Assumption~\ref{assum:opl} and \ref{assum:PC}. Our result (Theorem~\ref{thm:OPL}) has several merits: it is instance-dependent (as opposed to the worst-case guarantees in the existing online SSP works), enjoys faster $\widetilde{O}(1/n)$ convergence when the system is deterministic, and is also minimax-rate optimal. We believe Theorem~\ref{thm:OPL} is (in general) unimprovable for the current tabular setting.
	
	\item To understand the statistical limit of offline SSP, we prove the minimax lower bound $\Omega(B_\star\sqrt{\frac{SC^\star}{n}})$ (Theorem~\ref{thm:lower_main}) under the marginal coverage concentrability $\max_{s,a,s\neq g}\frac{d^{\pi^\star}(s,a)}{d^{\mu}(s,a)}\leq C^\star$. Our Theorem~\ref{thm:OPL} can match this rate (up to the logarithmic factor).
	
	\item Along the way for solving the problem, we highlight two new technical observations: Lemma~\ref{lem:T_pi} and Lemma~\ref{lem:HD}. The first one depicts the connection between the expected time $T^\pi$ and marginal coverage $d^\pi(s,a)$. As a result, we can express our result without using $T^\pi$ but the ratio-based quantity $\frac{d^\pi(s,a)}{d^\mu(s,a)}$, which matches the flavor of previous finite-horizon RL studies (also see Remark~\ref{remark:T_pi}). The second one is a general dependence improvement lemma that works with arbitrary policy $\pi$ and is the key for guaranteeing minimax optimal rate (also see Remark~\ref{remark:HD}). Both Lemmas are general and may be of independent interest.
	
\end{itemize}

\subsection{Related works.}

Stochastic shortest path itself is a broad topic and we are not aiming for the exhaustive review. Here we discuss two aspects that are most relevant to us.

\textbf{Online SSP.} Previous literatures intensively focus on the online aspect of SSP learning. Earlier works consider two types of problems: online shortest path routing problem with deterministic dynamics, which can be solved using the combinatorial bandit technique (\emph{e.g.} \cite{gyorgy2007line,talebi2017stochastic}); or SSP with stochastic transitions but adversarial feedbacks \citep{neu2012adversarial,zimin2013online,rosenberg2019online,chen2021finding,chen2021minimax}. Recently, \cite{tarbouriech2020no} starts investigating general online SSP learning problem and introduce the UC-SSP algorithm to first achieve the no-regret bound $\widetilde{O}(DS\sqrt{ADK})$.\footnote{Here the diameter of SSP is defined as $D:=\max_{s\in\mathcal{S}}\min_{\pi\in\Pi}T^\pi_s$ and by Lemma~2 of \cite{tarbouriech2020no} $B_\star:=\norm{V^\star}_\infty\leq c_{\max}D$. In this paper, we consider the dependence on $B_\star$ only since our $c_{\max}=1$ and this implies $B_\star\leq D$.} \cite{rosenberg2020near} improves this result to $\widetilde{O}(B_\star S\sqrt{AK})$ via a UCRL2-style algorithm with Bernstein-type bonus for exploration. Later, \cite{cohen2021minimax} eventually achieves the minimax rate $\widetilde{O}(B_\star\sqrt{SAK})$ by a reduction from SSP to finite-horizon MDP. However, the reduction technique requires the knowledge of $B_\star$ and $T_\star$. Most recently, \cite{tarbouriech2021stochastic} proposes EB-SSP which recovers the minimax rate but gets rid of the parameter knowledge (\emph{parameter-free}). When the parameters are known, their results can be \emph{horizon-free}.

Other than the general tabular SSP learning, there are also other threads, \emph{e.g.} Linear MDPs \citep{min2021learning,vial2021regret,chen2021improved} and posterior sampling \citep{jafarnia2021online}. Nevertheless, no analysis has been conducted for offline SSP yet.

\textbf{Offline tabular RL.} In the offline RL regime, there are fruitful results under different type of assumptions. \cite{yin2021near} first achieves the minimax rate $\widetilde{O}(\sqrt{H^3/nd_m})$ for non-stationary MDP with the strong uniform coverage assumption. \cite{ren2021nearly} improves the result to $\widetilde{O}(\sqrt{H^2/nd_m})$ for the stationary MDP setting. Later, \cite{rashidinejad2021bridging,xie2021policy,li2022settling} use the weaker single concentrability assumption and achieve the minimax rate $\widetilde{O}\sqrt{H^3SC^\star/n}$ (or $\widetilde{O}\sqrt{(1-\gamma)^{-3}SC^\star/n}$). Recently, this is further subsumed by the tighter instance-dependent result \citep{yin2021towards}. For offline policy evaluation (OPE) task, statistical efficiency has been achieved in tabular \citep{yin2020asymptotically}, linear \citep{duan2020minimax} and differentiable function approximation settings \citep{zhang2022off}.

\section{Problem setup }\label{sec:formulation}

\paragraph{Stochastic Shortest Path.} An SSP problem consists of a \emph{Markov decision process} (MDP) together with an initial state $s_{\mathrm{init}}$ and an extra goal state $g$ and it is denoted by the tuple $M:=\langle \mathcal{S},\mathcal{A},P,c,s_\ini,g\rangle$. In particular, we denote $\mathcal{S}':=\mathcal{S}\cup\{g\}$. Each state-action pair $(s,a)$ incurs a bounded random cost (within $[0,1]$) drawn i.i.d. from a distribution with expectation $c(s,a)$ and will transition to the next state $s'\in\mathcal{S}'$ according to the probability distribution $P(\cdot|s,a)$. Here $\sum_{s^{\prime} \in \mathcal{S}^{\prime}} P\left(s^{\prime} \mid s, a\right)=1$. The goal state $g$ is a termination state with absorbing property and has cost zero (\emph{i.e.} $P(g|g,a)=1,c(g,a)=0$ for all $a\in\mathcal{A}$). 

The optimal behavior of the agent is characterized by a stationary, deterministic and proper policy that minimizes the expected total cost of reaching the goal state from \emph{any} state $s$. A stationary policy $\pi:\mathcal{S}\rightarrow \Delta^\mathcal{A}$ is a mapping from state $s$ to a probability distribution over action space $\mathcal{A}$, here  $\Delta^\mathcal{A}$ is the set of probability distributions over $\mathcal{A}$. The definition of proper policy is defined as follows.

\begin{definition}[Proper policies]\label{def:proper}
	A policy $\pi$ is proper if playing $\pi$ reaches the goal state with probability $1$ when starting from any state. A policy is improper if it is not proper. Denote the set of proper policies as $\Pi_{\mathrm{prop}}$.
\end{definition}

\paragraph{Value and $Q$-functions in SSP.} Any policy $\pi$ induces a \emph{cost-to-go} value function $V^\pi:\mathcal{S}\mapsto [0,\infty]$ defined as 
\[
V^\pi(s):=\lim_{T\rightarrow\infty} \E^\pi\left[\sum_{t=0}^T c(s_t,a_t)|s_0=s\right],\;\;\forall s\in\mathcal{S}
\]
and the Q-function is defined as $\forall s,a\in\mathcal{S}\times\mathcal{A}$,
\[
Q^\pi(s,a):=\lim_{T\rightarrow\infty} \E^\pi\left[\sum_{t=0}^T c(s_t,a_t)|s_0=s,a_0=a\right],
\]
where the expectation is taking w.r.t. the random trajectory of states generated by executing $\pi$ and transitioning according to $P$. Also, we denote $T^\pi_s:=\lim_{T\rightarrow\infty}\E[\sum_{t=0}^T\mathbf{1}[s_t\neq g]|s_0=s]=\E[\sum_{t=0}^\infty\mathbf{1}[s_t\neq g]|s_0=s]$ to be the expected time that $\pi$ takes to enter $g$ starting from $s$. By Definition~\ref{def:proper}, $\pi$ is proper if $T^\pi_s<\infty$ for all $s$, and improper if $T^\pi_s=\infty$ for some state $s$. Moreover, by definition it follows $V^\pi(g)=Q^\pi(s,a)=0$ for all $\pi$ and action $a$. The next proposition is the Bellman equation for the SSP problem.

\begin{proposition}[Bellman equations for SSP problem \citep{bertsekas1991analysis}]\label{prop:Bellman}
	Suppose there exists at least one proper policy and that for every improper policy $\pi'$ there exists at least one state $s\in\mathcal{S}$ such that $V^{\pi'}(s)=+\infty$. Then the optimal policy $\pi^\star$ is stationary, deterministic, and proper. Moreover, $V^\star=V^{\pi^\star}$ is the unique solution of the equation $V^\star = \mathcal{L}V^\star$, where 
	\[
	\mathcal{L} V(s):=\min _{a \in \mathcal{A}}\left\{c(s, a)+P_{s, a} V\right\}\quad \forall V\in\R^{S'}.
	\]
	Similarly, for a proper policy $\pi$, $V^\pi$ is the unique solution of $V^\pi=\mathcal{L}^\pi V^\pi$ with $\mathcal{L}^\pi V(s):= \E_{a\sim\pi(\cdot|s)}[c(s, a)+P_{s, a} V],\;\forall V\in\R^{S'}$. Furthermore, it holds
	\begin{equation}\label{eqn:bellman}
	\begin{aligned}
	Q^{\star}(s, a)=c(s, a)+P_{s, a} V^{\star}, \; &V^{\star}(s)=\min_{a \in \mathcal{A}} Q^{\star}(s, a), \\
	Q^{\pi}(s, a)=c(s, a)+P_{s, a} V^{\pi}, \; &V^{\pi}(s)=\E_{a\sim\pi(\cdot|s)} [Q^{\pi}(s, a)].
	\end{aligned}
	\end{equation}
\end{proposition}

We use $T^\star_s$ to denote the expected arriving time when coupled with the optimal policy $\pi^\star$ and the proof of Proposition~\ref{prop:Bellman} can be found in Appendix~\ref{sec:gen_Bellman}.

\paragraph{The Offline SSP task.} The goal of offline SSP is to reach the goal state but also minimize the cost using offline data $\mathcal{D}:=\{(s^{(i)}_0,a^{(i)}_0,c^{(i)}_0,s^{(i)}_1,\ldots,s^{(i)}_{T_i})\}_{i=1,\ldots,n}$, which is collected by a proper (possibly stochastic) behavior policy $\mu$. The optimal policy is a proper policy $\pi^\star$ (the existence of $\pi^\star$ is guaranteed by the Proposition~\ref{prop:Bellman}) which minimizes the value function for all states, \emph{i.e.}, 
\begin{align}
\pi^\star(s)=\arg\min_{\pi\in\Pi_{\text{prop}}}V^{\pi}(s).
\end{align}
The final learning objective is to come up with a (proper) policy $\widehat{\pi}$ using $\mathcal{D}$ such that the suboptimality gap $V^{\widehat{\pi}}(s_{\text{init}})-V^\star(s_{\text{init}})<\epsilon$ for a given accuracy $\epsilon>0$.

\paragraph{Some Notations.} In the paper, we may abuse the notation $V^\star$ with $V^{\pi^\star}$, and define $B_\star:=\max_s\left\{V^\star(s)\right\}$. In addition, we denote $\xi_h^{\pi}(s,a)$ to be the marginal state-action occupancy at time step $h$ under the policy $\pi$ and $\xi_h^{\pi}(s)$ the marginal state occupancy at time $h$. Furthermore, we define the \emph{marginal coverage} $d^\pi$ as (given the initial state is $s_{\text{init}}$):
\begin{equation}\label{def:mar_cov}
d^\pi(s,a):=\sum_{h=0}^\infty \xi^\pi_h(s,a),\;\;\forall s,a\in\mathcal{S}\times\mathcal{A}.
\end{equation}
\begin{remark} 
	The notation of marginal coverage mirrors the marginal state-action occupancy in the infinite horizon $\gamma$-discounted setting but without normalization. Therefore, it is likely that $d^\pi(s,a)>1$ (or even $\infty$) for the offline SSP problem. Nevertheless, the key Lemma~\ref{lem:T_pi} guarantees $d^\pi(s,a)$ is finite when $\pi$ is a proper policy. This feature helps formalize the following assumptions in offline SSP.
\end{remark}

\subsection{Assumptions}
Offline learning/evaluation in SSP is impossible without assumptions. We now present three required assumptions.

\begin{assumption}[offline policy evaluation (OPE)]\label{assum:ope}
	We assume both the target policy $\pi$ and behavior policy $\mu$ are proper. In this case, we have $
	\Pi_{\text{prop}}\neq\emptyset$. Moreover, we assume behavior policy $\mu$ can cover the exploration (state-action) space of $\pi$, i.e. $\forall s,a\in\mathcal{S}\times\mathcal{A}$ s.t. $d^\pi_{\bar{s}}(s,a):=\sum_{h=0}^\infty \xi^\pi_{h,\bar{s}}(s,a)>0$, it implies $d^\mu_{\bar{s}}(s,a):=\sum_{h=0}^\infty \xi^\mu_{h,\bar{s}}(s,a)>0$, where $d^\pi_{\bar{s}}(s,a)$ is the marginal coverage and $\xi^\pi_{h,\bar{s}}(s,a)$ the marginal state-action occupancy given the initial state $\bar{s}$. In particular, when $\bar{s}=s_{\text{init}}$, we suppress the subscript and use $d^\pi,\xi^\pi_h$ only.
\end{assumption}

There are two remarks that are in order.

Assumption~\ref{assum:ope} requires that the behavior policy $\mu$ can explore all the state-action locations that are explored by $\pi$ and this mirrors the necessary OPE assumption made in the standard RL setting (\emph{e.g.} \cite{thomas2016data,yin2020asymptotically,uehara2020minimax}). Otherwise, policy evaluation for SSP would incur constant suboptimality gap even when \emph{infinite many} trajectories are collected. 

Moreover, instead of making assumption only on $d^\mu_{s_\text{init}}$, \ref{assum:ope} assumes $\mu$ can cover $\pi$ when starting from any state $\bar{s}$ (\emph{i.e.} $d^\pi_{\bar{s}}>0$ implies $d^\mu_{\bar{s}}>0$ for all $\bar{s}$). This extra requirement is mild since, by Definition~\ref{def:proper}, a proper policy can reach goal state $g$ with probability $1$ when starting from any state $\bar{s}$. Similarly, we need the assumptions for offline learning tasks.

\begin{assumption}[offline policy learning]\label{assum:opl}
	We assume there exists a deterministic proper policy and the behavior policy $\mu$ is (possible random) proper. Next, by Proposition~\ref{prop:Bellman}, we know there exists a deterministic optimal proper policy $\pi^\star$. We assume behavior policy $\mu$ can cover the exploration (state-action) space of $\pi^\star$, i.e. $\forall s,a\in\mathcal{S}\times\mathcal{A}$ s.t. $d^{\pi^\star}_{\bar{s}}(s,a):=\sum_{h=0}^\infty \xi^{\pi^\star}_{h,\bar{s}}(s,a)>0$, it implies $d^\mu_{\bar{s}}(s,a):=\sum_{h=0}^\infty \xi^\mu_{h,\bar{s}}(s,a)>0$, where $d^{\pi^\star}_{\bar{s}}(s,a)$ and $\xi^{\pi^\star}_{h,\bar{s}}(s,a)$ is the same notion used in Assumption~\ref{assum:ope}. In particular, when $\bar{s}=s_{\text{init}}$, we suppress the subscript and use $d^{\pi^\star},\xi^{\pi^\star}_h$ only.
\end{assumption}

\ref{assum:opl} provides the offline learning version of Assumption~\ref{assum:ope}. It echos its offline RL counterpart assumed in \cite{liu2019off,yin2021towards,uehara2021pessimistic}. Similar to the offline RL setting (\emph{e.g.} see \cite{yin2021towards} for detailed explanations), this assumption is also required for the tabular offline SSP problem.

\begin{assumption}[Positive cost \citep{rosenberg2020near}]\label{assum:PC}
	There exists $c_{\min}>0$ such that $c(s,a)\geq c_{\min}$ for every $(s,a)\in\mathcal{S}\times\mathcal{A}$.\footnote{Note this assumption only holds for $(s,a)\in\mathcal{S}\times\mathcal{A}$. For goal state $g$, it always has $c(g,a)=0$ for all $a\in\mathcal{A}$.}
\end{assumption}

This assumption guarantees there is no \emph{``free-cost''} state. With \ref{assum:PC} it holds that any policy does not reach the goal state has infinite cost, and this certifies the condition in Proposition~\ref{prop:Bellman} that for every improper policy $\pi'$ there exists at least one state $s$ such that $V^{\pi'}(s)=+\infty$. When $c_{\min}$ is $0$, a simple workaround is to solve a perturbed SSP instance with all observed costs clipped to $\epsilon$ if they are below some $\epsilon>0$, and in this case $c_{\min}=\epsilon>0$. This will cause only an additive term of order $O(\epsilon)$ (see \cite{tarbouriech2020no} for online SSP). Therefore, as the first attempt for offline SSP problem, we stick to this assumption throughout the paper. Last but not least, Assumption~\ref{assum:PC} is only used in offline learning problem (Section~\ref{sec:OPL}) and our OPE analysis (Section~\ref{sec:discussion}) can work well with zero cost.

\section{Off-Policy Evaluation in SSP} \label{sec:discussion}

\begin{algorithm}
	\caption{VI-OPE (Value Iteration for OPE problem of Stochastic Shortest Path)}
	\label{alg:VI_OPE}
	\begin{algorithmic}[1]
		\STATE {\bfseries Input:} $\epsilon_{\text{OPE}}$, $\mathcal{D}:=\{(s^{(i)}_1,a^{(i)}_1,c^{(i)}_1,s^{(i)}_2,\ldots,s^{(i)}_{T_i})\}_{i=1}^n$.
		\FOR{$(s,a,s')\in \mathcal{S}\times\mathcal{A}\times\mathcal{S}'$}
		\STATE Set $n(s,a) =  \sum_{i=1}^n\sum_{j=1}^{T_i}\mathbb{I}(s_j^{(i)}=s\text{, }a_j^{(i)}=a)$.
		\IF{$n(s,a)>0$}
		\STATE Calculate $\widehat{c}(s,a)=\frac{\sum_{i=1}^n\sum_{j=1}^{T_i}\mathbb{I}(s_j^{(i)}=s\text{, }a_j^{(i)}=a)c^{(i)}_j}{n(s,a)}$
		\STATE  $\widehat{P}(s'|s,a)=\frac{\sum_{i=1}^n\sum_{j=1}^{T_i}\mathbb{I}(s_j^{(i)}=s\text{, }a_j^{(i)}=a\text{, }s_{j+1}^{(i)}=s')}{n(s,a)}$, 
		\ELSE
		\STATE $\widehat{c}(s,a)\leftarrow c_\text{min}$, $\widehat{P}(s'|s,a)\leftarrow \mathbb{I}(s'=g)$.
		\ENDIF
		\STATE \blue{$\diamond$ Perturb the estimated transition kernel} 
		\STATE $\widetilde{P}(s'|s,a)=\frac{n(s,a)}{n(s,a)+1}\widehat{P}(s'|s,a)+\frac{\mathbb{I}[s'=g]}{n(s,a)+1}$
		\ENDFOR
		\STATE \blue{$\diamond$ Value Iteration for SSP problem} 
		\STATE {\bfseries Initialize:} ${V}^{(-1)}(\cdot)\leftarrow -\infty$, ${V}^{(0)}(\cdot)\leftarrow\mathbf{0}$, $i=0$.
		\WHILE{$\norm{V^{(i)}-V^{(i-1)}}_\infty>\epsilon_{\text{OPE}}$}
		\FOR{$(s,a)\in \mathcal{S}\times\mathcal{A}$}
		\STATE $Q^{(i+1)}(s,a)=\widehat{c}(s,a)+\widetilde{P}_{s,a}V^{(i)}$
		\STATE  $V^{(i+1)}(s)=\langle \pi(\cdot|s), Q^{(i+1)}(s,\cdot)\rangle$
		\STATE $i\leftarrow i+1$
		\ENDFOR
		\ENDWHILE
		\STATE \textbf{Output}: $V^{(i)}(\cdot)\in\R^S$, $V^{(i)}(s_{\text{init}})$.
	\end{algorithmic}
\end{algorithm}

In this section, we assume that Assumption~\ref{assum:ope} holds and consider \emph{offline policy evaluation} (OPE) for the \emph{stochastic shortest path} (SSP) problem. Our algorithmic design follows the natural idea of \emph{approximate value iteration} \citep{munos2005error} and is named \textbf{VI-OPE} (Algorithm~\ref{alg:VI_OPE}). Specifically, VI-OPE approximates \eqref{eqn:bellman} by solving the fixed point solution of the empirical Bellman equation associated with estimated cost $\widehat{c}$ and transition $\widetilde{P}$. One highlight is that we construct $\widetilde{P}$ to be the skewed version of the vanilla empirical estimation $\widehat{P}$ by injecting $\frac{1}{n(s,a)+1}$ probability to state $g$ (Line~11 of Algorithm~\ref{alg:VI_OPE}).\footnote{This treatment is also used in \cite{tarbouriech2021stochastic}.} By such a shift, the empirical Bellman operator $\widehat{\mathcal{T}}^{\pi}(\cdot):=\widehat{c}^\pi+\widetilde{P}^\pi(\cdot) $ becomes a contraction with rate $\rho:=\max_{\substack{s,a \\ s\neq g}}(\frac{n_{s,a}}{n_{s,a}+1})<1$ (see Lemma~\ref{lem:contraction} for details). Hence, \emph{contraction mapping theorem} \citep{diaz1968fixed} guarantees the loop (Line15-21) will end after $O(\log(\epsilon_{\text{OPE}})/\log(\rho))$ iterations for any $\epsilon_{\text{OPE}}>0$. We have the following main result for VI-OPE, whose proof can be found in Appendix~\ref{sec:proof_ope}.

\begin{theorem}[Offline Policy Evaluation in SSP]
	\label{thm:ope}
	Denote $d_m:=\min\{\sum_{h=0}^\infty \xi^\mu_h(s,a):s.t. \sum_{h=0}^\infty \xi^\mu_h(s,a)>0\}$, and $T^\pi_s$ to be the expected time to hit $g$ when starting from $s$. Define $\bar{T}^\pi=\max_{\bar{s}\in\mathcal{S}}T^\pi_{\bar{s}}$ and the quantity ${T}_{\max}=\max_{i\in[n]} T_i$. Then when $n\geq \max\{\frac{49 S\iota}{9d_m}, 64(\bar{T}^\pi)^2\frac{S\iota}{d_m},O(\iota/d_m),O( {T}^2_{\max}\log(SA/\delta)/d_m^2)\}$, we have with probability $1-\delta$, the output of Algorithm~\ref{alg:VI_OPE} satisfies ($\iota=O(\log(SA/\delta))$)
	\begin{align*}
	&|V^{(i)}(s_\mathrm{init})-V^\pi(s_\mathrm{init})|
	 \leq 4\sum_{s,a,s\neq g} d^\pi(s,a)\sqrt{\frac{2\Var_{P_{s,a}}[V^\pi+c]\iota}{n\cdot d^\mu(s,a)}}
	+\widetilde{O}(\frac{1}{n})+\frac{\epsilon_{\mathrm{OPE}}}{1-\rho}.
	\end{align*}
	where the $\widetilde{O}$ absorbs Polylog term and higher order terms.
\end{theorem}

\textbf{On statistical efficiency.} First of all, when VI-OPE converges exactly (\emph{i.e.} $\epsilon_\text{OPE}=0$), the output $\widehat{V}^\pi:=\lim_{i\rightarrow\infty}V^{(i)}$ possesses no optimization error (\emph{i.e.} $\epsilon_{\text{OPE}}/(1-\rho)=0$) and the (non-squared) statistical rate achieved by VI-OPE is dominated by $O(\sum_{s,a} d^\pi(s,a)\sqrt{\frac{\Var_{P_{s,a}}[V^\pi+c]\iota}{n\cdot d^\mu(s,a)}})$. As a comparison, for the well-studied finite-horizon tabular MDP problem, the statistical limit $O(\sqrt{\sum_{h=1}^H\sum_{s,a} d^\pi_h(s,a)^2\frac{\Var_{P_h}[V^\pi_{h+1}+c]}{n\cdot d^\mu_h(s,a)}]})$ has been achieved by \cite{yin2020asymptotically,duan2020minimax,kallus2020double} which matches the previous proven lower bound \citep{jiang2016doubly}. Therefore, it is natural to conjecture that the statistical lower bound for SSP-OPE problem has the rate $O(\sqrt{\sum_{s,a} d^\pi(s,a)^2\frac{\Var_{P_{s,a}}[V^\pi+c]}{n\cdot d^\mu(s,a)}})$. Our simple VI-OPE algorithm nearly matches this conjectured lower bound and only has the expectation outside of the square root. How to obtain the Carmer-Rao-style lower bound for SSP OPE problem and how to close the gap are beyond this initial attempt. We leave these as the future works.

\textbf{Parameter-free.} Different from the standard MDPs (\emph{e.g.} finite-horizon, discounted), the SSP formulation generally has variable horizon length which yields no explicit bound on $\norm{V^\pi}_\infty$. Consequently, most of the previous literature that study SSP problem requires the knowledge of expected running time $T^\pi/T^\star$ or $B^\pi/B_\star$, the upper bound on $\norm{V^\pi}/\norm{V^\star}$ (\emph{e.g.} \cite{tarbouriech2020no,rosenberg2020near,cohen2021minimax,chen2021finding,chen2021improved}). In contrast, VI-OPE is fully parameter-free as it requires no prior information about neither $T^\pi$ nor $B^\pi$ and the main term of our bound does not explicitly scale with those parameters. Last but not least, VI-OPE does not reply on the positive cost Assumption~\ref{assum:PC}.

\section{Offline Learning in SSP} \label{sec:OPL}

In this section, we consider the offline policy optimization problem. Similar to previous work, we assume the knowledge of an upper bound on the $B_\star:=\norm{V^\star}_\infty$, which is denoted as $\tB$. How to deal with the case when $\tB$ is unknown is discussed in Section~\ref{sec:knowladge}. Throughout the section, we suppose Assumption~\ref{assum:opl} and Assumption~\ref{assum:PC} holds.

We introduce our algorithm in Algorithm~\ref{alg:OPO}. The main idea behind the algorithm is the pessimistic update of the value function via adding a bonus function to $V^{(i)}$. Here the \textbf{bonus function} 
$b_{s,a}(V):=\sqrt{\frac{2\hat{c}(s,a)\iota}{n(s,a)}}+\frac{7\iota}{3n(s,a)}+\frac{\tB}{n(s,a)}+\frac{16\tB\iota}{3n(s,a)}+\max\{2\sqrt{\frac{\Var(\tP',V)\iota}{n(s,a)}},4\frac{\tB\iota}{n(s,a)}\}+180\sqrt{\frac{3\widetilde{T}\tB S}{2n(s,a)n_{\min}}}(\sqrt{\tB}+1)\iota$ $\forall(s,a)\in\mathcal{S}\times\mathcal{A}$, where $n_{\text{max}}=\max_{s,a}n(s,a)$ and $n_{\text{min}}=\min_{s,a}\{n(s,a):n(s,a)>0\}$. For the goal state $b_{g,a}(V)=0$ $\forall a\in\mathcal{A}$. Here $\widetilde{T}$ is an upper bound of $T^\star$.\footnote{Here we do point the design of $b_{s,a}$ requires $\widetilde{T}$ in addition to $\widetilde{B}$. However, this is not essential as (by Assumption~\ref{assum:PC}) $\widetilde{T}$ can be bounded by $\widetilde{B}/c_\text{min}$.}

\begin{algorithm}[tbh]
	\caption{PVI-SSP (Pessimistic Value Iteration for SSP)}
	\label{alg:OPO}
	\begin{algorithmic}[1]
		\STATE {\bfseries Input:} $\epsilon_{\text{OPL}}$, $\mathcal{D}:=\{(s^{(i)}_1,a^{(i)}_1,c^{(i)}_1,s^{(i)}_2,\ldots,s^{(i)}_{T_i})\}_{i=1}^n$. $\tB$ and $\iota=O(\log(SA/\delta))$. $n_{\text{max}}$ and $b_{s,a}$ see above.
		\FOR{$(s,a,s')\in \mathcal{S}\times\mathcal{A}\times\mathcal{S}'$}
		\STATE Set $n(s,a) =  \sum_{i=1}^n\sum_{j=1}^{T_i}\mathbb{I}(s_j^{(i)}=s\text{, }a_j^{(i)}=a)$.
		\IF{$n(s,a)>0$}
		\STATE Calculate $\widehat{c}(s,a)=\frac{\sum_{i=1}^n\sum_{j=1}^{T_i}\mathbb{I}(s_j^{(i)}=s\text{, }a_j^{(i)}=a)c^{(i)}_j}{n(s,a)}$
		\STATE $\widehat{P}(s'|s,a)=\frac{\sum_{i=1}^n\sum_{j=1}^{T_i}\mathbb{I}(s_j^{(i)}=s\text{, }a_j^{(i)}=a\text{, }s_{j+1}^{(i)}=s')}{n(s,a)}$,  
		\ELSE
		\STATE $\widehat{c}(s,a)\leftarrow c_{\min}$, $\widehat{P}(s'|s,a)\leftarrow\mathbb{I}(s'=g)$.
		\ENDIF
		\STATE $\widetilde{P}'(s'|s,a)=\frac{n_{\text{max}}}{n_{\text{max}}+1}\widehat{P}(s'|s,a)+\frac{\mathbb{I}[s'=g]}{n_{\text{max}}+1}$
		\ENDFOR
		\STATE \blue{$\diamond$ Pessimistic Value Iteration for offline learning} 
		\STATE {\bfseries Initialize:} ${V}^{(-1)}(\cdot)\leftarrow \infty$, ${V}^{(0)}(\cdot)\leftarrow\tB\cdot\mathbf{1}$, $i=0$.
		\WHILE{$\norm{V^{(i)}-V^{(i-1)}}_\infty> 0 \blue{(\epsilon_{\text{OPL}})}$}
		\FOR{$(s,a)\in \mathcal{S}'\times\mathcal{A}$}
		\STATE $Q^{(i+1)}(s,a)=\min\{\widehat{c}(s,a)+\widetilde{P}'_{s,a}V^{(i)}+b_{s,a}(V^{(i)})\text{ , }\tB\}$
		\STATE  $V^{(i+1)}(s)=\min_a Q^{(i+1)}(s,a)$
		\STATE $i\leftarrow i+1$
		\ENDFOR
		\ENDWHILE
		\STATE Calculate $\bar{\pi}(\cdot)=\argmin_{a}Q^{(i)}(\cdot,a)$
		\STATE \textbf{Output}: $\bar{\pi}$, $\bar{V}(\cdot)=\min_{a}Q^{(i)}(\cdot,a)$
	\end{algorithmic}
\end{algorithm}

The use of value iteration to approximate the underlying Bellman optimality equation $V^{\star}(s)=\max_{a\in\mathcal{A}}\{c(s, a)+P_{s, a} V^{\star}\},\;\forall s\in\mathcal{S}'$ is natural when model components $P,c$ are accurately estimated by $\widetilde{P}',\widehat{c}$. Moreover, comparing to VI-OPE, there are several differences for PVI-SSP. First, $\widetilde{P}'$ is chosen according to $n_\text{max}$ (instead of $n(s,a)$), which makes $\widetilde{P}'$ ``closer'' to $\widehat{P}$ but preserves the positive one-step transition to $g$. More importantly, a pessimistic bonus $b_{s,a}$ is added to the value update differently at each state-action location which measures the uncertainty learnt so far from the offline data. Action with higher uncertainty are less likely to be chosen for the next update. Concretely, $\sqrt{\frac{\Var(\widetilde{P}',V^{(i)})}{n}}$ measures the uncertainty of $V^{(i)}$ and $\sqrt{\frac{\widehat{c}}{n}}$ measures the uncertainty of per-step cost $\widehat{c}$.\footnote{This is due to $\Var(c)\leq E[c^2]\leq E[c]$ for r.v. $c\in[0,1]$.} However, to guarantee proper pessimism, we require the knowledge of $\widetilde{B}$ in the design of $b_{s,a}$.

In addition, for analysis purpose we state our result under the regime where the iteration converges exactly and the output $\bar{V}$ (in Line 22) is fixed point of the operator $\widetilde{\mathcal{T}}$ (see Appendix~\ref{sec:converge_OPO} for details). In practice, one can stop the iteration when the update difference is smaller than $\epsilon_{\text{OPL}}$. We have the following offline learning guarantee for $\bar{\pi}$, which is our major contribution. The proof is deferred to Appendix~\ref{sec:proof_OPO}.

\begin{theorem}[Offline policy learning in SSP]\label{thm:OPL}
	Denote $d_m:=\min\{\sum_{h=0}^\infty \xi^\mu_h(s,a):s.t. \sum_{h=0}^\infty \xi^\mu_h(s,a)>0\}$, and $T^\pi_s$ to be the expected time to hit $g$ when starting from $s$. Define $\bar{T}^\pi=\max_{\bar{s}\in\mathcal{S}}T^\pi_{\bar{s}}$. Then when $n\geq n_0$, we have with probability $1-\delta$, the output $\bar{\pi}$ of Algorithm~\ref{alg:OPO} is a proper policy and satisfies ($\iota=O(\log(SA/\delta))$)
	{\begin{align*}
	0\leq &V^{\bar{\pi}}(s_\mathrm{init})-V^\star(s_\mathrm{init})
	\leq 4\sum_{s,a,s\neq g} d^\star(s,a)\sqrt{\frac{2\Var_{P_{s,a}}[V^\star+c]\iota}{n\cdot d^\mu(s,a)}}+\widetilde{O}(\frac{1}{n}),
	\end{align*}
}where the quantity $d_\text{max}=\max_{s,a}d^\mu(s,a)$, the quantity ${T}_{\max}=\max_{i\in[n]} T_i$ and we define $n_0:=\max\{\frac{4B_\star-2c_\text{min}}{c_\text{min}d_{\text{max}}}, \frac{26^2\times 2S\iota(\bar{T}^\star)^2(\sqrt{B_\star}+1)^2}{d_m}, \frac{10^6(\sqrt{\tB}+1)^4S\iota\bar{T}^\star\widetilde{T}}{B^\star(\sqrt{B^\star}+1)^2d_m}, $ $O( {T}_{\max}^2\log(SA/\delta)/d_m^2)\}$.
\end{theorem}

\textbf{On guarantee for policy.} Existing online SSP works measure the algorithm performance using \emph{regret} $R^{\text{SSP}}_{K}:=\sum_{k=1}^{K} \sum_{h=1}^{I^{k}} c_{h}^{k}-K \cdot \min_{\pi \in \Pi_{\text {proper }}} V^{\pi}\left(s_{\text{init}}\right)$ (\emph{e.g.} \citep{tarbouriech2021stochastic}) and is different from policy-based regret measurement $R_K:=\sum_{k=1}^{K} V_{1}^{\star}\left(x_{k, 1}\right)-V_{1}^{\pi_{k}}\left(x_{k, 1}\right)$ (\emph{e.g.} \cite{azar2017minimax}) in standard RL. The notion of $R^{\text{SSP}}_{K}$ provides the flexibility for policy update even within the episode (since it suffices to minimize $\sum_{h=1}^{I^{k}} c_{h}^{k}$), therefore unable to output a concrete stationary policy for the policy learning purpose. In contrast, Theorem~\ref{thm:OPL} provides a policy learning result via bounding the performance of output policy $\bar{\pi}$ explicitly.

\textbf{Instance-dependent bound.} Prior online SSP studies focus on deriving better worst-case regret (\emph{e.g.} the minimax rate is of order $\Theta(B_\star\sqrt{SAK})$) where the bounds are expressed by the parameters $B_\star/D,S,A$ that lack the characterization of individual instances. In offline SSP, the main term of PVI-SSP is fully expressed by the system quantities with marginal coverage $d^\star$ and $d^\mu$, conditional variance over transition $P$ and cost function $c$. This instance-adaptive result characterizes the hardness of learning better since the magnitude of the bounds changes with the instances. It fully avoids the explicit use of worst-case parameters $B_\star,S,A$.

\textbf{Faster convergence.} When the SSP system is deterministic for both cost $c$ and transition $P$, the conditional variances $\Var_{P_{s,a}}[V^\star+c]$ are always zero. In these scenarios, Theorem~\ref{thm:OPL} automatically guarantees faster convergence rate $\widetilde{O}(1/n)$ in deterministic SSP learning. Such a feature is not enjoyed by the existing worst-case studies in online SSP as their regrets are dominated by the statistical rate $\widetilde{O}(\sqrt{K})$ even for deterministic systems.

\textbf{On optimality.} While instance-dependent, it is still of great interest to understand whether this result is optimal. We provide the affirmative answer by showing a (nearly) matching minimax lower bound under the single concentrability condition in the next section.

\section{SSP Minimax Lower Bound}\label{sec:lower}

In this section, we study the statistical limit of offline policy learning in SSP. Concretely, we consider the family of problems satisfying bounded partial coverage, \emph{i.e.} $\max_{s,a,s\neq g}\frac{d^{\pi^\star}(s,a)}{d^{\mu}(s,a)}\leq C^\star$, where $d^{\pi}(s,a)=\sum_{h=0}^\infty \xi^\pi_h(s,a)<\infty$ for all $s,a$ (excluding $g$) for any proper policy $\pi$. This $C^\star$ formally defines the maximum ratio between $\pi^\star$ and $\mu$ in Assumption~\ref{assum:opl}. Consequently, we have the following result (the full proof is in Appendix~\ref{sec:lower_proof}):

\begin{theorem}\label{thm:lower_main}
	We define the following family of SSPs:
	\[
	\mathrm{SSP}(C^\star)=\{(s_{\mathrm{init}},\mu,P,c)|\max_{s,a,s\neq g}\frac{d^{\pi^\star}(s,a)}{d^{\mu}(s,a)}\leq C^\star\},
	\]
	where $d^\pi(s,a)=\sum_{h=0}^\infty\xi^\pi_h(s,a)$. Then for any $C^\star\geq 1$, $\norm{V^\star}_\infty=B_\star>1$, it holds (for some universal constant $c$)
	\begin{align*}
	&\inf_{\widehat{\pi} \;\mathrm{proper}}\sup_{(s_{\mathrm{init}},\mu,P,c)\in\mathrm{SSP}(C^\star)}\E_{\mathcal{D}}[V^{\widehat{\pi}}(s_{\mathrm{init}})-V^\star(s_{\mathrm{init}})]
	\geq c\cdot B_\star\sqrt{\frac{SC^\star}{n}}.
	\end{align*}
\end{theorem}

Theorem~\ref{thm:lower_main} reveals for the family with proper policy $\pi^\star$ and $\mu$ with bounded ratio $C^\star$, the minimax lower bound is $\Omega(B_\star\sqrt{\frac{SC^\star}{n}})$. In particular, the dominant term in Theorem~\ref{thm:OPL} directly implies this rate (recall $\pi^\star$ is deterministic by \ref{assum:opl}) by the following calculation (assuming $B_\star>1$ just like Theorem~\ref{thm:lower_main}):
{
\begin{equation}\label{eqn:minimax_derivation}
\begin{aligned}
&\sum_{s,a,s\neq g} d^\star(s,a)\sqrt{\frac{\Var_{P_{s,a}}[V^\star+c]}{n\cdot d^\mu(s,a)}}\\
=&\sum_{s,s\neq g} d^\star(s,\pi^\star(s))\sqrt{\frac{\Var_{P_{s,\pi^\star(s)}}[V^\star+c]}{n\cdot d^\mu(s,\pi^\star(s))}}\\
\leq&\sqrt{\sum_{s,s\neq g}\frac{d^\star(s,\pi^\star(s))}{d^\mu(s,\pi^\star(s))}\cdot\sum_{s,s\neq g}\frac{d^\star(s,\pi^\star(s))\Var_{P_{s,\pi^\star(s)}}[V^\star+c]}{n}}\\
\leq&\sqrt{\sum_{s,s\neq g}C^\star\cdot \frac{B_\star^2}{n}}=B_\star\sqrt{\frac{SC^\star}{n}} \;\;(\text{also see Proposition~\ref{prop:simplified}}),
\end{aligned}
\end{equation}
}where the first inequality uses CS inequality and the second one uses the key Lemma~\ref{lem:T_pi}.\footnote{Here since $B_\star>1$, when applying Lemma~\ref{lem:T_pi}, $B_\star$ will dominate $c\in[0,1]$.} This verifies PVI-SSP is near-optimal up to the logarithmic and higher order terms.

\section{Sketch of the analysis}\label{sec:pf_overview}

In this section, we sketch the proofs of our main theorems. In particular, we focus on describing the procedure of offline policy learning Theorem~\ref{thm:OPL}. First of all, when the condition $n\geq n_0$ holds, the output $\bar{\pi}$ is proper with high probability and following this one can conduct standard decomposition:
\[
V^{\bar{\pi}}-V^\star = (V^{\bar{\pi}}-\bar{V})+(\bar{V}-V^\star)
\]
where $V^\star$ is the solution of Bellman optimality operator $\mathcal{T}$ and $\bar{V}$ is the fixed point solution of the operator $\tT(V)(s)=\min_{a}\left\{\min\{\widehat{c}(s,a)+\widetilde{P}_{s,a}V+b_{s,a}(V)\text{ , }\tB\}\right\}$ (Lem~\ref{lem:contraction2}). Also, $V^{\bar{\pi}}$ satisfies general Bellman equation (Lemma~\ref{lem:general_Bellman}) therefore we first decompose $V^{\bar{\pi}}-\bar{V}$ using a \emph{simulation-lemma} style decomposition (Lemma~\ref{lem:Vpi-V}):
\begin{align*}
\Vb-\bar{V}=&\sumS\xi_{h}^{\bar{\pi}}(s)\bigg\{(P_{s,\bar{\pi}(s)}-\tP'_{s,\bar{\pi}(s)})\bar{V}
+c(s,\bar{\pi}(s))-\hat{c}(s,\bar{\pi}(s))-b_{s,\bar{\pi}(s)}(\bar{V})\bigg\}
\end{align*}
By the careful design of $b_{s,a}(\cdot)$, the pessimism guarantees $V^{\bar{\pi}}-\bar{V}\leq 0$ (Lemma~\ref{lem:pessimism}). For $\bar{V}-V^\star$, a similar \emph{simulation-lemma} style SSP decomposition (Lemma~\ref{lem:bar_V-V*}) follows:
\begin{equation}\label{eqn:decomp_star}
\begin{aligned}
\bar{V}-V^{\star}\leq&\sumS \xi^\star_h(s)\bigg\{(\tP'_{s,\pi^\star(s)}-P_{s,\pi^\star(s)})\bar{V}
+\widehat{c}(s,\pi^\star(s))-c(s,\pi^\star(s))+b_{s,\pi^\star(s)}(\bar{V})\bigg\}.
\end{aligned}
\end{equation}

Before we proceed to explain about how to bound the residual summations, we present two new lemmas, which help characterize the key features of stochastic shortest path problem.

\begin{lemma}[Informal version of Lemma~\ref{lem:propT}]\label{lem:T_pi}
	Let $T^{\pi}$ be the expected time of arrival to goal state $g$ when applying proper policy $\pi$ and starting from $s_\mathrm{init}$, then 
	\[
	T^{\pi}=\sum_{h=0}^{\infty}\sum_{\substack{s,a \\ s\neq g}}\xi_{h}^{\pi}(s,a)=\sum_{\substack{s,a \\ s\neq g}}d^{\pi}(s,a).
	\]
\end{lemma}

\begin{remark}\label{remark:T_pi}
Lemma~\ref{lem:T_pi} explicitly reflects the connection between the expected arriving time $T^\pi$ and marginal coverage $d^\pi(s,a)$. Unlike the finite-horizon problem where $d^\pi_h$ are probability measures (\emph{e.g.} see \cite{yin2021towards}), for SSP $d^\pi(s,a)$ can be arbitrary large (for a general policy $\pi$) due to definition~\ref{def:mar_cov}. Lemma~\ref{lem:T_pi} guarantees $d^\pi(s,a)<\infty$ for proper policy $\pi$ since by Definition~\ref{def:proper} $T^\pi<\infty$, and, as a result, make our bound in Theorem~\ref{thm:OPL} valid. Note similar result is of less interests in the standard finite-horizon episodic RL since it holds trivially that $H=\sum_{h=1}^H\sum_{s,a}d^\pi_h(s,a)$ and, in SSP, this becomes important as we have undetermined horizon length. With Lemma~\ref{lem:T_pi}, we can get away with estimating the aggregated measure $T^\pi/T^\star$ (like previous online SSP papers did) and use sub-component $d^\pi(s,a)/d^\star(s,a)$ to reflect the behaviors of individual state-action pairs and achieve more instance-dependent results. 
\end{remark}

\begin{lemma}[Informal version of Lemma~\ref{lem:bound_sum_var}]\label{lem:HD}
	For any probability transition matrix $P$, policy $\pi$, 
	and any cost function $c\in[0,1]$ associated with ${SSP}(P,\pi)$. Suppose $V\in\R^{S+1}$ is any value function satisfying order property (where $V(g)=0$), i.e., $V(s)\geq \sum_{a}\pi(a|s)P_{s,a}V$ for all $s\in\mathcal{S}$, then we have 
	{
		\[
	\sum_{h=0}^{\infty}\sum_{\substack{s,a \\ s\neq g}}\xi_h^{\pi}(s,a)\Var_{P_{s,a}}(V)\leq2\norm{V}_{\infty}\cdot V(s_\mathrm{init})\leq2\norm{V}_{\infty}^2.
\]}
\end{lemma}
\begin{remark}\label{remark:HD}
	Lemma~\ref{lem:HD} can be viewed as a dependence improvement result for SSP problem since it guarantees Theorem~\ref{thm:OPL} to achieve the minimax rate via \eqref{eqn:minimax_derivation}. More critically, it widely applies to arbitrary policies assuming the ordering condition holds for $V$. For instance, a direct upper bound using Lemma~\ref{lem:T_pi} would yield $T^\pi\norm{V}^2_\infty$ and $T^\pi$ could be very large or even $\infty$. In contrast, Lemma~\ref{lem:HD} always upper bounds by $2\norm{V}^2_\infty$ without extra dependence. Similar result was previously derived in RL, \emph{e.g.} Lemma~3.4 of \cite{yin2020asymptotically} and also \cite{ren2021nearly}, but their result only applies to $V^\pi$ due the analysis via law of total variances and ours applies to all $V$ (satisfying ordering condition) through only the telescoping sum.
\end{remark}

Now we go back to bounding \eqref{eqn:decomp_star}. First of all, by leveraging Lemma~\ref{lem:T_pi}, we are able to bound the $\infty$-norm of $\bar{V}-V^\star$ as (see Theorem~\ref{thm:crude_PO})
\begin{equation}\label{eqn:crude}
\norm{\bar{V}-V^{\star}}_\infty\leq30\sqrt{\frac{\bar{T}^\star B_\star\iota}{nd_m}}(\sqrt{B_\star}+1)
\end{equation}
which is a crude/suboptimal bound that serves as an intermediate step for the final bound. 

\textbf{What give rise to instance-dependencies.} Next, we apply \emph{empirical Bernstein inequality} for structure $(\tP'_{s,\pi^\star(s)}-P_{s,\pi^\star(s)})\bar{V}$ and $\widehat{c}(s,\pi^\star(s))-c(s,\pi^\star(s))$ separately. In particular, since both $\bar{V}$ and $\tP'_{s,\pi^\star(s)}$ depend on data, therefore Bernstein concentration cannot be directly applied. Informally, we can surpass this hurdle by decomposing
\[
(\tP'-P)\bar{V}=(\tP'-P)(\bar{V}-V^\star)+(\tP'-P){V}^\star.
\]
In this scenario, concentration can be readily applied to $(\tP'-P){V}^\star$ and crude bound \eqref{eqn:crude} is leveraged here for bounding $(\tP'-P)(\bar{V}-V^\star)\leq \norm{\tP'-P}_1\norm{\bar{V}-V^\star}_\infty$. As explained by \cite{zanette2019tighter}, the use of Bernstein concentration is the key for characterizing the structure of problem instance via the expression of conditional variance $\Var_{P_{s,a}}(V^\star)$.

\textbf{On the proof for VI-OPE.} At a high level, the proof for VI-OPE (Theorem~\ref{thm:ope}) shares the same flavor as that of Theorem~\ref{thm:OPL}. Ideally, in finite horizon setting the tighter analysis could be conducted by following the pipeline of Section~B.7 in \cite{duan2020minimax}, where the dominant error of $\widehat{V}^\pi-V^\pi$ (where $\widehat{V}^\pi=\lim_{i\rightarrow\infty} V^{(i)}$ in Algorithm~\ref{alg:VI_OPE}) can be decomposed as:
{
\[
\frac{1}{n}\sum_{i=1}^n\sum_{h=0}^\infty\frac{\xi^\pi_h(s^{(i)}_h,a^{(i)}_h)}{\xi^\mu_h(s^{(i)}_h,a^{(i)}_h)}\left(Q^\pi(s^{(i)}_h,a^{(i)}_h)-(c(s^{(i)}_h,a^{(i)}_h)+V^\pi(s^{(i)}_h,a^{(i)}_h))\right)
\]
}Applying Freedman's inequality for the above martingale structure, one can hope for a tighter rate $O(\sqrt{\sum_{s,a} d^\pi(s,a)^2\frac{\Var_{P_{s,a}}[V^\pi+c]}{n\cdot d^\mu(s,a)}})$. However, such a procedure will have technical issue for SSP problem since: (1) SSP has stationary transition $P$ and $n(s,a)$ is computed via collecting all the transitions that encounter $s,a$ for tighter dependence. This breaks the sequential ordering that is needed for martingale.\footnote{Note \cite{duan2020minimax} Corollary~1 considers time-inhomogeneous MDP and each $P_t$ can be estimated stage-wisely so the decomposition forms a martingale.} (2) Even if we have a martingale, the martingale difference will incorporate an infinite sum that could be arbitrary large. Both facts indicate Freedman's inequality cannot be directly applied due to the technical hurdle.

Lastly, the lower bound proof uses a generalized Fano's argument (Lemma~\ref{lem:gen_Fano}), followed by reducing estimation problem to testing. The packing set of hard MDP instances is based on the modification of \cite{rashidinejad2021bridging} so that Gilbert-Varshamov Lemma~\ref{lem:GV} can be applied.

\section{Discussions }


\subsection{The knowledge of $B_\star/\widetilde{B}$}\label{sec:knowladge}

While VI-OPE (Algorithm~\ref{alg:VI_OPE}) is parameter-free, our policy learning algorithm PVI-SSP (Algorithm~\ref{alg:OPO}) requires $\widetilde{B}$ in the pessimistic bonus design. Since $\widetilde{B}$ is an upper bound of $B_\star$, one natural idea is to use VI-OPE to provide an upper bound estimation given that a proper policy is provided. This idea is summarized as below. 

\begin{proposition}[Alternative offline learning algorithm VI-OPE+PVI-SSP]
	Suppose we are provided with an arbitrary proper policy $\pi$ (\emph{e.g.} some previously deployed strategy). In this scenario, one can equally halve the data $\mathcal{D}$ into $\mathcal{D}_1$ and $\mathcal{D}_2$, and use $\mathcal{D}_1$ to evaluate $V^\pi$. The $\infty$-norm of VI-OPE output serves as surrogate for $\widetilde{B}$ and uses as an input for computing $b_{s,a}$. Next, use $\mathcal{D}_2$ to run PVI-SSP (with calculated $b_{s,a}$).
\end{proposition}

The above procedure will not deteriorate the theoretical guarantee since $\widetilde{B}$ is only used in $O(1/n)$ terms and the estimation error can only be higher order terms. This means we will end up with the same dominant term as Theorem~\ref{thm:OPL}.

\subsection{On higher order terms.}
In our analysis of Theorem~\ref{thm:OPL}, while the dominant $\widetilde{O}(\sqrt{1/n})$ term is near-optimal, the higher order term $\widetilde{O}(1/n)$ is not and depends on the parameters including $\widetilde{T}$, $\widetilde{B}$ and $d_m$ (\emph{e.g.} check the last line of \eqref{eqn:final_derivation}). In particular, if one can remove the polynomial dependence of $\widetilde{T}$, then the result is called \emph{horizon-free} \citep{tarbouriech2021stochastic}. One potential approach for addressing the higher order dependence could be the recent development of robust estimation in RL \citep{wagenmaker2021first}. As the initial attempt for offline SSP, this is beyond our scope and we leave it as the future work.

\subsection{Future directions}

\textbf{SSP under weaker conditions.} Following previous works, we consider stochastic shortest path problem with a discrete action space $\mathcal{A}$ and non-negative cost bounded by $c\in[0,1]$. However, the convergence of SSP can hold under much weaker conditions. For instance, \cite{bertsekas2013stochastic} shows under \emph{compactness and continuity condition}, \emph{i.e.} for each state $s$ the admissible action set $\mathcal{A}(s)$ is a  compact metric space and a subset of $\mathcal{A}$ where (for all $s'$) transition $P(s'|s,\cdot)$ are continuous functions over $\mathcal{A}(s)$ and the cost function $c(s,\cdot)$ is lower semi-continuous over $\mathcal{A}(s)$, value iteration/policy iteration will still work under mild assumptions. This extends our setting (\emph{e.g.} cost can even be negative) and how to conduct SSP learning in this case remains open. 

\textbf{Extension to linear MDP case.} Another natural and promising generalization of the current study is the offline linear MDP for SSP problem. In the study of offline RL with linear MDPs, \cite{jin2021pessimism} shows the provable efficiency, \cite{zanette2021provable} improves the result in the \emph{linear Bellman complete} setting and \cite{yin2022near} leverages variance-reweighting for least square objective to obtain the near-optimal result. Adopting their useful results in offline SSP problem is hopeful.

\section{Conclusion}\label{sec:conclusion}

In this paper, we initiate the study of \emph{offline stochastic shortest path} problem. We consider both \emph{offline policy evaluation} (OPE) and \emph{offline policy learning} tasks and propose the simple value-iteration-based algorithms (VI-OPE and PVI-SSP) that yield strong theoretical guarantees for both evaluation and learning tasks. To complement the discussion, we also provide an information-theoretical lower bound and it certifies PVI-SSP is minimax rate optimal. We hope our work can draw further attention for studying offline SSP setting.

\begin{acknowledgements} 
   Ming Yin and Yu-Xiang Wang are partially supported by NSF Awards \#2007117 and \#2003257. MY would like to thank Tongzheng Ren for helpful discussions.
\end{acknowledgements}
\bibliography{sections/stat_rl}

\appendix

\section{General Bellman equation for a fixed policy}\label{sec:gen_Bellman}
In this section, we prove Proposition~\ref{prop:Bellman}. In particular, the first part of the proposition for $V^\star$ has been covered in \cite{bertsekas1991analysis}. Therefore, we only consider the second part, which is a Bellman equation for fixed policy. Moreover, we do not constraint to proper policy and our result holds true for all the policies.
\begin{lemma}[General Bellman equation for fixed policy $\pi$]\label{lem:general_Bellman}
	Let $\pi$ be a fixed policy, proper or improper and cost $c\geq 0$ for the SSP. Then the following Bellman equations hold: 
	\begin{equation}
	\begin{aligned}
	Q^{\pi}(s, a)=c(s, a)+P_{s, a} V^{\pi}, \quad V^{\pi}(s)=\E_{a\sim\pi(\cdot|s)} [Q^{\pi}(s, a)].
	\end{aligned}
	\end{equation}
\end{lemma}

\begin{proof}[Proof of Lemma~\ref{lem:general_Bellman}]
	By definition of $Q^\pi$, we have 
	\begin{align*}
	Q^{\pi}(s,a)=\lim_{T\rightarrow\infty}\mathbb{E}_{\pi}[\sum_{h=0}^Tc(s_h,a_h)|s_0=s,a_0=a].
	\end{align*}
	We can rewrite term $\mathbb{E}_{\pi}[\sum_{h=0}^Tc(s_h,a_h)|s_0=s,a_0=a]$ as
	\begin{align*}
	\mathbb{E}_{\pi}[\sum_{h=0}^Tc(s_h,a_h)|s_0=s,a_0=a]
	&=c(s,a)+\sum_{s'}\P(s'|s,a)\mathbb{E}_{\pi}[\sum_{h=1}^Tc(s_h,a_h)|s_1=s']\\
	&=c(s,a)+\sum_{s'}\P(s'|s,a)\left\{\mathbb{E}_{\pi}[\sum_{h=0}^{T-1}c(s_h,a_h)|s_0=s']\right\},
	\end{align*}
	where the first equality is by law of total expectation. The second equality follows from the fact that the transition kernel $P$ is \textbf{homogeneous} in SSP.

	Define the sequence $V_T(s):=\left\{\mathbb{E}_{\pi}[\sum_{h=0}^{T-1}c(s_h,a_h)|s_0=s]\right\}$. Since for any state-action pair $(s,a)$, $c(s,a)\geq0$, we know that the sequence $\{V_T(s)\}_{T=1}^{\infty}$ is non-decreasing. It implies that $\lim_{T\rightarrow\infty}V_T(s)$ exists. ($\lim_{T\rightarrow\infty}V_T(s)$  either diverges to $+\infty$ or converges to a positive number.) It follows that (the following switching the order of limit and summation is valid since the summation is finite sum)

	\begin{align}
	\lim_{T\rightarrow\infty}\sum_{s'}\P(s'|s,a)\left\{\mathbb{E}_{\pi}[\sum_{h=0}^{T-1}c(s_h,a_h)|s_0=s']\right\}=\sum_{s'}\P(s'|s,a)\lim_{T\rightarrow\infty}\left\{\mathbb{E}_{\pi}[\sum_{h=0}^{T-1}c(s_h,a_h)|s_0=s']\right\}. 
	\end{align}
	
	Combine the above two equalities together, we can get
	\begin{align}
	Q^{\pi}(s,a)&=c(s,a)+\sum_{s'}\P(s'|s,a)\lim_{T\rightarrow\infty}\left\{\mathbb{E}_{\pi}[\sum_{h=0}^{T-1}c(s_h,a_h)|s_0=s']\right\}\notag\\
	&=c(s,a)+\sum_{s'}\P(s'|s,a)V^{\pi}(s').\notag\\
	\end{align}
	From the definition of value function, we have (where the second line uses law of total expectation)
	\begin{align*}
	V^{\pi}(s)&=\lim_{T\rightarrow\infty}\mathbb{E}_{\pi}[\sum_{h=0}^Tc(s_h,a_h)|s_0=s]\\
	&=\lim_{T\rightarrow\infty}\E_{a_0}[\mathbb{E}_{\pi}[\sum_{h=0}^Tc(s_h,a_h)|s_0=s,a_0]|s_0=s]\\
	&=\lim_{T\rightarrow\infty}\sum_{a}\pi(a|s)\mathbb{E}_{\pi}[\sum_{h=0}^Tc(s_h,a_h)|s_0=s,a_0=a]
	\end{align*}
	Similar to $\lim_{T\rightarrow\infty}\mathbb{E}_{\pi}[\sum_{h=0}^Tc(s_h,a_h)|s_0=s]$, we can prove that $\lim_{T\rightarrow\infty}\mathbb{E}_{\pi}[\sum_{h=0}^Tc(s_h,a_h)|s_0=s,a_0=a]$ exists. Then we have 
	\begin{align*}
	V^{\pi}(s)
	&=\lim_{T\rightarrow\infty}\sum_{a}\pi(a|s)\mathbb{E}_{\pi}[\sum_{h=0}^Tc(s_h,a_h)|s_0=s,a_0=a]\\
	&=\sum_{a}\pi(a|s)\lim_{T\rightarrow\infty}\mathbb{E}_{\pi}[\sum_{h=0}^Tc(s_h,a_h)|s_0=s,a_0=a]=\sum_{a}\pi(a|s)Q^{\pi}(s,a).
	\end{align*}
\end{proof}

\begin{remark}
	Essentially, the above proof only requires $c(s,a)\geq 0$. Moreover, even if the general Bellman equation holds, it does not imply $c^\pi+P^\pi(\cdot)$ is a contraction (\emph{i.e.} doing value iteration for general policy $\pi$ might not converge to $V^\pi$).
\end{remark}

\section{Results for general Stochastic Shortest path problem}
\begin{lemma}
	\label{lem:compare_T}
	For any two contraction mapping $T_1$ and $T_2$ that are monotone ({i.e.} for any vector greater than $V\geq V'$, it holds $T_1V\geq T_1V'$ and $T_2V\geq T_2V'$) on the metric space $\R^{\mathcal{S'}}$. Suppose $V_1$ and $V_2$ are the fixed points for $T_1$ and $T_2$ respectively. If we have $T_1(V)(s)\geq T_2(V)(s)$ for any $s\in\mathcal{S}'$, then we have $V_1(s)\geq V_2(s)$ for any $s\in\mathcal{S}'$.
\end{lemma}
\begin{proof}
	First we have $T_1V_1\geq T_2V_1$. Since $V_1$ is the fixed point of $T_1$, we know $V_1:=T_1V_1\geq T_2V_1$. By monotone property with recursion, we have that
	\begin{equation}
	V_1\geq (T_2)^kV_1.
	\end{equation}
	Since $V_2$ is the fixed point of $T_2$, we have
	\begin{align*}
	\lim_{k\rightarrow\infty}(T_2)^kV_1=V_2.
	\end{align*}
	Combine the above inequalities together we can get $V_1\geq V_2$.
\end{proof}

\section{Convergences for Algorithm~\ref{alg:VI_OPE}}
\begin{lemma}
\label{lem:contraction}
$\widehat{\mathcal{T}}^{\pi}:\mathbb{R}^\mathcal{S}\times\{0\}\rightarrow \mathbb{R}^\mathcal{S}\times\{0\}$ is a contraction mapping, i.e., $\forall V_1,V_2\in\mathbb{R}^\mathcal{S'}$, $V_1(g)=V_2(g)=0$, we have
\begin{align}
    \norm{\widehat{\mathcal{T}}^{\pi}
    V_1
    -\widehat{\mathcal{T}}^{\pi}V_2}_\infty
    \leq \rho \norm{V_1-V_2}_\infty,
\end{align}
Here $\rho:=\max_{\substack{s,a \\ s\neq g}}(\frac{n_{s,a}}{n_{s,a}+1})<1$ and $\widehat{\mathcal{T}}^{\pi}V(s)=\langle \pi(\cdot|s), \widehat{c}(s,\cdot)+\widetilde{P}_{s,\cdot}V\rangle$.
\end{lemma}
\begin{proof}[Proof of Lemma~\ref{lem:contraction}]
We first prove the result for state $g$. Since $g$ is a zero-cost absorbing state, we have for any $a\in\mathcal{A}$, $\widehat{c}(g,a)=0$ and $\widetilde{P}_{g,a}V=V(g)$. Then for any $V\in\mathbb{R}^\mathcal{S}\times\{0\}$, $V(g)=0$, we have 
\begin{align}
    \widehat{\mathcal{T}}^{\pi}V(g)=\langle \pi(\cdot|g), \widehat{c}(g,\cdot)+\widetilde{P}_{g,\cdot}V\rangle=0.
\end{align}
Therefore $\widehat{\cT}^{\pi}V_1(g)-\widehat{\mathcal{T}}^{\pi}V_2(g)=0\leq \rho \norm{V_1-V_2}_\infty$. Next we only need to prove for all state $\forall s\neq g$. Indeed,
\begin{align}
   |\widehat{\cT}^{\pi}V_1(s)-\widehat{\mathcal{T}}^{\pi}V_2(s)|&= |\langle \pi(\cdot|s), \widetilde{P}_{s,\cdot}(V_1-V_2)\rangle|\notag\\
   &\leq\max_a|{\widetilde{P}_{s,a}(V_1-V_2)}|\notag\\
   &=\max_a|\sum_{s'\neq g}\widetilde{P}(s'|s,a)(V_1(s')-V_2(s'))|\notag\\
   &\leq\max_a(\frac{n_{s,a}}{n_{s,a}+1})|\sum_{s'\neq g}\widehat{P}(s'|s,a)(V_1(s')-V_2(s'))|\notag\\
   &\leq\max_a(\frac{n_{s,a}}{n_{s,a}+1})\norm{V_1-V_2}_{\infty}.
\end{align}
where the second inequality is due to $V_1(g)=V_2(g)=0$ and the third inequality is by the definition of $\widetilde{P}$. Take the supremum over $s$, we get
\begin{align}
    \norm{\widehat{\cT}^{\pi}V_1-\widehat{\mathcal{T}}^{\pi}V_2}_\infty\leq\max_{\substack{s,a \\ s\neq g}}(\frac{n_{s,a}}{n_{s,a}+1})\norm{V_1-V_2}_{\infty}
\end{align}
\end{proof}

\begin{lemma}\label{lem:bound}
$\forall \pi\in\Pi_{\text{proper}}$, define $\widehat{V}^\pi:=\lim_{i\rightarrow\infty} V^{(i)}$ (Note by Lemma~\ref{lem:contraction} this limit always exists since $\widehat{V}^\pi$ is the fixed point of $\widehat{\mathcal{T}}^\pi$ and $V^{(i+1)}=\widehat{\mathcal{T}}^\pi V^{(i)}$). Then (recall $\rho:=\max_{\substack{s,a \\ s\neq g}}(\frac{n_{s,a}}{n_{s,a}+1})<1$)
\[
||\widehat{V}^{\pi}||_{\infty}\leq\max_{\substack{s,a \\ s\neq g}}n(s,a)+1.
\]
\end{lemma}
\begin{proof}[Proof of Lemma~\ref{lem:bound}]
Recall the definition, $\hV^{\pi}=\hT^{\pi}\widehat{V}^{\pi}$
\begin{align}
    &\norm{\hV^\pi}_{\infty}=\norm{\hT^{\pi}\widehat{V}^{\pi}}_{\infty}\\
    &\leq \max_{s, s\neq g}|\langle \pi(\cdot|s), \widehat{c}(s,\cdot)\rangle|+ \max_{s, s\neq g}|\langle \pi(\cdot|s), \widetilde{P}_{s,\cdot}\hV^{\pi}\rangle|\\
    &\leq\max_{s, s\neq g}|\langle \pi(\cdot|s), \widehat{c}(s,\cdot)\rangle|+\max_{\substack{s,a \\ s\neq g}}|\widetilde{P}_{s,a}\hV^{\pi}|\notag\\
    &\leq1+\max_{\substack{s,a \\ s\neq g}}|\widetilde{P}_{s,a}\hV^{\pi}|\notag\\
   &\leq1+\max_{\substack{s,a \\ s\neq g}}\{(\frac{n_{s,a}}{n_{s,a}+1})|\sum_{s'\neq g}\widehat{P}(s'|s,a)\hV^{\pi}(s')|\}\notag\\
   &\leq1+\rho\norm{\hV^{\pi}}_{\infty}.
\end{align}
\end{proof}
The first inequality follows from $\widehat{\cT}^{\pi}\hV^{\pi}(g)=0$ and triangle inequality.
Since $\rho<1$, we can get $\norm{\hV^{\pi}}_{\infty}\leq\frac{1}{1-\rho}$. From the definition of $\rho$, we can conclude the proof.

\begin{lemma}\label{lem:diff_bound}
$\norm{\widehat{V}^{\pi}-V^{(i)}}_\infty\leq\frac{\epsilon_{\text{OPE}}}{1-\rho}$, where $\rho:=\max_{\substack{s,a \\ s\neq g}}(\frac{n_{s,a}}{n_{s,a}+1})<1$ as in Lemma~\ref{lem:contraction} and $V^{(i)}$ is the output of Algorithm~\ref{alg:VI_OPE}.
\end{lemma}
\begin{proof}[Proof of Lemma~\ref{lem:diff_bound}]
	Using definition $\widehat{V}^\pi:=\lim_{j\rightarrow\infty} V^{(j)}$ and the telescoping sum we obtain
\begin{align*}
    \norm{\widehat{V}^{\pi}-V^{(i)}}_{\infty}&\leq\sum_{j=i}^{\infty}\norm{V^{(j+1)}-V^{(j)}}_{\infty}
    \leq\norm{V^{(i+1)}-V^{(i)}}_{\infty}\sum_{j=0}^{\infty}\rho^j\leq \frac{\epsilon_{\mathrm{OPE}}}{1-\rho},
\end{align*}
where the second inequality uses $\widehat{\mathcal{T}}^\pi$ is a $\rho$-contraction.
\end{proof}

\begin{remark}
	Throughout the paper, we denote the number of state-action visitation as either $n_{s,a}$ or $n(s,a)$. They represent the same quantity.
\end{remark}

\begin{lemma}
	\label{lem:(tP-hP)}
	For any ${V}(\cdot)\in\mathbb{R}^{S}$ satisfying $V(g)=0$,
	\begin{align}
	|(\widetilde{P}_{s,a}-\widehat{P}_{s,a})V|\leq\frac{||V||_{\infty}}{n(s,a)+1}, \qquad |\Var(\tP_{s,a},V)-\Var(\hP_{s,a},V)|\leq\frac{2\norm{V}_{\infty}^2S}{n(s,a)+1}.
	\end{align}
\end{lemma}
\begin{proof}
	See Lemma~12 of \cite{tarbouriech2021stochastic}.
\end{proof}

\section{Some key lemmas}
\subsection{High Probability Event}
We define the \emph{``good property''} event $\mathcal{E}:=\mathcal{E}_1\cap\mathcal{E}_2\cap\mathcal{E}_3\cap\mathcal{E}_4\cap\mathcal{E}_5$ according the following (where $\iota:=\log(10S^2A/\delta)$) 
\begin{align}\label{eqn:high_prob}
    \mathcal{E}_1&:=\left\{\forall(s,a,s')\in(\mathcal{S}\times\mathcal{A}\times\mathcal{S})\text{, }\forall n(s,a)\geq1: |P(s'|s,a)-\widehat{P}(s'|s,a)|\leq\sqrt{\frac{2P(s'|s,a)\iota}{n(s,a)}}+\frac{2\iota}{3n(s,a)}\right\}\notag\\
    \mathcal{E}_2&:=\left\{\forall(s,a)\in(\mathcal{S}\times\mathcal{A})\text{, }\forall n(s,a)\geq1: |(P_{s,a}-\widehat{P}_{s,a})V|\leq\sqrt{\frac{2\Var(P_{s,a},V)\iota}{n(s,a)}}+\frac{2||V||_{\infty}\iota}{3n(s,a)}\right\}\notag\\
    \mathcal{E}_3&:=\left\{\forall(s,a)\in(\mathcal{S}\times\mathcal{A})\text{, }\forall n(s,a)\geq1: |(P_{s,a}-\widehat{P}_{s,a})V|\leq\sqrt{\frac{2\Var(\hP_{s,a},V)\iota}{n(s,a)}}+\frac{7||V||_{\infty}\iota}{3n(s,a)}\right\}\notag\\
    \mathcal{E}_4&:=\left\{\forall(s,a)\in(\mathcal{S}\times\mathcal{A})\text{, }\forall n(s,a)\geq1: |\widehat{c}(s,a)-c(s,a)|\leq\sqrt{\frac{2\Var_c(s,a)\iota}{n(s,a)}}+\frac{2\iota}{3n(s,a)}\right\}\notag\\
    \mathcal{E}_5&:=\left\{\forall(s,a)\in(\mathcal{S}\times\mathcal{A})\text{, }\forall n(s,a)\geq1: |\widehat{c}(s,a)-c(s,a)|\leq\sqrt{\frac{2\widehat{c}(s,a)\iota}{n(s,a)}}+\frac{7\iota}{3n(s,a)}\right\}.
\end{align}

\begin{lemma}\label{lem:high_probability}
The event $\mathcal{E}$ holds for any $V$ that is independent from $\hP$ with probability $1-\frac{\delta}{2}$.
\end{lemma}

\begin{proof}
	From the empirical Bernstein's inequality given in Lemma~\ref{lem:empirical_bernstein_ineq}, we have that for each fixed $(s,a)$, the event $|(P_{s,a}-\widehat{P}_{s,a})V|\leq\sqrt{\frac{2\Var(\hP_{s,a},V)\iota}{n(s,a)}}+\frac{7||V||_{\infty}\iota}{3n(s,a)}$ holds with probability $1-\frac{\delta}{10S^2A}$. By taking a union bound, we have that event $\mathcal{E}_3$ holds with probability $1-\frac{\delta}{10S}$. Similarly, we have event $\mathcal{E}_5$ holds with probability $1-\frac{\delta}{10S}$. By applying the standard Bernstein's inequality in
	Lemma~\ref{lem:bernstein_ineq} and taking union bound over $(s,a,s')$, we can get that event $\mathcal{E}_1$, $\mathcal{E}_2$ and $\mathcal{E}_4$ holds with probability $1-\frac{\delta}{10}$. Since $\mathcal{E}$ is the intersection of the above events, we can prove the lemma by taking a union bound again over all of the five events.
\end{proof}

\subsection{Value Decomposition Lemma}
\begin{lemma}\label{lem:regret_decomp}
Suppose $\widehat{V}^\pi:=\lim_{j\rightarrow\infty} V^{(j)}$ where $V^{(j)}=\widehat{\mathcal{T}}^\pi V^{(j-1)}$ for all $j$, then we have the following suboptimality decomposition for any initial state $\bar{s}$:
\begin{align}
    \widehat{V}^{\pi}(\bar{s})-V^{\pi}(\bar{s})=\sum_{h=0}^{\infty}\sum_{\substack{s,a \\ s\neq g}}\xi_{h,\bar{s}}^{\pi}(s,a)\{(\widehat{c}-c)(s,a)+(\tilde{P}_{s,a}-{P}_{s,a})\widehat{V}^{\pi}\}
\end{align}
\end{lemma}
\begin{proof}[Proof of Lemma~\ref{lem:regret_decomp}]
 We prove this lemma by recursion. First, we have for any $h\geq 0$,
\begin{align*}
    \sum_{\substack{s \\ s\neq g}}\xi_{h,\bar{s}}^{\pi}(s)(\widehat{V}^{\pi}(s)-V^{\pi}(s))&=\sum_{\substack{s \\ s\neq g}}\xi_{h,\bar{s}}^{\pi}(s)\sum_a\pi(a|s)(\widehat{Q}^{\pi}(s,a)-Q^{\pi}(s,a))\notag\\
    &=\sum_{\substack{s,a \\ s\neq g}}\xi_{h,\bar{s}}^{\pi}(s,a)(\widehat{Q}^{\pi}(s,a)-Q^{\pi}(s,a))\notag\\
    &=\sum_{\substack{s,a \\ s\neq g}}\xi_{h,\bar{s}}^{\pi}(s,a)\{(\widehat{c}-c)(s,a)+(\widetilde{P}_{s,a}\widehat{V}^{\pi}-{P}_{s,a}V^{\pi})\}\notag\\
    &=\sum_{\substack{s,a \\ s\neq g}}\xi_{h,\bar{s}}^{\pi}(s,a)\{(\widehat{c}-c)(s,a)+(\widetilde{P}_{s,a}-P_{s,a})\widehat{V}^{\pi}+{P}_{s,a}(\widehat{V}^{\pi}-V^{\pi})\}\notag\\
    &=\sum_{\substack{s,a \\ s\neq g}}\xi_{h,\bar{s}}^{\pi}(s,a)\{(\widehat{c}-c)(s,a)+(\widetilde{P}_{s,a}-P_{s,a})\widehat{V}^{\pi}\}+\sum_{\substack{s \\ s\neq g}}\xi_{h+1,\bar{s}}^{\pi}(s)(\widehat{V}^{\pi}-V^{\pi})(s),
\end{align*}
where the third equality uses both Bellman equations and empirical Bellman equations and the last equality follows from the fact that $\xi_{h+1}(s')=\sum_{s,a}\xi_{h}(s,a)P(s'|s,a)$ and  $\widehat{V}^{\pi}(g)=V^{\pi}(g)=0$.
By recursion, we have that
\begin{align}
\label{equation:decomposition_recursion}
    \widehat{V}^{\pi}(\bar{s})-V^{\pi}(\bar{s})=\sum_{h=0}^{H}\sum_{\substack{s,a \\ s\neq g}}\xi_{h,\bar{s}}^{\pi}(s,a)\{(\widehat{c}-c)(s,a)+(\tilde{P}_{s,a}-{P}_{s,a})\widehat{V}^{\pi}\}+\sum_{\substack{s \\ s\neq g}}\xi_{H+1,\bar{s}}^{\pi}(s)(\widehat{V}^{\pi}(s)-V^{\pi}(s)),
\end{align}
for all $H$. Then we have
\begin{align}
    |\sum_{\substack{s \\ s\neq g}}\xi_{H+1,\bar{s}}^{\pi}(s)(\widehat{V}^{\pi}(s)-V^{\pi}(s))|\leq \sum_{\substack{s \\ s\neq g}}\xi_{H+1,\bar{s}}^{\pi}(s)\cdot\norm{\widehat{V}^{\pi}-V^{\pi}}_{\infty}\notag
    \leq P^\pi_{\bar{s}}(s_{H+1}\neq g)\cdot\norm{\widehat{V}^{\pi}-V^{\pi}}_{\infty}.
\end{align}
Since $\pi$ is proper, we have $||V^{\pi}||_{\infty}\leq\infty$ and by Lemma~\ref{lem:propT} $\lim_{H\rightarrow+\infty}P^\pi_{\bar{s}}(s_{H}\neq g)=0$. From Lemma~\ref{lem:bound}, we have $||\widehat{V}^{\pi}||_{\infty}\leq\infty$. It follows that
\begin{align}
   \lim_{H\rightarrow+\infty}\sum_{\substack{s \\ s\neq g}}\xi_{H+1,\bar{s}}^{\pi}(s)(\widehat{V}^{\pi}(s)-V^{\pi}(s))=0
\end{align}
By taking $H$ to infinity in Equation (\ref{equation:decomposition_recursion}), we conclude the proof.
\end{proof}

\subsection{Key lemmas: arrival time decomposition and dependence improvement for SSP}

Below we present two lemmas for SSP problem, which is the key for obtaining tight instance-dependent bounds.
\begin{lemma}[Arrival time decomposition]
\label{lem:propT}
Let $T^{\pi}_{\bar{s}}$ be the expected time of arrival to goal state $g$ when applying proper policy $\pi$ and starting from $\bar{s}$, then $T^{\pi}_{\bar{s}}=\sum_{h=0}^{\infty}\sum_{\substack{s,a \\ s\neq g}}\xi_{h,\bar{s}}^{\pi}(s,a)$. Moreover, $T^\pi_{\bar{s}}<\infty$ for all $\bar{s}$.
\end{lemma}
\begin{proof}[Proof of Lemma~\ref{lem:propT}]
 Denote $T$ to be the random variable of arrival time to goal state $g$ when applying proper policy $\pi$, starting from $\bar{s}$. Then $\E[T]=T^\pi_{\bar{s}}$. Furthermore, since $T$ is non-negative integral variable, it holds $\E[T]=\sum_{h=0}^\infty \P(T>h)$.
 
 Then we have
\begin{align*}
    T^{\pi}_{\bar{s}}&=\E_{P,\pi}T
    =\sum_{h=0}^\infty\P_{P,\pi}(T> h)\\
    &=\sum_{h=0}^\infty\P_{P,\pi}(s_1\neq g,s_2\neq g,...,s_h\neq g)\\
        &\explaineq{(i)}{=}\sum_{h=0}^\infty\P_{P,\pi}(s_h\neq g)
        =\sum_{h=0}^{\infty}\sum_{\substack{s,a \\ s\neq g}}\xi_{h,\bar{s}}^{\pi}(s,a),
\end{align*}
where equality (i) follows from the fact that $g$ is an absorbing state, so we can only reach a state which is not a goal state if all the previous steps are not goal state and vice versa.

Lastly, since $\pi$ is proper, $T^\pi_{\bar{s}}<\infty$ for all $\bar{s}$ and this implies $\lim_{h\rightarrow\infty}\P_{P,\pi}(s_h\neq g)=0$.
\end{proof}

The next lemma is the key for achieving optimal rate.

\begin{lemma}[Dependency Improvement]
\label{lem:bound_sum_var}
For any probability transition matrix $P$, policy $\pi$, 
and any cost function $c\in[0,1]$,  
we use $\xi_h^{\pi}(s,a)$ to denote the probability of visiting $(s,a)$ associated with $\widehat{SSP}(P,\pi)$. Suppose $V\in\R^{S+1}$ is any value function satisfying order property (where $V(g)=0$), i.e., $V(s)\geq \sum_{a}\pi(a|s)P_{s,a}V$ for all $s\in\mathcal{S}$, then we have 
\begin{equation}
    \sum_{h=0}^{\infty}\sum_{\substack{s,a \\ s\neq g}}\xi_h^{\pi}(s,a)\Var(P_{s,a},V)\leq2\norm{V}_{\infty}\sum_{\substack{s \\ s\neq g}}\xi_{0}(s)V(s)\leq2\norm{V}_{\infty}^2.
\end{equation}
\end{lemma}

\begin{proof}[Proof of Lemma~\ref{lem:bound_sum_var}]
\begin{align*}
    &\sum_{h=0}^{\infty}\sum_{\substack{s,a \\ s\neq g}}\xi_h^{\pi}(s,a)\Var(P_{s,a},V)=\sumSA\xi^\pi_{h}(s,a)\{P_{s,a}(V)^2-(P_{s,a}V)^2\}\\
    =&\sumS\xi_{h+1}^{\pi}(s)V^2(s)-\sumSA\xi^\pi_{h}(s,a)(P_{s,a}V)^2\\
    \leq&\sumS\xi^\pi_{h}(s)V^2(s)-\sumSA\xi^\pi_{h}(s,a)(P_{s,a}V)^2\\
    \explaineq{(i)}{=}&\sumS\xi^\pi_{h}(s)\{V^2(s)-\sum_{a}\pi(a|s)(P_{s,a}V)^2\}\\
    \explaineq{(ii)}{\leq}&\sumS\xi^\pi_{h}(s)\{V^2(s)-(\sum_{a}\pi(a|s)P_{s,a}V)^2\}\\
    {=}&\sumS\xi^\pi_{h}(s)\{(V(s)-\sum_{a}\pi(a|s)P_{s,a}V)(V(s)+\sum_{a}\pi(a|s)P_{s,a}V)\}\\
    \explaineq{(iii)}{\leq}&2\norm{V}_{\infty}\sumS\xi^\pi_{h}(s)(V(s)-\sum_{a}\pi(a|s)P_{s,a}V)\\
    =&2\norm{V}_{\infty}\left[\sumS\xi^\pi_{h}(s)V(s)-\sumS\xi_{h+1}^{\pi}(s)V(s)\right]\\
    =&2\norm{V}_{\infty}\sum_{\substack{s \\ s\neq g}}\xi_{0}(s)V(s)\leq2\norm{V}_{\infty}^2,
\end{align*}
where $(i)$ follows from the fact that $\xi(s,a)=\xi(s)\pi(a|s)$, $(ii)$ uses the Jensen's inequality and the fact that $f(x)=x^2$ is a convex function. $(iii)$ uses the ordering condition.
\end{proof}

\section{Crude evaluation bound}
\begin{theorem}\label{thm:crude_ope}
	
	Denote $d_m:=\min\{\sum_{h=1}^\infty \xi^\mu_h(s,a):s.t. \sum_{h=1}^\infty \xi^\mu_h(s,a)>0\}$, and $T^\pi_s$ to be the expected time to hit $g$ when starting from $s$. Define $\bar{T}^\pi=\max_{\bar{s}\in\mathcal{S}}T^\pi_{\bar{s}}$. Then when $n\geq \max\{\frac{49 S\iota}{9d_m}, 64(\bar{T}^\pi)^2\frac{S\iota}{d_m},C\cdot\log(SA/\delta)/\sum_{h=1}^\infty \xi_h^\mu(s,a)\}$, we have with probability $1-\delta$, (here $\iota=O(\log(SA/\delta)$)
	\[
	\norm{\widehat{V}^\pi-V^\pi}_\infty \leq O\left(\frac{\bar{T}^\pi\sqrt{\max_{s,a} \Var_c(s,a)}\iota+\sqrt{\bar{T}^\pi\norm{V^\pi}^2_\infty\iota}}{\sqrt{n\cdot d_m}}\right)+O\left(\frac{\bar{T}^\pi\norm{V^\pi}_\infty\cdot\iota}{n\cdot d_m}\right).
	\]
\end{theorem}
\begin{proof}
	
	We denote $\bar{s}\in\mathcal{S}$ to be any initial state. From Lemma \ref{lem:regret_decomp}, we have (here $\xi^\pi_{h,\bar{s}}(s,a)$ is the marginal state-action probability when starting from state $\bar{s}$ and following $\pi$)
	\begin{align*}
	|V^{\pi}(\bar{s}) - \hV^{\pi}(\bar{s})|&=\left|\sum_{h=1}^{\infty}\sum_{\substack{s,a \\ s\neq g}}\xi^\pi_{h,\bar{s}}(s,a)\{(\widehat{c}-c)(s,a)+(\tilde{P}_{s,a}-{P}_{s,a})\hV^{\pi}\}\right|\\
	&\leq\sum_{h=1}^{\infty}\sum_{\substack{s,a \\ s\neq g}}\xi^\pi_{h,\bar{s}}(s,a)|(\widehat{c}-c)(s,a)|+\sum_{h=1}^{\infty}\sum_{\substack{s,a \\ s\neq g}}\xi^\pi_{h,\bar{s}}(s,a)|(\tilde{P}_{s,a}-{P}_{s,a})(\hV^{\pi}-V^{\pi})|\\
	&\qquad+\sum_{h=1}^{\infty}\sum_{\substack{s,a \\ s\neq g}}\xi^\pi_{h,\bar{s}}(s,a)|(\tilde{P}_{s,a}-{P}_{s,a})V^{\pi}|\\
	\end{align*}
	We bound the above three parts one by one. First of all, by Bernstein inequality, Lemma~\ref{lem:Chern_SSP} and union bound, with probability $1-\delta$, 
	\begin{align*}
	\label{eqn:Delta_R}
	&\sum_{h=1}^{\infty}\sum_{\substack{s,a \\ s\neq g}}\xi^\pi_{h,\bar{s}}(s,a)|(\widehat{c}-c)(s,a)|\leq\sum_{h=1}^{\infty}\sum_{\substack{s,a \\ s\neq g}}\xi^\pi_{h,\bar{s}}(s,a)\left[2\sqrt{\frac{\mathrm{Var}_c(s,a)\iota}{n(s,a)}}+\frac{4\iota}{3n(s,a)}\right]\\
	&\explaineq{(i)}{\leq}\sum_{h=1}^{\infty}\sum_{\substack{s,a \\ s\neq g}}\xi^\pi_{h,\bar{s}}(s,a)\left[2\sqrt{\frac{2\mathrm{Var}_c(s,a)\iota}{n\sum_{h=1}^\infty\xi^\mu_h(s,a)}}+\frac{8\iota}{3n\sum_{h=1}^\infty\xi^\mu_h(s,a)}\right]\\
	&{\leq}\sum_{h=1}^{\infty}\sum_{\substack{s,a \\ s\neq g}}\xi^\pi_{h,\bar{s}}(s,a)\left[2\sqrt{\frac{2\max_{s,a}\mathrm{Var}_c(s,a)\iota}{n\cdot d_m}}+\frac{8\iota}{3n\cdot d_m}\right]\\
	&{\leq}T^\pi_{\bar{s}}\left[2\sqrt{\frac{2\max_{s,a}\mathrm{Var}_c(s,a)\iota}{n\cdot d_m}}+\frac{8\iota}{3n\cdot d_m}\right]
	\end{align*}
	$(i)$ uses Lemma~\ref{lem:Chern_SSP} and the last inequality uses Lemma~\ref{lem:propT}.
	
	For the second part, note 
	\begin{align*}
	&\sum_{h=1}^{\infty}\sum_{\substack{s,a \\ s\neq g}}\xi^\pi_{h,\bar{s}}(s,a)|(\tilde{P}_{s,a}-{P}_{s,a})(\hV^{\pi}-V^{\pi})|\\
	&\leq\sum_{h=1}^{\infty}\sum_{\substack{s,a \\ s\neq g}}\xi^\pi_{h,\bar{s}}(s,a)\left[\sqrt{\frac{2S\cdot\Var(P_{s,a},\hV^{\pi}-V^{\pi})\iota}{n(s,a)}}+\frac{4\norm{\hV^{\pi}-V^{\pi}}_{\infty}S\iota}{3n(s,a)}+\frac{\norm{\hV^{\pi}-V^{\pi}}_{\infty}}{n(s,a)+1}\right]\\
	&\leq\sum_{h=1}^{\infty}\sum_{\substack{s,a \\ s\neq g}}\xi^\pi_{h,\bar{s}}(s,a)\left[\sqrt{\frac{4S\cdot\Var(P_{s,a},\hV^{\pi}-V^{\pi})\iota}{n\sum_{h=1}^\infty \xi^\pi_h(s,a)}}+\frac{14\norm{\hV^{\pi}-V^{\pi}}_{\infty}S\iota}{3n\sum_{h=1}^\infty \xi^\pi_h(s,a)}\right]\\
	&{\leq}\sqrt{\sum_{\substack{s,a \\ s\neq g}}\frac{\sum_{h=0}^{\infty}\xi^\pi_{h,\bar{s}}(s,a)}{\sum_{h=0}^{\infty}\xi^\mu_h(s,a)}}\sqrt{4S\sum_{h=1}^{\infty}\sum_{\substack{s,a \\ s\neq g}}\xi_{h,\bar{s}}(s,a)\norm{\widehat{V}^\pi-V^\pi}^2\frac{\iota}{n}}+\frac{14\norm{\widehat{V}^\pi-V^\pi}_\infty S\iota}{3n}\sum_{s,a,s\neq g}\frac{d^\pi_{\bar{s}}(s,a)}{d^\mu(s,a)}\\
	&\leq\sqrt{\sum_{\substack{s,a \\ s\neq g}}\frac{d^\pi_{\bar{s}}(s,a)}{d^\mu(s,a)}\cdot 4ST^\pi_{\bar{s}}\norm{\widehat{V}^\pi-V^\pi}^2_\infty\cdot\frac{\iota}{n}}+\frac{14\norm{\widehat{V}^\pi-V^\pi}_\infty S\iota}{3n}\sum_{s,a,s\neq g}\frac{d^\pi_{\bar{s}}(s,a)}{d^\mu(s,a)}\\
	&\leq 4\sqrt{\sum_{\substack{s,a \\ s\neq g}}\frac{d^\pi_{\bar{s}}(s,a)}{d^\mu(s,a)}\cdot ST^\pi_{\bar{s}}\norm{\widehat{V}^\pi-V^\pi}^2_\infty\cdot\frac{\iota}{n}}\leq 4T^\pi_{\bar{s}}\sqrt{\frac{S\iota}{n\cdot d_m}}\norm{\hV^\pi-V^\pi}_\infty,
	\end{align*}
	where the first inequality uses Lemma~\ref{lemma:(P-hat_P)V2} and Lemma~\ref{lem:(tP-hP)}, the second inequality uses Lemma~\ref{lem:Chern_SSP}, the third inequality uses $\Var(\cdot)\leq\norm{\cdot}_\infty^2$, the fourth and fifth inequality use Lemma~\ref{lem:propT} and the last inequality follows from the condition $n\geq \frac{49 S\iota}{9d_m}\geq \frac{49 S\iota}{9T^\pi_{\bar{s}}}\sum_{s,a,s\neq g}\frac{d^\pi_{\bar{s}}(s,a)}{d^\mu(s,a)}$.
	
	For the third part, we have 

		\begin{align*}
		&\sum_{h=1}^{\infty}\sum_{\substack{s,a \\ s\neq g}}\xi^\pi_{h,\bar{s}}(s,a)|(\tilde{P}_{s,a}-{P}_{s,a})V^{\pi}|\\
		\leq&\sum_{h=1}^{\infty}\sum_{\substack{s,a \\ s\neq g}}\xi^\pi_{h,\bar{s}}(s,a)\left[2\sqrt{\frac{\Var(P_{s,a},V^{\pi})\iota}{n(s,a)}}+\frac{4\norm{V^{\pi}}_{\infty}\iota}{3n(s,a)}+\frac{\norm{V^{\pi}}_{\infty}}{n(s,a)}\right]\\
		\leq&\sum_{h=1}^{\infty}\sum_{\substack{s,a \\ s\neq g}}\xi^\pi_{h,\bar{s}}(s,a)\left[2\sqrt{\frac{2\Var(P_{s,a},V^{\pi})\iota}{n\sum_{h=1}^\infty \xi^\mu_h(s,a)}}+\frac{14\norm{V^{\pi}}_{\infty}\iota}{3n\sum_{h=1}^\infty \xi^\mu_h(s,a)}\right]\\
		=&\sqrt{ 8\sum_{\substack{s,a \\ s\neq g}}\frac{\sum_{h=1}^{\infty}\xi^\pi_{h,\bar{s}}(s,a)}{\sum_{h=1}^\infty \xi^\mu_{h}(s,a)}\cdot\sum_{s,a,s\neq g}\sum_{h=1}^\infty \xi^\pi_{h,\bar{s}}(s,a)\Var(P_{s,a},V^{\pi}) \frac{\iota}{n}}+\frac{14\norm{V^\pi}_\infty\iota}{3n}\cdot\sum_{s,a,s\neq g}\frac{d^\pi_{\bar{s}}(s,a)}{d^\mu(s,a)}\\
		\leq & \sqrt{ 8\frac{T^\pi_{\bar{s}}}{d_m}\cdot \frac{\norm{V^\pi}^2_\infty\iota}{n}}+\frac{14\norm{V^\pi}_\infty\iota}{3n}\cdot\frac{T^\pi_{\bar{s}}}{d_m}.\\
		\end{align*}
		where the first inequality uses Lemma~\ref{lem:(tP-hP)} and Bernstein inequality, the second inequality uses Lemma~\ref{lem:Chern_SSP} and the last one uses Lemma~\ref{lem:propT}. Recall $\bar{T}^\pi=\max_{\bar{s}\in\mathcal{S}}T^\pi_{\bar{s}}$, then combine all the three parts together and take the max over $\bar{s}$, we can derive
	{
		\begin{align*}
		\left(1-4\bar{T}^\pi\sqrt{\frac{S\iota}{n d_m}}\right)\norm{\hV^\pi-V^\pi}_\infty\leq& T^\pi_{\bar{s}}\left[2\sqrt{\frac{2\max_{s,a}\mathrm{Var}_c(s,a)\iota}{n\cdot d_m}}+\frac{8\iota}{3n\cdot d_m}\right]+\sqrt{ 8\frac{T^\pi_{\bar{s}}}{d_m}\cdot \frac{\norm{V^\pi}^2_\infty\iota}{n}}+\frac{14\norm{V^\pi}_\infty\iota}{3n}\frac{T^\pi_{\bar{s}}}{d_m}\\
		\leq &O\left(\frac{\bar{T}^\pi\sqrt{\max_{s,a} \Var_c(s,a)}\iota+\sqrt{\bar{T}^\pi\norm{V^\pi}^2_\infty\iota}}{\sqrt{n\cdot d_m}}\right)+O\left(\frac{\bar{T}^\pi\norm{V^\pi}_\infty\cdot\iota}{n\cdot d_m}\right),
		\end{align*}}
	therefore it implies (by applying the condition $n\geq 64(\bar{T}^\pi)^2\frac{S\iota}{d_m}$)
	\[
	\norm{\widehat{V}^\pi-V^\pi}_\infty \leq O\left(\frac{\bar{T}^\pi\sqrt{\max_{s,a} \Var_c(s,a)}\iota+\sqrt{\bar{T}^\pi\norm{V^\pi}^2_\infty\iota}}{\sqrt{n\cdot d_m}}\right)+O\left(\frac{\bar{T}^\pi\norm{V^\pi}_\infty\cdot\iota}{n\cdot d_m}\right)
	\]
\end{proof}

\section{Proof of Theorem~\ref{thm:ope}}\label{sec:proof_ope}

\begin{theorem}[Restatement of Theorem~\ref{thm:ope}]
	\label{thm:ope1}
	Denote $d_m:=\min\{\sum_{h=1}^\infty \xi^\mu_h(s,a):s.t. \sum_{h=1}^\infty \xi^\mu_h(s,a)>0\}$, and $T^\pi_s$ to be the expected time to hit $g$ when starting from $s$. Define $\bar{T}^\pi=\max_{\bar{s}\in\mathcal{S}}T^\pi_{\bar{s}}$. Then when $n\geq \max\{\frac{49 S\iota}{9d_m}, 64(\bar{T}^\pi)^2\frac{S\iota}{d_m},C\cdot\iota/d_m\}$, we have with probability $1-\delta$, 
	\[
	|V^{(i)}(s_\mathrm{init})-V^\pi(s_\mathrm{init})|
	\leq 4\sum_{s,a,s\neq g} d^\pi(s,a)\sqrt{\frac{2\Var_{P_{s,a}}[V^\pi+c]}{n\cdot d^\mu(s,a)}}
	+\widetilde{O}(\frac{1}{n})+\frac{\epsilon_{\mathrm{OPE}}}{1-\rho}.
	\]
	where the $\widetilde{O}$ absorbs Polylog term and even higher order term.
\end{theorem}
\begin{proof}
	Recall that we start from the initial state $s_\text{init}$. Then by Lemma~\ref{lem:diff_bound},
	\begin{equation}\label{eqn:main_first}
	\left|V^{\pi}(s_\text{init}) - V^{(i)}(s_\text{init})\right|
	\leq\left| V^{\pi}(s_\text{init})- \hV^{\pi}(s_\text{init})\right|+\left|\hV^{\pi}(s_\text{init})- V^{(i)}(s_\text{init})\right|\leq\left| V^{\pi}(s_\text{init})- \hV^{\pi}(s_\text{init})\right|+\frac{\epsilon_{\text{OPE}}}{1-\rho}, 
	\end{equation}
	it remains to bound $| V^{\pi}(s_\text{init})- \hV^{\pi}(s_\text{init})|$.
	
	From Lemma \ref{lem:regret_decomp}, we have
	\begin{align*}
	\left| V^{\pi}(s_\text{init})- \hV^{\pi}(s_\text{init})\right|&=\left|\sum_{h=1}^{\infty}\sum_{\substack{s,a \\ s\neq g}}\xi^\pi_h(s,a)\{(\widehat{c}-c)(s,a)+(\tilde{P}_{s,a}-{P}_{s,a})\hV^{\pi}\}\right|\\
	&\leq\sum_{h=1}^{\infty}\sum_{\substack{s,a \\ s\neq g}}\xi^\pi_h(s,a)|(\widehat{c}-c)(s,a)|+\sum_{h=1}^{\infty}\sum_{\substack{s,a \\ s\neq g}}\xi^\pi_h(s,a)|(\tilde{P}_{s,a}-{P}_{s,a})(\hV^{\pi}-V^{\pi})|\\
	&\qquad+\sum_{h=1}^{\infty}\sum_{\substack{s,a \\ s\neq g}}\xi^\pi_h(s,a)|(\tilde{P}_{s,a}-{P}_{s,a})V^{\pi}|
	\end{align*}
	We bound the above three parts one by one. First of all, by Bernstein inequality and Lemma~\ref{lem:Chern_SSP} together with union bound, with probability $1-\delta$, 
	\begin{equation}
	\begin{aligned}
	\label{eqn:Delta_R_E}
	&\sum_{h=1}^{\infty}\sum_{\substack{s,a \\ s\neq g}}\xi^\pi_h(s,a)|(\widehat{c}-c)(s,a)|\leq\sum_{h=1}^{\infty}\sum_{\substack{s,a \\ s\neq g}}\xi^\pi_h(s,a)\left[2\sqrt{\frac{\mathrm{Var}_c(s,a)\iota}{n(s,a)}}+\frac{4\iota}{3n(s,a)}\right]\\
	&\explaineq{(i)}{\leq}\sum_{h=1}^{\infty}\sum_{\substack{s,a \\ s\neq g}}\xi^\pi_h(s,a)\left[2\sqrt{\frac{2\mathrm{Var}_c(s,a)\iota}{n\sum_{h=1}^\infty\xi^\mu_h(s,a)}}+\frac{8\iota}{3n\sum_{h=1}^\infty\xi^\mu_h(s,a)}\right]\\
	&\explaineq{(ii)}{\leq}\sqrt{\sum_{\substack{s,a \\ s\neq g}}\frac{\sum_{h=1}^{\infty}\xi^\pi_h(s,a)}{\sum_{h=1}^{\infty}\xi^\mu_h(s,a)}}\sqrt{\sum_{\substack{s,a \\ s\neq g}}\sum_{h=1}^{\infty}\xi^\pi_h(s,a)\Var_c(s,a)\frac{\iota}{n}}+
	\left(\sum_{s,a,s\neq g}\frac{d^\pi(s,a)}{d^\mu(s,a)}\right)\frac{8\iota}{3n}\\
	&\explaineq{(iii)}{\leq}\sqrt{\sum_{\substack{s,a \\ s\neq g}}\frac{d^\pi(s,a)}{d^\mu(s,a)}\cdot\frac{V^\pi(s_\text{init})\cdot \iota}{n}}+
	\left(\sum_{s,a,s\neq g}\frac{d^\pi(s,a)}{d^\mu(s,a)}\right)\frac{8\iota}{3n},
	\end{aligned}
	\end{equation}
	(i) uses Lemma~\ref{lem:Chern_SSP} and (ii) uses Cauchy-Schwartz inequality. (iii) uses the fact that $\mathrm{Var}_c(s,a)\leq \E C(s,a)^2\leq c(s,a)$ since the realization $C(s,a)\in[0,1]$ and the definition of $V^\pi(s_\text{init})$. 
	
	For the second part, note 
	\begin{equation}
	\begin{aligned}
	\label{eqn:Delta_P(V_hat-V)}
	&\sum_{h=1}^{\infty}\sum_{\substack{s,a \\ s\neq g}}\xi_h(s,a)|(\tilde{P}_{s,a}-{P}_{s,a})(\hV^{\pi}-V^{\pi})|\\
	&\leq\sum_{h=1}^{\infty}\sum_{\substack{s,a \\ s\neq g}}\xi_h(s,a)\left[\sqrt{\frac{2S\cdot\Var(P_{s,a},\hV^{\pi}-V^{\pi})\iota_{s,a}}{n(s,a)}}+\frac{4\norm{\hV^{\pi}-V^{\pi}}_{\infty}S\iota_{s,a}}{3n(s,a)}+\frac{\norm{\hV^{\pi}-V^{\pi}}_{\infty}}{n(s,a)+1}\right]\\
	&\leq\sum_{h=1}^{\infty}\sum_{\substack{s,a \\ s\neq g}}\xi_h(s,a)\left[\sqrt{\frac{4S\cdot\Var(P_{s,a},\hV^{\pi}-V^{\pi})\iota_{s,a}}{n\sum_{h=1}^\infty \xi^\pi_h(s,a)}}+\frac{14\norm{\hV^{\pi}-V^{\pi}}_{\infty}S\iota_{s,a}}{3n\sum_{h=1}^\infty \xi^\pi_h(s,a)}\right]\\
	&{\leq}\sqrt{\sum_{\substack{s,a \\ s\neq g}}\frac{\sum_{h=0}^{\infty}\xi^\pi_h(s,a)}{\sum_{h=0}^{\infty}\xi^\mu_h(s,a)}}\sqrt{4S\sum_{h=1}^{\infty}\sum_{\substack{s,a \\ s\neq g}}\xi_h(s,a)\norm{\widehat{V}^\pi-V^\pi}^2\frac{\iota}{n}}+\frac{14\norm{\widehat{V}^\pi-V^\pi}_\infty S\iota}{3n}\sum_{s,a,s\neq g}\frac{d^\pi(s,a)}{d^\mu(s,a)}\\
	&\leq\sqrt{\sum_{\substack{s,a \\ s\neq g}}\frac{d^\pi(s,a)}{d^\mu(s,a)}\cdot 4ST^\pi\norm{\widehat{V}^\pi-V^\pi}^2_\infty\cdot\frac{\iota}{n}}+\frac{14\norm{\widehat{V}^\pi-V^\pi}_\infty S\iota}{3n}\sum_{s,a,s\neq g}\frac{d^\pi(s,a)}{d^\mu(s,a)}\\
	&\leq 4\sqrt{\sum_{\substack{s,a \\ s\neq g}}\frac{d^\pi(s,a)}{d^\mu(s,a)}\cdot ST^\pi\norm{\widehat{V}^\pi-V^\pi}^2_\infty\cdot\frac{\iota}{n}}
	\end{aligned}
	\end{equation}
	where the first inequality uses Lemma~\ref{lemma:(P-hat_P)V2} and Lemma~\ref{lem:(tP-hP)}, the second inequality uses Lemma~\ref{lem:Chern_SSP}, the third inequality uses $\Var(\cdot)\leq\norm{\cdot}_\infty^2$ and CS inequality, the fourth inequality use Lemma~\ref{lem:propT} and the last inequality follows from the condition $n\geq \frac{49 S\iota}{9d_m}\geq \frac{49 S\iota}{9T^\pi}\sum_{s,a,s\neq g}\frac{d^\pi(s,a)}{d^\mu(s,a)}$.
	
	For the third part, we have 
	\begin{equation}
	\begin{aligned}
	\label{eqn:Delta_PV}
	&\sum_{h=1}^{\infty}\sum_{\substack{s,a \\ s\neq g}}\xi^\pi_h(s,a)|(\widetilde{P}_{s,a}-{P}_{s,a})V^{\pi}|\\
	\leq&\sum_{h=1}^{\infty}\sum_{\substack{s,a \\ s\neq g}}\xi^\pi_h(s,a)\left[2\sqrt{\frac{\Var(P_{s,a},V^{\pi})\iota}{n(s,a)}}+\frac{4\norm{V^{\pi}}_{\infty}\iota}{3n(s,a)}+\frac{\norm{V^{\pi}}_{\infty}}{n(s,a)}\right]\\
	\leq&\sum_{h=1}^{\infty}\sum_{\substack{s,a \\ s\neq g}}\xi^\pi_h(s,a)\left[2\sqrt{\frac{2\Var(P_{s,a},V^{\pi})\iota}{n\sum_{h=1}^\infty \xi^\mu_h(s,a)}}+\frac{14\norm{V^{\pi}}_{\infty}\iota}{3n\sum_{h=1}^\infty \xi^\mu_h(s,a)}\right]\\
	\huge(\leq&\sqrt{ 8\sum_{\substack{s,a \\ s\neq g}}\frac{\sum_{h=1}^{\infty}\xi^\pi_h(s,a)}{\sum_{h=1}^\infty \xi^\mu_h(s,a)}\cdot\sum_{s,a,s\neq g}\sum_{h=1}^\infty \xi^\pi_h(s,a)\Var(P_{s,a},V^{\pi}) \frac{\iota}{n}}+\frac{14\norm{V^\pi}_\infty\iota}{3n}\cdot\sum_{s,a,s\neq g}\frac{d^\pi(s,a)}{d^\mu(s,a)}\\
	\leq & \sqrt{ 8\sum_{\substack{s,a \\ s\neq g}}\frac{d^\pi(s,a)}{d^\mu(s,a)}\cdot \frac{\norm{V^\pi}^2_\infty\iota}{n}}+\frac{14\norm{V^\pi}_\infty\iota}{3n}\cdot\sum_{s,a,s\neq g}\frac{d^\pi(s,a)}{d^\mu(s,a)}\\
	\leq&4\sqrt{2\sum_{\substack{s,a \\ s\neq g}}\frac{d^\pi(s,a)}{d^\mu(s,a)}\cdot \frac{\norm{V^\pi}^2_\infty\iota}{n}}\huge),
	\end{aligned}
	\end{equation}
	where the first inequality uses Lemma~\ref{lemma:(P-hat_P)V2} and Lemma~\ref{lem:(tP-hP)}, the second inequality uses Lemma~\ref{lem:Chern_SSP}. The third inequality uses the Cauchy-Schwartz inequality.
	
	Combine Equation (\ref{eqn:Delta_R_E}), (\ref{eqn:Delta_P(V_hat-V)}) and (\ref{eqn:Delta_PV}) together, we obtain
	\begin{align*}
	&\left| V^{\pi}(s_\text{init})- \hV^{\pi}(s_\text{init})\right|\leq \sum_{h=1}^{\infty}\sum_{\substack{s,a \\ s\neq g}}\xi^\pi_h(s,a)\left[2\sqrt{\frac{2\Var(P_{s,a},V^{\pi})\iota}{n\sum_{h=1}^\infty \xi^\mu_h(s,a)}}+\frac{14\norm{V^{\pi}}_{\infty}\iota}{3n\sum_{h=1}^\infty \xi^\mu_h(s,a)}\right]\\
	+&\sum_{h=1}^{\infty}\sum_{\substack{s,a \\ s\neq g}}\xi^\pi_h(s,a)\left[2\sqrt{\frac{2\mathrm{Var}_c(s,a)\iota}{n\sum_{h=1}^\infty\xi^\mu_h(s,a)}}+\frac{8\iota}{3n\sum_{h=1}^\infty\xi^\mu_h(s,a)}\right]\\
	+&4\sqrt{\sum_{\substack{s,a \\ s\neq g}}\frac{d^\pi(s,a)}{d^\mu(s,a)}\cdot ST^\pi\norm{\widehat{V}^\pi-V^\pi}^2_\infty\cdot\frac{\iota}{n}}\\
	\leq&4\sum_{h=1}^{\infty}\sum_{\substack{s,a \\ s\neq g}}\xi^\pi_h(s,a)\sqrt{\frac{2[\Var(P_{s,a},V^{\pi})+\Var_c(s,a)]\iota}{n\sum_{h=1}^\infty \xi^\mu_h(s,a)}}+\sum_{\substack{s,a \\ s\neq g}}\frac{22\sum_{h=1}^{\infty}\xi^\pi_h(s,a)\norm{V^{\pi}}_{\infty}\iota}{3n\sum_{h=1}^\infty \xi^\mu_h(s,a)}\\
	+&4\sqrt{\sum_{\substack{s,a \\ s\neq g}}\frac{d^\pi(s,a)}{d^\mu(s,a)}\cdot ST^\pi\cdot \frac{(\bar{T}^{\pi})^2\cdot\max_{s,a} \Var_c(s,a)\iota+\bar{T}^\pi\norm{V^\pi}^2_\infty\iota}{n\cdot d_m}\cdot\frac{\iota}{n}}+\widetilde{O}\left(\frac{1}{n^{3/2}}\right)\\
	=&4\sum_{s,a,s\neq g} d^\pi(s,a)\sqrt{\frac{2\Var_{P_{s,a}}[V^\pi+c]}{n d^\mu(s,a)}}+\widetilde{O}(\frac{1}{n})
	\end{align*}
	where the only inequality uses Theorem~\ref{thm:crude_ope}. Combining this with \eqref{eqn:main_first} we finish the proof of Theorem~\ref{thm:ope}.
\end{proof}


\section{Preparations for Proving Offline Learning SSP}

Throughout the whole section, we denote $\iota=O(\log(SA/\delta))$. All the results apply to the construction of Algorithm~\ref{alg:OPO}. In particular, we use $\bar{V}$ to denote the limit of $V^{(i)}$ (by letting $\epsilon_{\text{OPL}}=0$). This limit exists, as guaranteed by Lemma~\ref{lem:contraction2}.

\subsection{Auxiliary Lemmas}

\begin{lemma}
\label{lem:bound_Vbar}
Denote the limit of sequence $V^{(i)}$ in Algorithm~\ref{alg:OPO}  as $\bar{V}$, we have that
\begin{align*}
\norm{\bar{V}}_{\infty}\leq\tB
\end{align*}
\end{lemma}

\begin{proof}
First of all, by Lemma~\ref{lem:contraction2}, we know $\bar{V}$ exists. Next, from the Algorithm~\ref{alg:OPO}, we can get that
\begin{align*}
Q^{(i+1)}(s,a)=\min\{\widehat{c}(s,a)+\widetilde{P}'_{s,a}V^{(i)}+b_{s,a}(V^{(i)})\text{ , }\tB\}\leq\tB\qquad\qquad\forall (s,a)\in\mathcal{S}\times\mathcal{A},\forall i\in\mathbb{N}
\end{align*}
and thus
\begin{align*}
V^{(i+1)}(s)=\min_a Q^{(i+1)}(s,a)\leq\tB\qquad\qquad\forall s\in\mathcal{S}\times\mathcal{A},\forall i\in\mathbb{N}.
\end{align*}
It implies that $\bar{V}(s)=\lim_{i\rightarrow\infty}V^{(i)}(s)\leq\tB$.
\end{proof}

\begin{lemma}
	\label{lem:(tP-hP)PO}
	For any ${V}(\cdot)\in\mathbb{R}^{S}$ satisfying $V(g)=0$,
	\begin{align}
	|(\widetilde{P}'_{s,a}-\widehat{P}_{s,a})V|\leq\frac{||V||_{\infty}}{n_{\text{max}}+1}, \qquad |\Var(\tP'_{s,a},V)-\Var(\hP_{s,a},V)|\leq\frac{2\norm{V}_{\infty}^2}{n_{\text{max}}+1}.
	\end{align}
\end{lemma}
\begin{proof}
	The proof is similar to Lemma~12 in \cite{tarbouriech2021stochastic}. We include the proof for completeness. Since $V(g)=0$, for all state $s\neq g$, we have $\widetilde{P}'_{s,a}V=\sum_{s',s'\neq g}(\frac{n_{max}}{n_{max}+1})\hP(s'|s,a)V(s')=(\frac{n_{max}}{n_{max}+1})\hP_{s,a}V$
	\begin{align*}
	|(\widetilde{P}'_{s,a}-\widehat{P}_{s,a})V|&=|(\frac{n_{max}}{n_{max}+1})\hP_{s,a}V-\hP_{s,a}V|\\
	&=\frac{|\sum_{s',s'\neq g}\hP(s'|s,a)V(s')|}{n_{{max}}+1}\leq\frac{||V||_{\infty}}{n_{{max}}+1}.
	\end{align*}
	Then we prove the second inequality. Similarly,
	\begin{align*}
	|\Var(\tP'_{s,a},V)-\Var(\hP_{s,a},V)|&=|\tP'_{s,a}V^2-(\tP'_{s,a}V)^2-\hP_{s,a}V^2+(\hP_{s,a}V)^2|\\
	&=|(\frac{n_{max}}{n_{max}+1})\hP_{s,a}V^2-(\frac{n_{max}}{n_{max}+1}\hP_{s,a}V)^2-\hP_{s,a}V^2+(\hP_{s,a}V)^2|\\
	&=|-(\frac{1}{n_{max}+1})\{\hP_{s,a}V^2-(\hP_{s,a}V)^2\}+\frac{n_{max}}{(n_{max}+1)^2}(\hP_{s,a}V)^2|\\
	&=(\frac{1}{n_{max}+1})|\Var(\hP_{s,a},V)|+|\frac{n_{max}}{(n_{max}+1)^2}(\hP_{s,a}V)^2|\leq\frac{2\norm{V}_{\infty}^2}{n_{{max}}+1}.
	\end{align*}
\end{proof}

\begin{lemma}
	\label{lem:(P-tP)barV}
	Let ${T}_{\max}=\max_i T_i$ and $n>O( {T}_{\max}^2\log(SA/\delta)/d_m^2)$. If in addition $n\geq\frac{S\iota}{2d_m}$, with probability at least $1-\delta$, we have that for any state action pair (s,a), 
	\begin{align*}
	|(P_{s,a}-\tP'_{s,a})\bar{V}|
	&\leq\frac{\tB}{n(s,a)}+\frac{16\tB\iota}{3n(s,a)}+2\sqrt{\frac{\Var(\tP',\bar{V})\iota}{n(s,a)}}+6\sqrt{\frac{S\iota}{n(s,a)}}\norm{\bar{V}-V^\star}_{\infty}
	\end{align*}
\end{lemma}
\begin{proof}
	First, we can bound term $(P_{s,a}-\tP'_{s,a})\bar{V}$.
	\begin{align}
	|(P_{s,a}-\tP'_{s,a})\bar{V}|\leq|(\hP_{s,a}-\tP'_{s,a})\bar{V}|+|(P_{s,a}-\hP_{s,a})(\bar{V}-V^\star)|+|(P_{s,a}-\hP_{s,a})V^\star|
	\end{align}
	Then we bound the above three terms one by one. From Lemma~\ref{lem:bound_Vbar} and Lemma~\ref{lem:(tP-hP)PO}, we have
	\begin{align}
	|(\hP_{s,a}-\tP_{s,a})\bar{V}|\leq\frac{\norm{\bar{V}}_{\infty}}{n_{max}+1}\leq\frac{\tB}{n_{max}+1}.
	\end{align}
	For the second term, we have
	\begin{align}
	|(P_{s,a}-\hP_{s,a})(\bar{V}-V^\star)|&\leq\sqrt{\frac{2S\Var(P_{s,a},\bar{V}-V^\star)\iota}{n(s,a)}}+\frac{2\norm{\bar{V}-V^\star}_{\infty}S\iota}{3n(s,a)}\notag\\
	&\leq\sqrt{\frac{2S\iota}{n(s,a)}}\norm{\bar{V}-V^\star}_{\infty}+\frac{\norm{\bar{V}-V^\star}_{\infty}S\iota}{n(s,a)},
	\end{align}
	where the first inequality holds because of lemma~\ref{lemma:(P-hat_P)V2}.
	For the last term, we have that
	\begin{align}
	|(P_{s,a}-\hP_{s,a})V^\star|&\explaineq{(i)}{\leq}\sqrt{\frac{2\Var(\widehat{P}_{s,a},V^\star)\iota}{n(s,a)}}+\frac{7||V^\star||_{\infty}\iota}{3n(s,a)}\notag\\
	&\explaineq{(ii)}{\leq}2\sqrt{\frac{\Var(\widehat{P}_{s,a},V^\star-\bar{V})\iota}{n(s,a)}}+2\sqrt{\frac{\Var(\widehat{P}_{s,a},\bar{V})\iota}{n(s,a)}}+\frac{7\tB\iota}{3n(s,a)}\notag\\
	&\leq2\sqrt{\frac{\iota}{n(s,a)}})\norm{\bar{V}-V^\star}_{\infty}+2\sqrt{\frac{\Var(\widehat{P}_{s,a},\bar{V})\iota}{n(s,a)}}+\frac{7\tB\iota}{3n(s,a)}\notag\\
	&\explaineq{(iii)}{\leq}2\sqrt{\frac{\iota}{n(s,a)}})\norm{\bar{V}-V^\star}_{\infty}+2\sqrt{\frac{\Var(\tP'_{s,a},\bar{V})\iota}{n(s,a)}}+\frac{2\sqrt{2\iota}||\bar{V}||_{\infty}}{n(s,a)}+\frac{7\tB\iota}{3n(s,a)}\notag\\
	&\leq2\sqrt{\frac{\iota}{n(s,a)}})\norm{\bar{V}-V^\star}_{\infty}+2\sqrt{\frac{\Var(\tP'_{s,a},\bar{V})\iota}{n(s,a)}}+\frac{16\tB\iota}{3n(s,a)},
	\end{align}
	where (i) holds under event $\mathcal{E}_3$. (ii) holds because of $\Var(X+Y)\leq2\Var(X)+2\Var(Y)$. (iii) comes from Lemma~\ref{lem:(tP-hP)PO}. Both (ii) and (iii) uses the result that $\sqrt{a+b}\leq\sqrt{a}+\sqrt{b}$ when $a\geq0$ and $b\geq0$. 
	Combine the above inequalities together, we can get
	\begin{align*}
	(P_{s,a}-\tP'_{s,a})\bar{V}
	&\leq\frac{\tB}{n(s,a)}+\frac{16\tB\iota}{3n(s,a)}+2\sqrt{\frac{\Var(\tP',\bar{V})\iota}{n(s,a)}}+(\sqrt{\frac{2S\iota}{n(s,a)}}+\frac{S\iota}{n(s,a)}+2\sqrt{\frac{\iota}{n(s,a)}})\norm{\bar{V}-V^\star}_{\infty}.
	\end{align*}
	Since with probability $1-\delta$, by Lemma~\ref{lem:chernoff_multiplicative} $n(s,a)\geq\frac{1}{2}nd_m$. When $n\geq\frac{S\iota}{2d_m}$, we have
	\begin{align*}
	(\sqrt{\frac{2S\iota}{n(s,a)}}+\frac{S\iota}{n(s,a)}+2\sqrt{\frac{\iota}{n(s,a)}})\norm{\bar{V}-V^\star}_{\infty}&\leq(\sqrt{2}+2+2)\sqrt{\frac{S\iota}{n(s,a)}}\norm{\bar{V}-V^\star}_{\infty}\\
	&\leq6\sqrt{\frac{S\iota}{n(s,a)}}\norm{\bar{V}-V^\star}_{\infty}\\
	\end{align*}
\end{proof}

\begin{lemma}
	\label{lem:derivativeF}
	Define function $f:\R^{S'}\times\R^{S'}\times \R\rightarrow \R$ as $f(p,v,n)=pv+\max\{2\sqrt{\frac{\Var(p,v)\iota}{n}},4\frac{\tB\iota}{n}\}$, if $\norm{v}_\infty\leq \tB$ and $v(g)=0$, then we have $(\frac{\partial f}{\partial v})(s)\geq0$ and $\sum_{s,s\neq g}(\frac{\partial f}{\partial v})(s)\leq1-p(g)^2$.
\end{lemma}

\begin{proof}
	\begin{align*}
	(\frac{\partial f}{\partial v})(s)&=p(s)+\mathbb{I}\{2\sqrt{\frac{\Var(p,v)\iota}{n}}\geq4\frac{\tB\iota}{n}\}2\sqrt{\frac{\iota}{n}}\frac{\partial (\sqrt{\Var(p,v)})}{\partial v(s)}\\
	&=p(s)+\mathbb{I}\{2\sqrt{\frac{\Var(p,v)\iota}{n}}\geq4\frac{\tB\iota}{n}\}2\sqrt{\frac{\iota}{n}}\frac{p(s)(v(s)-pv)}{\sqrt{\Var(p,v)}}
	\end{align*}
	Simplifying the above equation, we can get
	\begin{align}
	\label{eqn:gamma}
	(\frac{\partial f}{\partial v})(s)
	&\geq \min\{p(s),p(s)-\frac{p(s)(pv-v(s))}{\tB}\}\notag\\
	\end{align}
	Since $|pv-v(s)|\leq \tB$, we have $p(s)-\frac{p(s)(pv-v(s))}{\tB}\geq0$. Then we would have $(\frac{\partial f}{\partial v})(s)\geq0$.
	For the second part, we have
	\begin{enumerate}
		\item Case I: $2\sqrt{\frac{\Var(p,v)\iota}{n}}\geq4\frac{\tB\iota}{n}$, we have
		\begin{align*}
		\sum_{s,s\neq g}(\frac{\partial f}{\partial v})(s)&=\sum_{s,s\neq g}\{p(s)+2\sqrt{\frac{\iota}{n\Var(p,v)}}p(s)(v(s)-pv)\}\\
		&\leq\sum_{s,s\neq g}p(s)+2\sqrt{\frac{\iota}{n\Var(p,v)}}\{\sum_{s,s\neq g}p(s)v(s)-\sum_{s,s\neq g}p(s)(\sum_{s,s'\neq g}p(s')v(s'))\}\\
		&\leq\sum_{s,s\neq g}p(s)+2\sqrt{\frac{\iota}{n\Var(p,v)}}[\sum_{s,s\neq g}p(s)v(s)](1-\sum_{s,s\neq g}p(s))\\
		&\leq\sum_{s,s\neq g}p(s)+\frac{[\sum_{s,s\neq g}p(s)v(s)](1-\sum_{s,s\neq g}p(s))}{\tB}\\
		&\leq\sum_{s,s\neq g}p(s)+(\sum_{s,s\neq g}p(s))(1-\sum_{s,s\neq g}p(s))=1-(p(g))^2
		\end{align*}
		\item Case II: $2\sqrt{\frac{\Var(p,v)\iota}{n}}\leq4\frac{\tB\iota}{n}$, we have
		\begin{align*}
		\sum_{s,s\neq g}(\frac{\partial f}{\partial v})(s)&=\sum_{s,s\neq g}p(s)=1-p(g)\leq1-p(g)^2
		\end{align*}
	\end{enumerate}
	
	Combine this inequality with (\ref{eqn:gamma}), we complete the proof.
\end{proof}

\begin{lemma}
	\label{lem:piProper}
	When $n\geq \max\{\frac{4B_\star-2c_\text{min}}{c_\text{min}d_{max}},\frac{26^2\times 2S\iota(\bar{T}^\star)^2}{d_m}, \frac{10^6(\sqrt{\tB}+1)^4S\iota\bar{T}^\star\widetilde{T}}{B_\star (\sqrt{B_\star }+1)^2d_m},O( {T}_{\max}^2\log(SA/\delta)/d_m^2)\}$, $\bar{\pi}$ is a proper policy (Recall $\bar{\pi}$ is the output of Algorithm~\ref{alg:OPO}).
	\begin{proof}
		By definition we need to show that $T^{\bar{\pi}}(s)<\infty$ for any $s\in \mathcal{S}$. We prove this by contradiction. Suppose $T^{\bar{\pi}}(s)=\infty$, then we have that there exists at least one state $e$ such that the expected visiting times of state $e$ is infinite, i.e., $\exists e\in\mathcal{S}$, such that $T^{\bar{\pi}}_{e,e}=\infty$. In this case, $e$ is a (positive) recurrent state in the finite Markov Chain induced by policy $\bar{\pi}$. Denote the communication class which $e$ belongs as $\mathcal{S}_0$. Since the state space is finite, we have that every state in the communication class $\mathcal{S}_0$ is recurrent. From the finite Markov Chain theory, we know that the communication class $\mathcal{S}_0$ is closed. In other words, $\forall  x\in\mathcal{S}_0$, and $\forall  y\in\mathcal{S}\backslash\mathcal{S}_0$, we have $P(y|x,\bar{\pi}(x))=0$. Thus with probability $1$, we have $\sum_{i=1}^n\sum_{j=1}^{T_i}\mathbb{I}(s_j^{(i)}=x\text{, }a_j^{(i)}=\bar{\pi}(x)\text{, }s_{j+1}^{(i)}=y)=0$. This implies that $\widehat{P}(y|x,\bar{\pi}(x))=\frac{\sum_{i=1}^n\sum_{j=1}^{T_i}\mathbb{I}(s_j^{(i)}=x\text{, }a_j^{(i)}=\bar{\pi}(x)\text{, }s_{j+1}^{(i)}=y)}{n(x,\bar{\pi}(x))}=0$. By definition of the estimated transition matrix $\tP'$, we have $\tP'(g|x,\bar{\pi}(x))=\frac{1}{n_{max}+1}$. It follows that
		\begin{align}
		\sum_{h=0}^{\infty}\sum_{s\in\mathcal{S}_0}\tilde{\xi}_{h,e}(s)=\sum_{h=0}^\infty(\frac{n_{max}}{n_{max}+1})^h=n_{max}+1.
		\end{align}
		Then we have
		\begin{align}
		\sum_{h=0}^{\infty}\sum_{s\in\mathcal{S}_0}\tilde{\xi}_{h,e}(s)=n_{max}+1\geq\half n{d_{max}}+1,
		\end{align}
		where the last inequality holds with probability $1-\delta$ by Lemma~\ref{lem:Chern_SSP}.
		Define $\widetilde{V}^{\bar{\pi}}(e)=\sum_{h=0}^\infty\E_{\tP',\bar{\pi}}[\hat{c}(s_h,a_h)|s_0=e]$, then
		\begin{align*}
		\widetilde{V}^{\bar{\pi}}(e)&\geq\sum_{h=0}^\infty\E_{\tP',\bar{\pi}}[c_{min}]\\
		&\geq\sum_{h=0}^\infty\E_{\tP',\bar{\pi}}[c_{min}\mathbb{I}(s_h\in\mathcal{S}_0)]=c_{min}\sum_{h=0}\sum_{s\in\mathcal{S}_0}\tilde{\xi}_{h,e}(s)\geq c_{min}(\half n{d_{max}}+1).
		\end{align*}
		Apply Lemma \ref{lem:general_Bellman} to the SSP problem with $M:=\left\langle\mathcal{S}, \mathcal{A}, \tP', \hat{c}, e, g\right\rangle$, we can get
		\begin{align*}
		\widetilde{V}^{\bar{\pi}}(s)=\tP'_{s,\bar{\pi}(s)}\widetilde{V}^{\bar{\pi}}+\hat{c}(s,\bar{\pi}(s)):=\tT'\widetilde{V}^{\bar{\pi}}(s)
		\end{align*}
		Since $\bar{V}(s)=\tP'_{s,\bar{\pi}(s)}\bar{V}+\hat{c}(s,\bar{\pi}(s))+b_{s,\bar{\pi}(s)}(\bar{V})=\tT\bar{V}(s)$, from Lemma \ref{lem:compare_T} we have $\bar{V}(s)\geq\widetilde{V}^{\bar{\pi}}(s)$ (since $b(\bar{V})$ is non-negative and both $\tT,\tT'$ are monotone operators). Then we get
		\begin{align}
		\label{eqn:barV_geqtV}
		\bar{V}(e)\geq\widetilde{V}^{\bar{\pi}}(e)\geq c_{min}(\half n{d_{max}}+1).
		\end{align}
		From Lemma~\ref{thm:crude_PO} (note Lemma~\ref{thm:crude_PO} only bounds $\bar{V}$ and $V^\star$ and has nothing to do with $\bar{\pi}$), we have that with probability $1-\delta$, when $n\geq \max\{\frac{26^2\times 2S\iota(\bar{T}^\star)^2}{d_m}, \frac{10^6(\sqrt{\tB}+1)^4S\iota\bar{T}^\star\widetilde{T}}{B_\star (\sqrt{B_\star }+1)^2d_m}\}$, which implies $n\geq\frac{900\bar{T}^\star \iota(\sqrt{B_\star}+1)^2}{B_\star d_m}$, we have $\bar{V}(e)\leq V^\star(e)+B_\star\leq 2B_\star$. Combine this inequality with (\ref{eqn:barV_geqtV}), we can get $n\leq \frac{4B_\star-2c_\text{min}}{c_\text{min}d_{max}}$, which contradicts with the assumptions in the lemma.
	\end{proof}
\end{lemma}


\subsection{Convergence of Pessimistic value iteration in Algorithm~\ref{alg:OPO}}\label{sec:converge_OPO}

Define the operator $\tT$ as $\tT(V)(s)=\min_{a}\left\{\min\{\widehat{c}(s,a)+\widetilde{P}_{s,a}V+b_{s,a}(V)\text{ , }\tB\}\right\}$. First, we prove that $\widetilde{\mathcal{T}}$ is a contraction mapping.
\begin{lemma}
\label{lem:contraction2}
$\widetilde{\mathcal{T}}:\mathbb{R}^\mathcal{S}\times\{0\}\rightarrow \mathbb{R}^\mathcal{S}\times\{0\}$ is a contraction mapping, i.e., $\forall V_1,V_2\in\mathbb{R}^\mathcal{S'}$, $V_1(g)=V_2(g)=0$, we have
\begin{align}
    \norm{\widetilde{\mathcal{T}}V_1-\widetilde{\mathcal{T}}V_2}_\infty\leq \gamma \norm{V_1-V_2}_\infty,
\end{align}
where $\gamma:=1-\frac{1}{(1+\max_{s,a}n(s,a))^2}$.
\end{lemma}
\begin{proof}[Proof of Lemma~\ref{lem:contraction2}]
First, we prove the result for $s=g$. Since $b_{g,a}(V)=0$, then we have $\tT(V)(g)=0$ and thus $\tT(V_1)(g)-\tT(V_2)(g)=0$. When $s\neq g$, we have
\begin{align*}
    |\widetilde{\mathcal{T}}V_1(s)-\widetilde{\mathcal{T}}V_2(s)|&\leq\max_{a}|\min\{\widehat{c}(s,a)+\tP'_{s,a}V_1+b_{s,a}(V_1)\text{ , }\tB\}-\min\{\widehat{c}(s,a)+\tP'_{s,a}V_2+b_{s,a}(V_2)\text{ , }\tB\}|\\
    &\explaineq{(i)}{\leq}\max_{a}|\{\widehat{c}(s,a)+\tP'_{s,a}V_1+b_{s,a}(V_1)\}-\{\widehat{c}(s,a)+\tP'_{s,a}V_2+b_{s,a}(V_2)\}|\\
    &=\max_{a}|f(\tP'_{s,a},V_1)-f(\tP'_{s,a},V_2)|\\
    &\explaineq{(ii)}{=}\max_{a}\sum_{s}|\langle(\frac{\partial f}{\partial v})(\tP'_{s,a},\theta (V_1-V_2)+V_2),V_1-V_2\rangle|\\
    &\leq\max_{a}(\sum_{s',s\neq g}|(\frac{\partial f}{\partial v(s')})(\tP'_{s,a},\theta (V_1-V_2)+V_2)|)\norm{V_1-V_2}_\infty\\
    &\explaineq{(iii)}{\leq}\max_a\{1-\tP'_{s,a}(g)^2\}\norm{V_1-V_2}_\infty,
\end{align*}
where (i) comes from Lemma~\ref{lem:minAB}. (ii) is due to the mean value theorem. (iii) uses the result in Lemma~\ref{lem:derivativeF}. Then we have
\begin{align*}
    \norm{\widetilde{\mathcal{T}}V_1(s)-\widetilde{\mathcal{T}}V_2(s)}_{\infty}\leq\max_{s,a}\{1-\tP'_{s,a}(g)^2\}\norm{V_1-V_2}_\infty\leq\{1-\frac{1}{(1+\max_{s,a}n(s,a))^2}\}\norm{V_1-V_2}_\infty
\end{align*}
\end{proof}

We then introduce the following two regret decomposition lemma.

\subsection{Regret decomposition lemma for policy optimization}
\begin{lemma} Suppose $\bar{V}$ is the limit of the sequence $V^{(i)}$ in Algorithm \ref{alg:OPO}, we have the following decomposition lemma.
\label{lem:bar_V-V*}
    \begin{align}
        \bar{V}-V^\star\leq\sumS \xi^\star_h(s)\{(\tP'_{s,\pi^\star(s)}-P_{s,\pi^\star(s)})\bar{V}+\hat{c}(s,\pi^\star(s))-c(s,\pi^\star(s))+b_{s,\pi^\star(s)}(\bar{V})\}
    \end{align}
\end{lemma}
\begin{proof}
\begin{align*}
    \bar{V}-V^\star=\sum_{s,s\neq g}\xi_0^\star(s)(\bar{V}(s)-V^\star(s))
\end{align*}
Since for any $h\in\textbf{N}$, we have
\begin{align*}
    \sum_{s,s\neq g}\xi_h^\star(s)(\bar{V}(s)-V^\star(s))&\explaineq{(i)}{\leq}\sum_{s,s\neq g}\xi_h^\star(s)(\bar{Q}(s,\pi^\star(s))-Q^\star(s,\pi^\star(s))\\
    &\explaineq{(ii)}{=}\sum_{s,s\neq g}\xi_h^\star(s)\{\tP'_{s,\pi^\star(s)}\bar{V}-P_{s,\pi^\star(s)}V^\star+\hat{c}(s,\pi^\star(s))-c(s,\pi^\star(s))+b_{s,\pi^\star(s)}(\bar{V})\}\\
    &{=}\sum_{s,s\neq g}\xi_h^\star(s)\{(\tP'_{s,\pi^\star(s)}-P_{s,\pi^\star(s)})\bar{V}+\hat{c}(s,\pi^\star(s))-c(s,\pi^\star(s))
   +b_{s,\pi^\star(s)}(\bar{V})\}\\ &\qquad+\sum_{s,s\neq g}\xi_h^\star(s)P_{s,\pi^\star(s)}(\bar{V}-V^\star)\\
   &{=}\sum_{s,s\neq g}\xi_h^\star(s)\{(\tP'_{s,\pi^\star(s)}-P_{s,\pi^\star(s)})\bar{V}+\hat{c}(s,\pi^\star(s))-c(s,\pi^\star(s))
   +b_{s,\pi^\star(s)}(\bar{V})\}\\ &\qquad+\sum_{s,s\neq g}\xi_{h+1}^\star(s)(\bar{V}-V^\star)(s)
\end{align*}
(i) follows from the fact that $\bar{V}(s)=\min_{a}\bar{Q}(s,a)\leq\bar{Q}(s,\pi^\star(s))$. (ii) uses the fact that $\bar{V}$ is the limit of $V^{(i)}$ and we have $\bar{Q}(s,a)=\hat{c}(s,a)+\tP'_{s,a}\bar{V}+b_{s,a}(\bar{V})$. By recursion over time step $h$, we can get
\begin{align}
        \bar{V}-V^\star&\leq\sum_{h=0}^H\sum_{s,s\neq g} \xi^\star_h(s)\{(\tP'_{s,\pi^\star(s)}-P_{s,\pi^\star(s)})\bar{V}+\hat{c}(s,\pi^\star(s))-c(s,\pi^\star(s))+b_{s,\pi^\star(s)}(\bar{V})\}\notag\\
        &\qquad+\sum_{s,s\neq g}\xi_{H+1}^\star(s)(\bar{V}-V^\star),
    \end{align}
Since $\pi^\star$ is a proper policy, we have that for any $s\neq g$, $\lim_{H\rightarrow\infty}\xi_{H+1}^\star(s)=0$. Also, $\bar{V}$ is bounded by $\widetilde{B}$ and $V^\star$ is bounded by $B_\star$. Thus let $H$ goes to infinity, we can complete the proof.
\end{proof}

\begin{lemma}
\label{lem:Vpi-V}
When $n\geq \max\{\frac{4B_\star-2c_\text{min}}{c_\text{min}d_{max}},\frac{26^2\times 2S\iota(\bar{T}^\star)^2}{d_m}, \frac{10^6(\sqrt{\tB}+1)^4S\iota\bar{T}^\star\widetilde{T}}{B_\star (\sqrt{B_\star }+1)^2d_m},O( {T}_{\max}^2\log(SA/\delta)/d_m^2)\}$, we have
\begin{align}
    \Vb-\bar{V}=\sumS\xi_{h}^{\bar{\pi}}(s)\{(P_{s,\bar{\pi}(s)}-\tP'_{s,\bar{\pi}(s)})\bar{V}+c(s,\bar{\pi}(s))-\hat{c}(s,\bar{\pi}(s))-b_{s,\bar{\pi}(s)}(\bar{V})\}.
\end{align}
\end{lemma}

\begin{proof}
First of all, by the condition we have $\bar{\pi}$ is a proper policy by Lemma~\ref{lem:piProper}. We prove by recursion formula.
\begin{align*}
    \sum_{s,s\neq g}\xi_h^{\bar{\pi}}(s)(\Vb-\bar{V})&=\sum_{s,s\neq g}\xi_h^{\bar{\pi}}(s)\{P_{s,\bar{\pi}(s)}\Vb+c(s,\bar{\pi}(s))-\tP'_{s,\bar{\pi}(s)}\bar{V}-\hat{c}(s,\bar{\pi}(s))-b_{s,\bar{\pi}(s)}(\bar{V})\}\\
    &=\sum_{s,s\neq g}\xi_h^{\bar{\pi}}(s)\{(P_{s,\bar{\pi}(s)}-\tP'_{s,\bar{\pi}(s)})\bar{V}+P_{s,\bar{\pi}(s)}(\Vb-\bar{V})+c(s,\bar{\pi}(s))-\hat{c}(s,\bar{\pi}(s))-b_{s,\bar{\pi}(s)}(\bar{V})\}\\
    &=\sum_{s,s\neq g}\xi_{h}^{\bar{\pi}}(s)\{(P_{s,\bar{\pi}(s)}-\tP'_{s,\bar{\pi}(s)})\bar{V}+c(s,\bar{\pi}(s))-\hat{c}(s,\bar{\pi}(s))-b_{s,\bar{\pi}(s)}(\bar{V})\}\\
    &\qquad+\sum_{s,s\neq g}\xi_{h+1}^{\bar{\pi}}(s)(\Vb-\bar{V}),
\end{align*}
where the first inequality uses the Bellman equation for policy $\bar{\pi}$, which follows from Lemma~\ref{lem:general_Bellman}.
By recursion, we have
\begin{align*}
    \Vb-\bar{V}&=\sum_{h=0}^H\sum_{s,s\neq g}\xi_{h}^{\bar{\pi}}(s)\{(P_{s,\bar{\pi}(s)}-\tP'_{s,\bar{\pi}(s)})\bar{V}+c(s,\bar{\pi}(s))-\hat{c}(s,\bar{\pi}(s))-b_{s,\bar{\pi}(s)}(\bar{V})\}\\
    &\qquad +\sum_{s,s\neq g}\xi_{H}^{\bar{\pi}}(s)(\Vb-\bar{V}).
\end{align*}
From Lemma~\ref{lem:piProper}, we have that when$n\geq \max\{\frac{4B_\star-2c_\text{min}}{c_\text{min}d_{max}},\frac{26^2\times 2S\iota(\bar{T}^\star)^2}{d_m}, \frac{10^6(\sqrt{\tB}+1)^4S\iota\bar{T}^\star\widetilde{T}}{B_\star (\sqrt{B_\star }+1)^2d_m},O( {T}_{\max}^2\log(SA/\delta)/d_m^2)\}$, $\bar{\pi}$ is a proper policy. Thus $\norm{\Vb}_\infty<+\infty$, and for any state $s,s\neq g$,  $\lim_{H\rightarrow\infty}\xi_{H}^{\bar{\pi}}(s)=0$. Let $H$ goes to infinity, we can prove the lemma. 
\end{proof}

\section{Crude optimization bound}
In this section, we give a rough bound for $\bar{V}-V^\star$.

\begin{theorem}
\label{thm:crude_PO}
Denote $d_m:=\min\{\sum_{h=1}^\infty \xi^\mu_h(s,a):s.t. \sum_{h=1}^\infty \xi^\mu_h(s,a)>0\}$ and ${T}_{\max}=\max_i T_i$. Let $T^\star_s$ be the expected time to hit $g$ when starting from $s$ with the optimal policy and denote $\bar{T}^\pi=\max_sT^\pi_s$. Then when $n\geq \max\{\frac{26^2\times 2S\iota(\bar{T}^\star)^2}{d_m}, \frac{10^6(\sqrt{\tB}+1)^4S\iota\bar{T}^\star\widetilde{T}}{B_\star (\sqrt{B_\star }+1)^2d_m},O( {T}_{\max}^2\log(SA/\delta)/d_m^2)\}$, we have with probability $1-\delta$,
\begin{equation}
     \norm{\bar{V}-V^\star}_\infty\leq30\sqrt{\frac{\bar{T}^\star B_\star\iota}{nd_m}}(\sqrt{B_\star}+1)
\end{equation}
\end{theorem}
\begin{proof}
From Lemma \ref{lem:bar_V-V*}, we have (by choosing $\xi_0=\mathbf{1}[s_0=\bar{s}]$)
\begin{align*}
     |\bar{V}(\bar{s})-V^\star(\bar{s})|&\leq\sumS \xi^\star_{h,\bar{s}}(s)\{(\tP_{s,\pi^\star(s)}-P_{s,\pi^\star(s)})\bar{V}+\widehat{c}(s,\pi^\star(s))-c(s,\pi^\star(s))+b_{s,\pi^\star(s)}(\bar{V})\}
\end{align*}
For the first term, we can bound it by Lemma~\ref{lem:(P-tP)barV}
\begin{align*} 
|(\tP_{s,a}-P_{s,a})\bar{V}|
    &\leq\frac{\tB}{n(s,a)}+\frac{16\tB\iota}{3n(s,a)}+2\sqrt{\frac{\Var(\tP',\bar{V})\iota}{n(s,a)}}+6\sqrt{\frac{S\iota}{n(s,a)}}\norm{V-V^\star}_{\infty}.
\end{align*}
Conditioned on the event $\mathcal{E}_5$, we have
\begin{align*}
    |\widehat{c}(s,a)-c(s,a)|\leq\sqrt{\frac{2\widehat{c}(s,a)\iota}{n(s,a)}}+\frac{7\iota}{3n(s,a)}
\end{align*}
Combine the above inequalities together, we can get
\begin{align*}
     |\bar{V}(\bar{s})-V^\star(\bar{s})|&\leq\sumS \xi^\star_{h,\bar{s}}(s)\{\frac{2\tB}{n(s,\pi^\star(s))}+\frac{32\tB\iota}{3n(s,\pi^\star(s))}+2\sqrt{\frac{\Var(\tP',\bar{V})\iota}{n(s,\pi^\star(s))}}+6\sqrt{\frac{S\iota}{n(s,\pi^\star(s))}}\norm{V-V^\star}_{\infty}\\
     &\qquad+2\sqrt{\frac{2\widehat{c}(s,\pi^\star(s))\iota}{n(s,\pi^\star(s))}}+\frac{14\iota}{3n(s,\pi^\star(s))}+\max\{2\sqrt{\frac{\Var(\tP,V)\iota}{n(s,\pi^\star(s))}},4\frac{\tB\iota}{n(s,\pi^\star(s))}\}\\
     &\qquad+180\sqrt{\frac{3\widetilde{T}\tB S}{2n(s,\pi^\star(s))n_{\min}}}(\sqrt{\tB}+1)\iota\}\\
     &\explaineq{(i)}{\leq}\sumS \xi^\star_{h,\bar{s}}(s)\{\frac{2\tB}{n(s,\pi^\star(s))}+\frac{44\tB\iota}{3n(s,\pi^\star(s))}+4\sqrt{\frac{\Var(\tP',\bar{V})\iota}{n(s,\pi^\star(s))}}+6\sqrt{\frac{S\iota}{n(s,\pi^\star(s))}}\norm{\bar{V}-V^\star}_{\infty}\\
     &\qquad+2\sqrt{\frac{2\widehat{c}(s,\pi^\star(s))\iota}{n(s,\pi^\star(s))}}+\frac{14\iota}{3n(s,\pi^\star(s))}+180\sqrt{\frac{3\widetilde{T}\tB S}{2n(s,\pi^\star(s))n_{\min}}}(\sqrt{\tB}+1)\iota\},
\end{align*}
where (i) uses the inequality that $\max\{a,b\}\leq a+b$. For notation simplicity, we define
\begin{align*}
    b_0(s,a):=180\sqrt{\frac{3\widetilde{T}\tB S}{2n(s,a)n_{\min}}}(\sqrt{\tB}+1)\iota.
\end{align*}
First, we bound the variance term
\begin{align*}
    \Var(\tP'_{s,a},\bar{V})&\explaineq{(i)}{\leq}\Var(\hP_{s,a},\bar{V})+\frac{2\norm{\bar{V}}_{\infty}^2}{n_{max}+1}\\
    &\explaineq{(ii)}{\leq}\frac{3}{2}\Var(P_{s,a},\bar{V})+\frac{2\norm{\bar{V}}^2_{\infty}S\iota}{n(s,a)}+\frac{2\norm{\bar{V}}_{\infty}^2}{n_{max}+1}\\
    &\explaineq{(iii)}{\leq}3\Var(P_{s,a},\bar{V}-V^\star)+3\Var(P_{s,a},V^\star)+\frac{2\norm{\bar{V}}^2_{\infty}S(\iota+1)}{n(s,a)}\\
     &\explaineq{(iv)}{\leq}3\norm{\bar{V}-V^\star}_{\infty}^2+3\Var(P_{s,a},V^\star)+\frac{2\tB^2S(\iota+1)}{n(s,a)},
\end{align*}
where (i) follows from Lemma~\ref{lem:(tP-hP)PO}. (ii) uses Lemma~\ref{lem:bound_var}. (iii) uses the fact that $\Var(X+Y)\leq2\Var(X)+2\Var(Y)$. (iv) uses the fact that $\norm{\bar{V}}_{\infty}\leq\tB$.
Then we can have
\begin{align*}
     |\bar{V}(\bar{s})-V^\star(\bar{s})|&\leq\sumS \xi^\star_{h,\bar{s}}(s)\{\frac{2\tB}{n(s,\pi^\star(s))}+\frac{44\tB\iota}{3n(s,\pi^\star(s))}+4\sqrt{\frac{3\iota}{n(s,\pi^\star(s))}}\norm{\bar{V}-V^\star}_{\infty}\\
     &\qquad+4\sqrt{\frac{3\Var(P_{s,\pi^\star(s)},V^\star)\iota}{n(s,\pi^\star(s))}}+4\sqrt{\frac{2S\iota(\iota+1)}{n^2(s,\pi^\star(s))}}\tB\\
     &\qquad+6\sqrt{\frac{S\iota}{n(s,\pi^\star(s))}}\norm{\bar{V}-V^\star}_{\infty}+2\sqrt{\frac{2\widehat{c}(s,\pi^\star(s))\iota}{n(s,\pi^\star(s))}}+\frac{14\iota}{3n(s,\pi^\star(s))}+b_0(s,a)\}\\
     &\explaineq{(i)}{\leq}\sumS \xi^\star_{h,\bar{s}}(s)\{\frac{27\max\{\tB,1\}\sqrt{S}\iota}{n(s,\pi^\star(s))}+13\sqrt{\frac{S\iota}{n(s,\pi^\star(s))}}\norm{\bar{V}-V^\star}_{\infty}\\
     &+4\sqrt{\frac{3\Var(P_{s,\pi^\star(s)},V^\star)\iota}{n(s,\pi^\star(s))}}+2\sqrt{\frac{2\widehat{c}(s,\pi^\star(s))\iota}{n(s,\pi^\star(s))}}+b_0(s,a)\}\\
     &\explaineq{(ii)}{\leq}\sumS \xi^\star_{h,\bar{s}}(s)\{\frac{31\max\{\tB,1\}\sqrt{S}\iota}{n(s,\pi^\star(s))}+13\sqrt{\frac{S\iota}{n(s,\pi^\star(s))}}\norm{\bar{V}-V^\star}_{\infty}\\
     &+4\sqrt{\frac{3\Var(P_{s,\pi^\star(s)},V^\star)\iota}{n(s,\pi^\star(s))}}+2\sqrt{\frac{3c(s,\pi^\star(s))\iota}{n(s,\pi^\star(s))}}+b_0(s,a)\},
\end{align*}
where (i) uses the assumption $\iota\geq1$ and that $S\geq1$. (ii) holds because of Lemma~\ref{lem:c}.
Then we have
\begin{align*}
     |\bar{V}(\bar{s})-V^\star(\bar{s})|
     &\explaineq{(i)}{\leq}\sumS \xi^\star_{h,\bar{s}}(s)\{\frac{62\max\{\tB,1\}\sqrt{S}\iota}{nd_m}+13\sqrt{2}\sqrt{\frac{S\iota}{nd_m}}\norm{\bar{V}-V^\star}_{\infty}\\
     &+4\sqrt{\frac{6\Var(P_{s,\pi^\star(s)},V^\star)\iota}{nd_m}}+2\sqrt{\frac{6c(s,\pi^\star(s))\iota}{nd_m}}+\bar{b}_0\}\\
     &\explaineq{(ii)}{\leq}\frac{62T_{\bar{s}}^\star\max\{\tB,1\}\sqrt{S}\iota}{nd_m}+13\sqrt{\frac{2S\iota}{nd_m}}\norm{\bar{V}-V^\star}_{\infty}T_{\bar{s}}^\star+T_{\bar{s}}^\star\bar{b}_0\\
     &\qquad+\sumS \xi^\star_{h,\bar{s}}(s)\{4\sqrt{\frac{6\Var(P_{s,\pi^\star(s)},V^\star)\iota}{nd_m}}+2\sqrt{\frac{6c(s,\pi^\star(s))\iota}{nd_m}}\},
\end{align*}
where $\bar{b}_0:=180\sqrt{\frac{6\widetilde{T}\tB S}{n^2d_m^2}}(\sqrt{\tB}+1)\iota$. (i) holds with probability $1-\delta$ because of Lemma~\ref{lem:Chern_SSP}. For any $(s,a)\in\mathcal{S}\times\mathcal{A} $, we have $n(s,a)\geq\half nd(s,a)\geq\half{nd_m}$. In particular, $n_{\min}\geq\half nd_m$. 
Since
\begin{align*}
    \sumS \xi^\star_{h,\bar{s}}(s)\{4\sqrt{\frac{6\Var(P_{s,\pi^\star(s)},V^\star)\iota}{nd_m}}\}&\explaineq{(i)}{\leq}4\sqrt{\sumS \xi^\star_{h,\bar{s}}(s)}\sqrt{\frac{6\sumS \xi^\star_{h,\bar{s}}(s)\Var(P_{s,\pi^\star(s)},V^\star)\iota}{nd_m}}\\
    &\explaineq{(ii)}{\leq}4\sqrt{T^\star_{\bar{s}}}\sqrt{\frac{12\iota}{nd_m}}\norm{V^\star}_{\infty},
\end{align*}
where (i) uses the Cauchy-Schwartz inequality. (ii) uses the result in Lemma~\ref{lem:propT} and Lemma~\ref{lem:bound_var}. Similarly, we have
\begin{align*}
    \sumS \xi^\star_{h,\bar{s}}(s)\{2\sqrt{\frac{6c(s,\pi^\star(s))\iota}{nd_m}}\}&{\leq}2\sqrt{\sumS \xi^\star_{h,\bar{s}}(s)}\sqrt{\frac{6\sumS \xi^\star_{h,\bar{s}}(s)c(s,\pi^\star(s))\iota}{nd_m}}\\
    &\leq2\sqrt{T^\star_{\bar{s}}}\sqrt{\frac{6\norm{V^\star}_{\infty}\iota}{nd_m}}.
\end{align*}
Combine the above together, we get
\begin{align*}
     |\bar{V}(\bar{s})-V^\star(\bar{s})|
     &{\leq}\frac{62T_{\bar{s}}^\star\max\{\tB,1\}\sqrt{S}\iota}{nd_m}+13\sqrt{\frac{2S\iota}{nd_m}}\norm{\bar{V}-V^\star}_{\infty}T_{\bar{s}}^\star+T_{\bar{s}}^\star\bar{b}_0\\
     &\qquad+4\sqrt{T^\star_{\bar{s}}}\sqrt{\frac{12\iota}{nd_m}}\norm{V^\star}_{\infty}+2\sqrt{T^\star_{\bar{s}}}\sqrt{\frac{6\norm{V^\star}_{\infty}\iota}{nd_m}}\\
     &{\leq}\frac{62\bar{T}^\star\max\{\tB,1\}\sqrt{S}\iota}{nd_m}+13\sqrt{\frac{2S\iota}{nd_m}}\norm{\bar{V}-V^\star}_{\infty}\bar{T}^\star+\bar{T}^\star\bar{b}_0\\
     &\qquad+8\sqrt{\bar{T}^\star}\sqrt{\frac{3\iota}{nd_m}}\norm{V^\star}_{\infty}+2\sqrt{\bar{T}^\star}\sqrt{\frac{6\norm{V^\star}_{\infty}\iota}{nd_m}},
\end{align*}

It implies that
\begin{align}
    (1-13\sqrt{\frac{2S\iota}{nd_m}}\bar{T}^\star)\norm{\bar{V}-V^\star}_{\infty}&\leq\frac{62\max\{\tB,1\}\sqrt{S}\iota\bar{T}^\star}{nd_m}+\bar{T}^\star\bar{b}_0\notag\\
    &\qquad+14\sqrt{\frac{\bar{T}^\star B_\star\iota}{nd_m}}(\sqrt{B_\star}+1).
\end{align}
Since $n\geq\frac{26^2\times 2S\iota(\bar{T}^\star)^2}{d_m}$, 
\begin{align}
    \norm{\bar{V}(s)-V^\star(s)}_{\infty}&\leq\frac{124\max\{\tB,1\}\sqrt{S}\iota\bar{T}^\star}{nd_m}+2\bar{T}^\star\bar{b}_0\notag\\
    &\qquad+21\sqrt{\frac{\bar{T}^\star B_\star\iota}{nd_m}}(\sqrt{B_\star}+1)\notag\\
    &\leq\frac{124\max\{\tB,1\}\sqrt{S}\iota\bar{T}^\star}{nd_m}+360\bar{T}^\star\sqrt{\frac{6\widetilde{T}\tB S}{n^2d_m^2}}(\sqrt{\tB}+1)\iota\notag\\
    &\qquad+28\sqrt{\frac{\bar{T}^\star B_\star\iota}{nd_m}}(\sqrt{B_\star}+1)\notag\\
    &\leq720\bar{T}^\star\sqrt{\frac{6\widetilde{T} S}{n^2d_m^2}}(\sqrt{\tB}+1)^2\iota+28\sqrt{\frac{\bar{T}^\star B_\star\iota}{nd_m}}(\sqrt{B_\star}+1)\notag
\end{align}
When $n\geq\frac{10^6(\sqrt{\tB}+1)^4S\iota\bar{T}^\star\widetilde{T}}{B_\star (\sqrt{B_\star }+1)^2d_m}$, we have
\begin{align}
    \norm{\bar{V}(s)-V^\star(s)}_{\infty}&\leq30\sqrt{\frac{\bar{T}^\star B_\star\iota}{nd_m}}(\sqrt{B_\star}+1)\notag
\end{align}
\end{proof}

\section{Proof of Theorem~\ref{thm:OPL}}\label{sec:proof_OPO}
In this section, we provide the proof of Theorem~\ref{thm:OPL}. However, before that, we first present a lemma that guarantees pessimism.

\begin{lemma}\label{lem:pessimism}
	When $n\geq \max\{\frac{26^2\times 2S\iota(\bar{T}^\star)^2}{d_m}, \frac{10^6(\sqrt{\tB}+1)^4S\iota\bar{T}^\star\widetilde{T}}{B_\star (\sqrt{B_\star }+1)^2d_m},O( {T}_{\max}^2\log(SA/\delta)/d_m^2)\}$ (where ${T}_{\max}=\max_i T_i$), with probability at least $1-\delta$, we have that for any state action pair (s,a), 
	\begin{align*}
	c(s,a)-\widehat{c}(s,a)+(P_{s,a}-\tP'_{s,a})\bar{V}-b_{s,a}(\bar{V})\leq0
	\end{align*}
\end{lemma}
\begin{proof}
	Applying the result in Theorem~\ref{thm:crude_PO}, we can get
	\begin{align*}
	6\sqrt{\frac{S\iota}{n(s,a)}}\norm{\bar{V}-V^\star}_{\infty}\leq180\sqrt{\frac{\bar{T}^\star B_\star S}{n(s,a)nd_m}}(\sqrt{B_\star}+1)\iota.
	\end{align*}
	Combine the above inequality with Lemma~\ref{lem:(P-tP)barV} implies that
	\begin{align*}
	(P_{s,a}-\tP'_{s,a})\bar{V}
	&\leq\frac{\tB}{n(s,a)}+\frac{16\tB\iota}{3n(s,a)}+2\sqrt{\frac{\Var(\tP',\bar{V})\iota}{n(s,a)}}+180\sqrt{\frac{\bar{T}^\star B_\star S}{n(s,a)nd_m}}(\sqrt{B_\star}+1)\iota.
	\end{align*}
	Conditioned on the event $\mathcal{E}_5$, then we have
	\begin{align*}
	c(s,a)-\widehat{c}(s,a)+(P_{s,a}-\tP'_{s,a})\bar{V}-b_{s,a}(\bar{V})\leq180\sqrt{\frac{\bar{T}^\star B_\star S}{n(s,a)nd_m}}(\sqrt{B_\star}+1)\iota-180\sqrt{\frac{3\widetilde{T}\tB S}{2n(s,a)n_{\min}}}(\sqrt{\tB}+1)\iota.
	\end{align*}
	Applying the Chernoff bound given in Lemma~\ref{lem:Chern_SSP}, we have that with probability $1-\delta$,  $n(s,a)<\frac{3}{2}nd^{\mu}(s,a)$ for any state action pair $(s,a)$. Thus $n_{\text{min}}:=\min_{s,a,n(s,a)>0}n(s,a)<\frac{3}{2}n(\min_{n(s,a)>0}d^{\mu}(s,a))$. For any $(s,a)\in\mathcal{S}\times\mathcal{A}$, if we have $d^{\mu}(s,a)>0$, by the Lemma~\ref{lem:Chern_SSP} we have that with probability $1-\delta$, $n(s,a)>\frac{nd^{\mu}(s,a)}{2}>0$, which implies
	\begin{align*}
	\{(s,a)\in\mathcal{S}\times\mathcal{A}:d^{\mu}(s,a)>0\}\subseteq\{(s,a)\in\mathcal{S}\times\mathcal{A}:n(s,a)>0\}.
	\end{align*}
	Then we can get $\min_{n(s,a)>0}d^{\mu}(s,a)\leq\min_{d^{\mu}(s,a)>0}d^{\mu}(s,a)=d_m$ and thus $n_{\text{min}}\leq\frac{3}{2}nd_m$. 
	Because $\bar{T}^\star\leq\widetilde{T}$ and $B_\star\leq\tB$, we can prove the result in the Lemma.
\end{proof}

Now we are ready to introduce the final proof.

\begin{theorem}\label{thm:main_appendix}
	Given Assumption~\ref{assum:opl} and Assumption~\ref{assum:PC}. When $n\geq n_0$, the suboptimality bound of the output policy $\bar{\pi}$ can be upper bounded as follows with probability $1-\delta$ (where $\iota=O(\log(SA/\delta))$),
	\begin{align}
	V^{\bar{\pi}}(s_\mathrm{init})-V^\star(s_\mathrm{init})\leq8\sum_{s,a,s\neq g} d^\star(s,a)\sqrt{\frac{3\Var_{P_{s,a}}[V^\star+c]\iota}{n\cdot d^\mu(s,a)}}+\widetilde{O}(\frac{1}{n}),
	\end{align}
	where we define $n_0:=n\geq \max\{\frac{4B_\star-2c_\text{min}}{c_\text{min}d_{max}}, \frac{26^2\times 2S\iota(\bar{T}^\star)^2(\sqrt{B_\star}+1)^2}{d_m}, \frac{10^6(\sqrt{\tB}+1)^4S\iota\bar{T}^\star\widetilde{T}}{B_\star (\sqrt{B_\star }+1)^2d_m},O( {T}_{\max}^2\log(SA/\delta)/d_m^2)\}$.
\end{theorem}
\begin{proof}
	\begin{align}
	V^{\bar{\pi}}(s)-V^\star(s)&= (V^{\bar{\pi}}(s)-\bar{V}(s))+ (\bar{V}(s)-V^\star(s)).
	\end{align}
	From Lemma~\ref{lem:Vpi-V}, we have with probability $1-\delta$ and when 
	$n\geq \max\{\frac{4B_\star-2c_\text{min}}{c_\text{min}d_{max}},\frac{26^2\times 2S\iota(\bar{T}^\star)^2}{d_m}, \frac{10^6(\sqrt{\tB}+1)^4S\iota\bar{T}^\star\widetilde{T}}{B_\star (\sqrt{B_\star }+1)^2d_m}\}$,
	\begin{align}\label{eqn:decomp_bar}
	\Vb-\bar{V}=\sumS\xi_{h}^{\bar{\pi}}(s)\{(P_{s,\bar{\pi}(s)}-\tP'_{s,\bar{\pi}(s)})\bar{V}+c(s,\bar{\pi}(s))-\hat{c}(s,\bar{\pi}(s))-b_{s,\bar{\pi}(s)}(\bar{V})\}.
	\end{align}
	Next by Lemma~\ref{lem:pessimism} with probability $1-\delta$, we have 
	\begin{align}\label{eqn:re_pressimism}
	(P_{s,\bar{\pi}(s)}-\tP'_{s,\bar{\pi}(s)})\bar{V}+c(s,\bar{\pi}(s))-\hat{c}(s,\bar{\pi}(s))-b_{s,\bar{\pi}(s)}(\bar{V})\leq0.
	\end{align}
	Thus combine \eqref{eqn:decomp_bar} and \eqref{eqn:re_pressimism} we have $\Vb-\bar{V}\leq0$ by pessimism. For the term $\bar{V}-V^\star$, we apply Lemma~\ref{lem:bar_V-V*} to obtain:
	\begin{align}\label{eqn:bound}
	\bar{V}-V^\star\leq\sumS \xi^\star_h(s)\{(\tP'_{s,\pi^\star(s)}-P_{s,\pi^\star(s)})\bar{V}+\hat{c}(s,\pi^\star(s))-c(s,\pi^\star(s))+b_{s,\pi^\star(s)}(\bar{V})\}.
	\end{align}
	From Lemma~\ref{lem:(P-tP)barV}, we have
	\begin{align}
	|(P_{s,a}-\tP'_{s,a})\bar{V}|
	&\leq\frac{\tB}{n(s,a)}+\frac{16\tB\iota}{3n(s,a)}+2\sqrt{\frac{\Var(\tP'_{s,a},\bar{V})\iota}{n(s,a)}}+6\sqrt{\frac{S\iota}{n(s,a)}}\norm{\bar{V}-V^\star}_{\infty}
	\end{align}
	Conditioned on the event $\mathcal{E}_5$, we have
	\begin{align*}
	|\widehat{c}(s,a)-c(s,a)|\leq\sqrt{\frac{2\widehat{c}(s,a)\iota}{n(s,a)}}+\frac{7\iota}{3n(s,a)}
	\end{align*}
	Combine the above inequalities together, we can get
	\begin{align*}
	&(\tP'_{s,a}-P_{s,a})\bar{V}+\widehat{c}(s,a)-c(s,a)+b_{s,a}(\bar{V})\\
	&\leq2\sqrt{\frac{2\widehat{c}(s,a)\iota}{n(s,a)}}+\frac{14\iota}{3n(s,a)}+\frac{2\tB}{n(s,a)}+\frac{32\tB\iota}{3n(s,a)}+4\sqrt{\frac{\Var(\tP'_{s,a},\bar{V})\iota}{n(s,a)}}\\
	&\qquad+\frac{4\tB\iota}{n(s,a)}+6\sqrt{\frac{S\iota}{n(s,a)}}\norm{\bar{V}-V^\star}_{\infty}+180\sqrt{\frac{3\widetilde{T}\tB S}{2n(s,a)n_{\min}}}(\sqrt{\tB}+1)\iota\\
	&\leq2\sqrt{\frac{2\widehat{c}(s,a)\iota}{n(s,a)}}+4\sqrt{\frac{\Var(\tP'_{s,a},\bar{V})\iota}{n(s,a)}}+\widetilde{O}(\frac{(\tB+1)\iota}{n(s,a)})+6\sqrt{\frac{S\iota}{n(s,a)}}\norm{\bar{V}-V^\star}_{\infty}+180\sqrt{\frac{3\widetilde{T}\tB S}{2n(s,a)n_{\min}}}(\sqrt{\tB}+1)\iota\\
	&\leq2\sqrt{\frac{3c(s,a)\iota}{n(s,a)}}+4\sqrt{\frac{3\Var(P_{s,a},V^\star)\iota}{n(s,a)}}+\widetilde{O}(\frac{(\tB+1)\sqrt{S}\iota}{n(s,a)})\\
	&\qquad+\widetilde{O}(\sqrt{\frac{S\iota}{n(s,a)}}\norm{\bar{V}-V^\star}_{\infty})+180\sqrt{\frac{3\widetilde{T}\tB S}{2n(s,a)n_{\min}}}(\sqrt{\tB}+1)\iota\\
	\end{align*}
	Plug the above into \eqref{eqn:bound}, then we bound all the terms one by one. First,
	\begin{align}
	\label{eqn:bound_c}
	\sum_{h=1}^{\infty}\sum_{\substack{s,a \\ s\neq g}}\xi^\star_h(s,a)\left(2\sqrt{\frac{3c(s,a)\iota}{n(s,a)}}\right)\leq\sum_{h=1}^{\infty}\sum_{\substack{s,a \\ s\neq g}}\xi^\star_h(s,a)\left[2\sqrt{\frac{6c(s,a)\iota}{n\sum_{h=1}^\infty\xi^\mu_h(s,a)}}\right]=2\sum_{\substack{s,a \\ s\neq g}}d^\star(s,a)\sqrt{\frac{6c(s,a)\iota}{nd^\mu(s,a)}}.
	\end{align}
	For the second term, first we have
	\begin{align}
	\label{eqn:bound_var}
	\sum_{h=1}^{\infty}\sum_{\substack{s,a \\ s\neq g}}\xi^\star_h(s,a)\left(4\sqrt{\frac{3\Var(P_{s,a},V^\star)\iota}{n(s,a)}}\right)&\leq\sum_{h=1}^{\infty}\sum_{\substack{s,a \\ s\neq g}}\xi^\star_h(s,a)\left[4\sqrt{\frac{6\Var(P_{s,a},V^\star)\iota}{n\sum_{h=1}^\infty\xi^\mu_h(s,a)}}\right]\notag\\
	&=4\sum_{\substack{s,a \\ s\neq g}}d^\star(s,a)\sqrt{\frac{6\Var(P_{s,a},V^\star)\iota}{nd^\mu(s,a)}}
	\end{align}
	From Chernoff bound given in Lemma~\ref{lem:Chern_SSP}, we have with probability $1-\delta$, we have
	\begin{align}
	\label{eqn:bound_SmallTerm1}
	\widetilde{O}(\sum_{h=1}^{\infty}\sum_{\substack{s,a \\ s\neq g}}\xi^\star_h(s,a)\frac{(\tB+1)\sqrt{S}\iota}{n(s,a)})\leq\widetilde{O}(\sum_{\substack{s,a \\ s\neq g}}\frac{d^\pi(s,a)}{d^\mu(s,a)}\cdot\frac{(\tB+1)\sqrt{S}\iota}{n}).
	\end{align}
	Similarly, we have
	\begin{align}
	\label{eqn:bound_SmallTerm2}
	\widetilde{O}(\sum_{h=1}^{\infty}\sum_{\substack{s,a \\ s\neq g}}\xi^\star_h(s,a)\sqrt{\frac{S\iota}{n(s,a)}}\norm{\bar{V}-V^\star}_{\infty})&\leq\widetilde{O}(\sum_{\substack{s,a \\ s\neq g}}d^\star(s,a)\sqrt{\frac{S\iota}{nd^{\mu}(s,a)}}\norm{\bar{V}-V^\star}_{\infty})\notag\\
	&\explaineq{(i)}{\leq}\widetilde{O}(\sum_{\substack{s,a \\ s\neq g}}d^\star(s,a)\sqrt{\frac{\bar{T}^\star B_\star S}{n^2d^{\mu}(s,a)d_m}}(\sqrt{B_\star}+1)\iota),
	\end{align}
	where (i) uses the Crude optimization bound given in Theorem~\ref{thm:crude_PO}. For the last term, we have
	\begin{align}
	\label{eqn:bound_SmallTerm3}
	\widetilde{O}(\sum_{h=1}^{\infty}\sum_{\substack{s,a \\ s\neq g}}\xi^\star_h(s,a)\sqrt{\frac{\widetilde{T}\tB S}{n(s,a)n_{\min}}}(\sqrt{\tB}+1)\iota)\leq\widetilde{O}(\sum_{\substack{s,a \\ s\neq g}}d^\star(s,a)\sqrt{\frac{\widetilde{T}\tB S}{n^2d^{\mu}(s,a)d_m}}(\sqrt{\tB}+1)\iota),
	\end{align}
	where the inequality comes from Lemma~\ref{lem:Chern_SSP} again.
	Combine the inequalities (\ref{eqn:bound_c}), (\ref{eqn:bound_var}), (\ref{eqn:bound_SmallTerm1}), (\ref{eqn:bound_SmallTerm2}) and (\ref{eqn:bound_SmallTerm1}) together, we have
	\begin{align}\label{eqn:final_derivation}
	\bar{V}(s_\text{init})-V^\star(s_\text{init})&\explaineq{(i)}{\leq} 2\sum_{\substack{s,a \\ s\neq g}}d^\star(s,a)\sqrt{\frac{6c(s,a)\iota}{nd^\mu(s,a)}}+4\sum_{\substack{s,a \\ s\neq g}}d^\star(s,a)\sqrt{\frac{6\Var(P_{s,a},V^\star)\iota}{nd^\mu(s,a)}}\notag\\
	&\qquad +\widetilde{O}(\sum_{\substack{s,a \\ s\neq g}}d^\star(s,a)\sqrt{\frac{\widetilde{T}\tB S}{n^2d^{\mu}(s,a)d_m}}(\sqrt{\tB}+1)\iota)+\widetilde{O}(\sum_{\substack{s,a \\ s\neq g}}\frac{d^\pi(s,a)}{d^\mu(s,a)}\cdot\frac{(\tB+1)\sqrt{S}\iota}{n})\notag\\
	&\explaineq{(ii)}{\leq}2\sum_{\substack{s,a \\ s\neq g}}d^\star(s,a)\sqrt{\frac{6c(s,a)\iota}{nd^\mu(s,a)}}+4\sum_{\substack{s,a \\ s\neq g}}d^\star(s,a)\sqrt{\frac{6\Var(P_{s,a},V^\star)\iota}{nd^\mu(s,a)}}\notag\\
	&\qquad +\widetilde{O}(\sum_{\substack{s,a \\ s\neq g}}d^\star(s,a)\sqrt{\frac{\widetilde{T} S}{n^2d^{\mu}(s,a)d_m}}(\tB+1)\iota)\notag\\
	&\leq8\sum_{\substack{s,a \\ s\neq g}}d^\star(s,a)\sqrt{\frac{3\Var(P_{s,a},V^\star+c)\iota}{nd^\mu(s,a)}}+\widetilde{O}(\sum_{\substack{s,a \\ s\neq g}}d^\star(s,a)\sqrt{\frac{\widetilde{T} S}{n^2d^{\mu}(s,a)d_m}}(\tB+1)\iota),
	\end{align}
	where the inequality (i) uses the assumption that $\bar{T}^\star\leq\widetilde{T}$ and $B_\star\leq \tB$. (ii) uses the fact that $\frac{d^\pi(s,a)}{d^\mu(s,a)}\leq\frac{d^\pi(s,a)}{\sqrt{d^\mu(s,a)d_m}}$ and that $\tB+\sqrt{\tB}\leq2(\tB+1)$. The last inequality comes from $\sqrt{a}+\sqrt{b}\leq\sqrt{2a+2b}$.
\end{proof}
Based on Theorem~\ref{thm:crude_PO}, we can also get the Proposition below.

\begin{proposition}\label{prop:simplified}
	When $n\geq n_0$ (where $n_0$ is defined the same as Theorem~\ref{thm:main_appendix}), then the suboptimality incurred by the limit of the output policy $\bar{\pi}$ can be upper bounded as (with probability $1-\delta$)
	\begin{align}
	V^{\bar{\pi}}(s_\text{init})-V^\star(s_\text{init})\leq 8\big(\sqrt{\sum_{\substack{s,a \\ s\neq g}}\frac{d^\star(s,a)}{ d^\mu(s,a)}\cdot\frac{6\iota}{n}}\big)\cdot(B_\star+1)+\widetilde{O}(\frac{1}{n}).
	\end{align}
\end{proposition}
\begin{proof}
	By Theorem~\ref{thm:OPL},
	\begin{align}
	\label{eqn:propPO}
	V^{\bar{\pi}}(s_\text{init})-V^\star(s_\text{init})&\leq8\sum_{s,a,s\neq g} d^\star(s,a)\sqrt{\frac{3\Var_{P_{s,a}}[V^\star+c]\iota}{n\cdot d^\mu(s,a)}}+\widetilde{O}(\frac{1}{n})\notag\\
	&\explaineq{(i)}{\leq}8\sqrt{\sum_{\substack{s,a \\ s\neq g}}\frac{d^\star(s,a)}{n\cdot d^\mu(s,a)}}\cdot\sqrt{3\sum_{\substack{s,a \\ s\neq g}}d^\star(s,a)\Var_{P_{s,a}}[V^\star+c]\iota},
	\end{align}
	where (i) uses the Cauchy-Schwartz inequality. 
	Since 
	\begin{align}
	\label{eqn:propPO_boundvar}
	\sum_{\substack{s,a \\ s\neq g}}d^\star(s,a)\Var_{P_{s,a}}[V^\star+c]
	&=\sum_{h=1}^{\infty}\sum_{\substack{s,a \\ s\neq g}}\xi^\star_h(s,a)(\Var(P_{s,a},V^\star)+c(s,a))\notag\\
	&\explaineq{(i)}{\leq}2\norm{V^\star}^2_{\infty}+V^\star(s_0)\notag\\
	&{\leq}2B_\star^2+B_\star,
	\end{align}
	where (i) comes from Lemma~\ref{lem:bound_var} and the definition of value function. Plug (\ref{eqn:propPO_boundvar}) into (\ref{eqn:propPO}), we obtain
	\begin{align}
	V^{\bar{\pi}}(s_\text{init})-V^\star(s_\text{init})&\leq8\sqrt{\sum_{\substack{s,a \\ s\neq g}}\frac{d^\star(s,a)}{n\cdot d^\mu(s,a)}}\cdot\sqrt{6(B_\star^2+B_\star)\iota}+\widetilde{O}(\frac{1}{n})\notag\\
	&\leq8\sqrt{\sum_{\substack{s,a \\ s\neq g}}\frac{6d^\star(s,a)\cdot\iota}{ d^\mu(s,a)\cdot n}}\cdot(B_\star+1)+\widetilde{O}(\frac{1}{n}),
	\end{align}
	which completes the proof.
\end{proof}

\section{Properties of transition matrix estimate $\hP$}

\begin{lemma}
\label{lem:bound_var}
For any $V(\cdot)\in\mathbb{R}^{S}$ satisfying $V(g)=0$, i.e. \eqref{eqn:high_prob}, and suppose event $\mathcal{E}_1$ holds. we have
\begin{align}
    \Var(\widehat{P}_{s,a},V)&\leq\frac{3}{2}\Var(P_{s,a},V)+\frac{2\norm{V}^2_{\infty}S\iota}{n(s,a)}\notag\\
    \Var(P_{s,a},V)&\leq2\Var(\hP_{s,a},V)+\frac{4\norm{V}^2_{\infty}S\iota}{n(s,a)}
\end{align}
\end{lemma}
\begin{proof}
From the event $\mathcal{E}_1$, we have
\begin{align*}
   |\widehat{P}(s'|s,a)-P(s'|s,a)|&\leq  \sqrt{\frac{2P(s'|s,a)\iota}{n(s,a)}}+\frac{2\iota}{3n(s,a)}
   \leq \frac{P(s'|s,a)}{2}+\frac{5\iota}{3n(s,a)},
\end{align*}
where the second inequality uses $\sqrt{ab}\leq \frac{a+b}{2}$ with $a=\frac{2\iota}{n(s,a)}$, $b=P(s'|s,a)$. Thus we have
\begin{align}
    \widehat{P}(s'|s,a)&\leq\frac{3P(s'|s,a)}{2}+\frac{5\iota}{3n(s,a)}\leq\frac{3P(s'|s,a)}{2}+\frac{2\iota}{n(s,a)}\notag\\
    P(s'|s,a)
   &\leq 2\hP(s'|s,a)+\frac{10\iota}{3n(s,a)}\leq2\hP(s'|s,a)+\frac{4\iota}{n(s,a)}.
\end{align}

For the first inequality, we have
\begin{align*}
    \Var(\widehat{P}_{s,a},V)&=\widehat{P}_{s,a}(V-\widehat{P}_{s,a}V)^2\leq\widehat{P}_{s,a}(V-P_{s,a}V)^2\notag\\
    &\leq\sum_{s'}\left(\frac{3P(s'|s,a)}{2}+\frac{2\iota}{n(s,a)}\right)(V(s')-P_{s,a}V)^2\\
    &\leq\frac{3}{2}\Var(P_{s,a},V)+\frac{2\norm{V}^2_{\infty}S\iota}{n(s,a)},
\end{align*}
here the first inequality is due to $\widehat{P}_{s,a}V:=\argmin_z\sum_{s'} \widehat{P}_{s,a}(s')(V(s')-z)^2$, and the last term has $S+1$ due to the extra state $g$. 
For the second part, we have
\begin{align*}
    \Var(P_{s,a},V)&=P_{s,a}(V-P_{s,a}V)^2\leq P_{s,a}(V-\hP_{s,a}V)^2\notag\\
    &\leq\sum_{s'}\left(2\hP(s'|s,a)+\frac{4\iota}{n(s,a)}\right)(V(s')-\hP_{s,a}V)^2\\
    &\leq2\Var(\hP_{s,a},V)+\frac{4\norm{V}^2_{\infty}(S+1)\iota}{n(s,a)},
\end{align*}
\end{proof}
    

\begin{lemma}
\label{lem:c}
With probability at least $1-\delta$, we have
    \begin{align*}
        c(s,a)&\leq2\widehat{c}(s,a)+\frac{10\iota}{3n(s,a)}\\
        \widehat{c}(s,a)&\leq\frac{3}{2}c(s,a)+\frac{5\iota}{3n(s,a)}
    \end{align*}
\end{lemma}
\begin{proof}
Conditioned on event $\mathcal{E}_4$, we have
\begin{align*}
    |c(s,a)-\widehat{c}(s,a)|&\leq\sqrt{\frac{2c(s,a)\iota}{n(s,a)}}+\frac{2\iota}{3n(s,a)}\\
    &\leq\frac{\iota}{n(s,a)}+\half c(s,a)+\frac{2\iota}{3n(s,a)}\\
    &\leq\frac{5\iota}{3n(s,a)}+\half c(s,a),
\end{align*}
where the first inequality uses the assumption that $c(s,a)\in[0,1]$. The second inequality follows from the result that $\sqrt{ab}\leq\frac{a+b}{2}$. Simplify the above inequality, we can conclude the proof.
\end{proof}

\begin{lemma}
\label{lemma:(P-hat_P)V2}
    With probability $1-\delta$, for all $V(\cdot)\in\mathbb{R}^{S'}$ such that $\norm{V}_{\infty}<\infty$, we have for all sate-action pair $(s,a)$
    \begin{align}        
    (\widehat{P}_{s,a}-P_{s,a})V\leq\sqrt{\frac{2S\Var(P_{s,a},V)\iota}{n(s,a)}}+\frac{2\norm{V}_{\infty}S\iota}{3n(s,a)},
    \end{align}
    where $\iota=O(\log(SA/\delta))$.
\end{lemma}

\begin{proof}
Suppose the event $\mathcal{E}_1$ holds. Then we have (deterministically)
\begin{align*}
|(\widehat{P}_{s,a}-P_{s,a})V|&\explaineq{(i)}{=}|(\widehat{P}_{s,a}-P_{s,a})(V-P_{s,a}V\textbf{1}_S)|\\
&\leq(\sqrt{\frac{2P(\cdot|s,a)\iota}{n(s,a)}}+\frac{2\iota}{3n(s,a)})|V-P_{s,a}V\textbf{1}_S|\\
&\leq\sqrt{\frac{2P_{s,a}\iota}{n(s,a)}}\left|V-P_{s,a}V\textbf{1}_S\right|+\frac{2S\norm{V}_{\infty}\iota}{3n(s,a)}\\
&\explaineq{(ii)}{\leq}(\sqrt{S}\sqrt{\frac{2P_{s,a}(V-P_{s,a}V\textbf{1}_S)^2\iota}{n(s,a)}})+\frac{2S\norm{V}_{\infty}\iota}{3n(s,a)}\\
&\leq\sqrt{\frac{2S\Var(P_{s,a},V)\iota}{n(s,a)}}+\frac{2\norm{V}_{\infty}S\iota}{3n(s,a)},
\end{align*}
where (i) follows from the fact that $P_{s,a}V$ is a scalar, which implies that $(\widehat{P}_{s,a}-P_{s,a})(P_{s,a}V)\textbf{1}_S=(P_{s,a}V)\sum_{s'}(\hP(s'|s,a)-P(s'|s,a))=0$. (ii) uses the Cauchy-Schwartz inequality. Lastly, $\mathcal{E}_1$ fails with probability only $\delta$ (by Lemma~\ref{lem:high_probability}).

\end{proof}

\section{Minimax Lower Bound for Offline SSP}\label{sec:lower_proof}

In this section, we provide the minimax lower bound for offline stochastic shortest path problem. Concretely, we consider the family of problems satisfying bounded partial coverage, \emph{i.e.} $\max_{s,a,s\neq g}\frac{d^{\pi^\star}(s,a)}{d^{\mu}(s,a)}\leq C^\star$, where $d^{\pi}(s,a)=\sum_{h=0}^\infty \xi^\pi_h(s,a)<\infty$ for all $s,a$ (excluding $g$) for any proper policy $\pi$. Formally, we have the following result:

\begin{theorem}[Restatement of Theorem~\ref{thm:lower_main}]\label{thm:lower}
We define the following family of SSPs:
\[
\mathrm{SSP}(C^\star)=\{(s_{\mathrm{init}},\mu,P,c)|\max_{s,a,s\neq g}\frac{d^{\pi^\star}(s,a)}{d^{\mu}(s,a)}\leq C^\star\},
\]
where $d^\pi(s,a)=\sum_{h=0}^\infty\xi^\pi_h(s,a)$. Then for any $C^\star\geq 1$, $\norm{V^\star}_\infty=B_\star >1$, it holds (for some universal constant $c$)
\[
\inf_{\widehat{\pi} \;proper}\sup_{(s_{\mathrm{init}},\mu,P,c)\in\mathrm{SSP}(C^\star)}\E_{\mathcal{D}}[V^{\widehat{\pi}}(s_{\mathrm{init}})-V^\star(s_{\mathrm{init}})]\geq c\cdot B_\star \sqrt{\frac{SC^\star}{n}}.
\]
\end{theorem}

The proof of Theorem~\ref{thm:lower} relies on the hard instances construction that is similar to \cite{rashidinejad2021bridging}. However, we need to incorporate the absorbing state $g$ and assign the transition of initial state $s_{\mathrm{init}}$ carefully to make sure the optimal proper policy exists. 

\begin{proof}[Proof of Theorem~\ref{thm:lower}]
We create hard instances of SSPs as follows: we split $S-1$ states (except $s_{\mathrm{init}}$) into $S'=(S-1)/2$ groups, and denote it as $\mathcal{S}=\{s_{\mathrm{init}}\}\cup \{s_1^j,s_+^j\}_{j=1}^{S'}$. For $s_1^j$, $j=1,\ldots,S'$, there are two actions $a_1,a_2$ and for states $s_{\mathrm{init}}$, $s_+^j$ and goal state $g$ there is only one default action $a_d$ (therefore the only choice is always optimal for those states). Concretely,
\begin{itemize}
            \item For state $s_{\mathrm{init}}$, it transitions to $s^j_1$ ($j=1,\ldots,S'$) uniformly with probability $1/S'$, \emph{i.e.} $P(s_1^j|s_{\mathrm{init}},a_d)=1/S'$;
            \item For each state $s_1^j$, it satisfies
            \[
                P(s_+^j|s_1^j,a_1)=P(g|s_1^j,a_1)=1/2;\; P(s_+^j|s_1^j,a_2)=\frac{1}{2}+v_j\delta;\; P(g|s_1^j,a_2)=\frac{1}{2}-v_j\delta.       
            \]
            where $v_j\in\{+1,-1\}$ and $\delta$ to be specified later.
            \item For $s_+^j$, it satisfies 
            \[
            P(s_+^j|s_+^j,a_d)=q,P(g|s_+^j,a_d)=1-q,
            \]
            where $q=1-\frac{1}{B_\star}$ and $g$ is absorbing.
            \item the cost function satisfies (regradless of actions):
            \[
            c(s_{\mathrm{init}})=c(s_1^j)=c(s_+^j)=1, c(g)=0.
            \]
            It is easy to check this is a SSP. Moreover, it is clear when $v_j=1$, the optimal action at $s^j_1$ is $a_1$ and if $v_j=-1$ the optimal action is $a_2$. Note by straightforward calculation we have that $\norm{V^\star}_\infty\leq 2 B_\star $.

        \end{itemize}
        
        We consider the family of SSP instances $\mathcal{P}$ to satisfy Lemma~\ref{lem:GV}, i.e. it satisfies $|\mathcal{P}|\geq e^{S'/8}$ and for any two instances in $\mathcal{P}$, $\norm{v_i-v_j}_1\geq S'/2$. Also, it suffices to consider all the deterministic learning algorithms, as stochastic output policies are randomized versions over deterministic ones (c.f. \cite{krishnamurthy2016pac}). Then we have the following lemma:
        \begin{lemma}\label{lem:diff}
        For any (deterministic) policy $\pi$ and any two different transition probabilities $P_1,P_2\in\mathcal{P}$, it holds:
        \[
        V^\pi_{P_1}(s_{\mathrm{init}})-V^{\pi^\star}_{P_1}(s_{\mathrm{init}})+V^\pi_{P_2}(s_{\mathrm{init}})-V^{\pi^\star}_{P_2}(s_{\mathrm{init}})\geq \delta B_\star /2.
        \]
        \end{lemma}
        \begin{proof}[Proof of Lemma~\ref{lem:diff}]
        Since $s_{\mathrm{init}}$ uniformly transitions to $S'$ states (w.r.t the default action $a_d$), therefore for any policy $\pi$,
        \[
        V^\pi_{P_1}(s_{\mathrm{init}})-V^{\pi^\star}_{P_1}(s_{\mathrm{init}})=1+\frac{1}{S'}\sum_{i=1}^{S'}V^\pi_{P_1}(s_1^i)-\left(1+\frac{1}{S'}\sum_{i=1}^{S'}V^{\pi^\star}_{P_1}(s_1^i)\right)=\frac{1}{S'}\sum_{i=1}^{S'}(V^\pi_{P_1}(s_1^i)-V^{\pi^\star}_{P_1}(s_1^i)).
        \]
        \textbf{Case 1.} If $v_j=1$, then 
        \[
        P(s_+^i|s_1^i,a_2)=\frac{1}{2}+\delta,\;P(g|s_1^i,a_2)=\frac{1}{2}-\delta
        \]
        and in this case $\pi^\star(s_1^i)=a_1$. 
        
        If $\pi(s_1^i)=a_2$, then 
        \[
        V^\pi(s_1^i)=1+(\frac{1}{2}+\delta)V^\pi(s_+^i)+(\frac{1}{2}-\delta)V^\pi(g)=1+(\frac{1}{2}+\delta)V^\pi(s_+^i),
        \]
        and this implies
        \begin{align*}
        V^\pi(s_1^i)-V^{\pi^\star}(s_1^i)&=(\frac{1}{2}+\delta)V^\pi(s_+^i)-\frac{1}{2}V^{\pi^\star}(s^i_+)\\
        &\geq \delta V^\pi(s_+^i) = \delta (1+q+q^2+\ldots)=\delta\cdot\frac{1}{1-q}=\delta B_\star .
        \end{align*}
        If $\pi(s_1^i)=a_1$, then $V^\pi(s_1^i)-V^{\pi^\star}(s_1^i)\geq 0$. Therefore, in this case, one has
        \[
        V^\pi(s_1^i)-V^{\pi^\star}(s_1^i)\geq\delta B_\star \cdot\mathbf{1}[\pi(s_1^i)\neq \pi^\star(s_1^i)].
        \]
        \textbf{Case 2.} If $v_j=-1$, then 
        \[
        P(s_+^i|s_1^i,a_2)=\frac{1}{2}-\delta,\;P(g|s_1^i,a_2)=\frac{1}{2}+\delta
        \]
        and in this case $\pi^\star(s_1^i)=a_2$. 
        
        If $\pi(s_1^i)=a_1$, then 
        \[
        V^\pi(s_1^i)=1+\frac{1}{2}\cdot V^\pi(s_+^i)+\frac{1}{2}V^\pi(g)=1+\frac{1}{2}V^\pi(s_+^i),
        \]
        and this implies
        \begin{align*}
        V^\pi(s_1^i)-V^{\pi^\star}(s_1^i)&=\frac{1}{2}V^\pi(s_+^i)-(\frac{1}{2}-\delta)V^{\pi^\star}(s^i_+)\\
        &\geq \delta V^{\pi^\star}(s_+^i) = \delta (1+q+q^2+\ldots)=\delta\cdot\frac{1}{1-q}=\delta B_\star .
        \end{align*}

        If $\pi(s_1^i)=a_2$, then $V^\pi(s_1^i)-V^{\pi^\star}(s_1^i)\geq 0$. Therefore, in this case, we still have
        \[
        V^\pi(s_1^i)-V^{\pi^\star}(s_1^i)\geq\delta B_\star \cdot\mathbf{1}[\pi(s_1^i)\neq \pi^\star(s_1^i)].
        \]
        
        Combine the above two cases, we have
        \begin{equation}
        \label{eqn:l1_diff}
        \begin{aligned}
            &V^\pi_{P_1}(s_{\mathrm{init}})-V^{\pi^\star}_{P_1}(s_{\mathrm{init}})+V^\pi_{P_2}(s_{\mathrm{init}})-V^{\pi^\star}_{P_2}(s_{\mathrm{init}})\\
            \geq & \frac{1}{S'}\sum_{i=1}^{S'}(V^\pi_{P_1}(s_1^i)-V^{\pi^\star}_{P_1}(s_1^i))+\frac{1}{S'}\sum_{i=1}^{S'}(V^\pi_{P_2}(s_1^i)-V^{\pi^\star}_{P_2}(s_1^i))\\
            \geq & \frac{1}{S'}\delta B_\star \sum_{i=1}^{S'}\left(\mathbf{1}[\pi(s_1^i)\neq \pi^\star_{P_1}(s_1^i)]+\mathbf{1}[\pi(s_1^i)\neq \pi^\star_{P_2}(s_1^i)]\right)\\
            \geq & \frac{\delta B_\star }{S'}\sum_{i=1}^{S'}\mathbf{1}[\pi^\star_{P_1}(s_1^i)\neq \pi^\star_{P_2}(s_1^i)]
        \end{aligned}
        \end{equation}
        
        Lastly, by Lemma~\ref{lem:GV}, $\sum_{i=1}^{S'}\mathbf{1}[\pi^\star_{P_1}(s_1^i)\neq \pi^\star_{P_2}(s_1^i)]=\norm{v_{P_1}-v_{P_2}}_1\geq S'/2$, and plug this back to \eqref{eqn:l1_diff} we obtain the result.

        \end{proof}

Now we construct the behavior policy $\mu$ such that the data trajectories generated from the induced distribution $\mu\circ P$ suffice for the lower bound. Since only $s_1^i$ has two actions, we specify below:
\[
\mu(a_2|s_1^i)=1/C^\star,\;\mu(a_1|s_1^i)=1-1/C^\star,\;\;\forall i\in\{1,\ldots,S'\}.
\]

First, we examine this choice belongs to SSP$(C^\star)$. Indeed, the only case where $a_2$ is the suboptimal action for all $s_1^i$ ($i\in\{1,\ldots,S'\}$) is when $v_1,\ldots,v_{S'}$ all equal $1$. We can eliminate this SSP from $\mathcal{P}$ and the property of Lemma~\ref{lem:GV} still holds. Then, for some $i_0$ such that $v_{i_0}=-1$ ($a_2$ is the optimal action for this state), we have 
\[
d^\star(s_1^{i_0},a_2)=d^\star(s_1^{i_0})\cdot 1=\frac{1}{S'},\quad d^\mu(s_1^{i_0},a_2)=d^\mu(s_1^{i_0})\mu(a_2|s_1^{i_0})=\frac{1}{S'C^\star},
\]
therefore $d^\star(s_1^{i_0},a_2)/d^\mu(s_1^{i_0},a_2)=C^\star$ and this $(s_{\mathrm{init}},\mu,P,c)\in$SSP$(C^\star)$. 

Recall $n$ is the number of episodes. Now apply Fano's inequality (Lemma~\ref{lem:gen_Fano}) (where each whole trajectory is considered one single data point over the distribution $\mu\circ P$ therefore $\mathcal{D}:=\{(s^{(i)}_1,a^{(i)}_1,c^{(i)}_1,s^{(i)}_2,\ldots,s^{(i)}_{T_i})\}_{i=1,\ldots,n}$ consists of $n$ i.i.d. samples) and Lemma~\ref{lem:diff}, we have\footnote{Note here we drop $\widehat{\pi}$ is proper as the theorem statement did. We can do this since, for all the instances in $\mathcal{P}$, any policy is proper.} 
\[
\inf_{\widehat{\pi} }\sup_{(s_{\mathrm{init}},\mu,P,c)\in\mathcal{P}}\E_{\mathcal{D}}[V^{\widehat{\pi}}(s_{\mathrm{init}})-V^\star(s_{\mathrm{init}})]]\geq \frac{\delta B_\star }{4}\left(1-\frac{n\cdot \max_{i\neq j}\mathrm{KL}(\mu\circ P_i||\mu\circ P_j)+\log 2}{\log|\mathcal{P}|}\right).
\]

Note by the choice of $\mathcal{P}$, $\log|\mathcal{P}|\geq S'/2$, therefore it remains to bound $\max_{i\neq j}\mathrm{KL}(\mu\circ P_i||\mu\circ P_j)$. By definition, we have 
\begin{align*}
    \mathrm{KL}(\mu\circ P_1||\mu\circ P_2)=\frac{1}{S'}\sum_{i=1}^{S'}\sum_{\tau_{s_1^i}}\P_1(\tau_{s_1^i})\log\frac{\P_1(\tau_{s_1^i})}{\P_2(\tau_{s_1^i})},
\end{align*}
where $\tau_{s_1^i}$ corresponds to all the possible trajectories starting from $s_1^i$. Then there are the following several cases:\footnote{We omit the subscript $j$ in $P_j$ here and only uses $\P$ to denote $\mu\circ P$ for the moment.}
\begin{itemize}
    \item If $\tau_{s_1^i}=\{s_1^i\rightarrow a_1\rightarrow g\}$, then $\P(s_1^i\rightarrow a_1\rightarrow g)=(1-\frac{1}{C^\star})\frac{1}{2}$; 
    \item If $\tau_{s_1^i}=\{s_1^i\rightarrow a_2\rightarrow g\}$, then $\P(s_1^i\rightarrow a_2\rightarrow g)=\frac{1}{C^\star}\cdot(\frac{1}{2}-v_i\delta)$; 
    \item If $\tau_{s_1^i}=\{s_1^i\rightarrow a_1\rightarrow s_+^i\rightarrow g\}$, then $\P(s_1^i\rightarrow a_1\rightarrow s^i_+\rightarrow g)=(1-\frac{1}{C^\star})\frac{1}{2}(1-q)$; 
    \item If $\tau_{s_1^i}=\{s_1^i\rightarrow a_2\rightarrow s_+^i\rightarrow g\}$, then $\P(s_1^i\rightarrow a_2\rightarrow s^i_+\rightarrow g)=\frac{1}{C^\star}\cdot(\frac{1}{2}+v_i\delta)(1-q)$; 
    \item If $\tau_{s_1^i}=\{s_1^i\rightarrow a_1\rightarrow s_+^i\rightarrow s_+^i\rightarrow g\}$, then $\P(s_1^i\rightarrow a_1\rightarrow s^i_+\rightarrow s^i_+\rightarrow g)=(1-\frac{1}{C^\star})\frac{1}{2}q(1-q)$; 
    \item If $\tau_{s_1^i}=\{s_1^i\rightarrow a_2\rightarrow s_+^i\rightarrow s_+^i\rightarrow g\}$, then $\P(s_1^i\rightarrow a_2\rightarrow s^i_+\rightarrow s^i_+\rightarrow g)=\frac{1}{C^\star}(\frac{1}{2}+v_i\delta)q(1-q)$; 
    \item If $\tau_{s_1^i}=\{s_1^i\rightarrow a_1\rightarrow  (s_+^i)_{\times k}\rightarrow g\}$, then $\P(s_1^i\rightarrow a_1\rightarrow (s_+^i)_{\times k}\rightarrow g)=(1-\frac{1}{C^\star})\frac{1}{2}q^{k-1}(1-q)$; 
    \item If $\tau_{s_1^i}=\{s_1^i\rightarrow a_2\rightarrow  (s_+^i)_{\times k}\rightarrow g\}$, then $\P(s_1^i\rightarrow a_2\rightarrow (s_+^i)_{\times k}\rightarrow g)=\frac{1}{C^\star}(\frac{1}{2}+v_i\delta)q^{k-1}(1-q)$; 
\end{itemize}
Note for path $\tau_{s^i_1}$ that chooses action $a_1$, $\P_1(\tau_{s_1^i})=\P_2(\tau_{s_1^i})$ which implies $\P_1(\tau_{s_1^i})\log\frac{\P_1(\tau_{s_1^i})}{\P_2(\tau_{s_1^i})}=0$, so we only need to sum over the paths that choose $a_2$. In particular, we have
\begin{align*}
    &\sum_{\tau_{s_1^i}}\P_1(\tau_{s_1^i})\log\frac{\P_1(\tau_{s_1^i})}{\P_2(\tau_{s_1^i})}=\frac{1}{C^\star}\cdot(\frac{1}{2}-v^{P_1}_i\delta)\log\frac{\frac{1}{C^\star}\cdot(\frac{1}{2}-v^{P_1}_i\delta)}{\frac{1}{C^\star}\cdot(\frac{1}{2}-v^{P_2}_i\delta)}\\
    +&\sum_{k=0}^\infty \frac{1}{C^\star}(\frac{1}{2}+v^{P_1}_i\delta)q^{k-1}(1-q)\log\frac{\frac{1}{C^\star}(\frac{1}{2}+v^{P_1}_i\delta)q^{k-1}(1-q)}{\frac{1}{C^\star}(\frac{1}{2}+v^{P_2}_i\delta)q^{k-1}(1-q)}\\
    =&\frac{1}{C^\star}\cdot(\frac{1}{2}-v^{P_1}_i\delta)\log\frac{\frac{1}{2}-v^{P_1}_i\delta}{\frac{1}{2}-v^{P_2}_i\delta}+\frac{1}{C^\star}\cdot(\frac{1}{2}+v^{P_1}_i\delta)\log\frac{\frac{1}{2}+v^{P_1}_i\delta}{\frac{1}{2}+v^{P_2}_i\delta}\\
    \leq & \frac{1}{C^\star}\cdot(\frac{1}{2}-v^{P_1}_i\delta)\log\frac{\frac{1}{2}-v^{P_1}_i\delta}{\frac{1}{2}+v^{P_1}_i\delta}+\frac{1}{C^\star}\cdot(\frac{1}{2}+v^{P_1}_i\delta)\log\frac{\frac{1}{2}+v^{P_1}_i\delta}{\frac{1}{2}-v^{P_1}_i\delta}\\
    =&\frac{1}{C^\star}2\delta\log\frac{\frac{1}{2}+\delta}{\frac{1}{2}-\delta}=\frac{1}{C^\star}2\delta\log(1+\frac{2\delta}{\frac{1}{2}-\delta})\leq \frac{4\delta^2}{C^\star},
\end{align*}
where the first inequality comes from when $v^{P_1}_i=v^{P_2}_i$ then the term is simply $0$ and the second to the last equality holds true regardless of whether $v_i=1$ or $v_i=-1$. The last inequality comes from $log(1+x)\leq x$ for all $x>-1$ (here $0<\delta<\frac{1}{2}$).

Plug above back into the definition we obtain 
\[
\max_{i\neq j}\mathrm{KL}(\mu\circ P_i||\mu\circ P_j)\leq 4\delta^2/C^\star,
\]
and as long as 
\[
\frac{4n\delta^2}{C^\star S'/2}\leq \frac{1}{2}
\]
e.g. if we choose $\delta=\sqrt{\frac{C^\star S}{16n}}$ (recall $S'=(S-1)/2$), then we have 
\[
\inf_{\widehat{\pi} }\sup_{(s_{\mathrm{init}},\mu,P,c)\in\mathcal{P}}\E_{\mathcal{D}}[V^{\widehat{\pi}}(s_{\mathrm{init}})-V^\star(s_{\mathrm{init}})]]\geq \frac{\delta B_\star }{4}\frac{1}{2}=\frac{1}{32}B_\star \sqrt{\frac{C^\star S}{n}}.
\]
This completes the proof.

\end{proof}

\section{Technical Lemmas}
\begin{lemma}[Gilbert-Varshamov]\label{lem:GV}
There exists a subset $\mathcal{V}$ of $\{-1,1\}^S$ such that 
\begin{itemize}
    \item $|\mathcal{V}|\geq 2^{S/8}$;
    \item for any two different $v_i,v_j\in\mathcal{V}$, it holds $\norm{v_i-v_j}_1\geq S/2$.
\end{itemize}
\end{lemma}

\begin{lemma}[Generalized Chernoff bound]\label{lem:chernoff_multiplicative}
	Suppose $X_1,\ldots,X_n$ are independent random variables taking values in $[a,b]$. Let $X=\sum_{i=1}^nX_i$ denote their sum and let $\mu = E[X_i]$. Then for any $\delta>0$,
	$$
	\mathbb{P}[X<(1-\theta) n\mu]\leq e^{-{2\theta^{2} n\mu^2/(b-a)^2}} \quad \text { and } \quad \mathbb{P}[X \geq(1+\theta) p n]\leq e^{-{2\theta^{2} n\mu^2/(b-a)^2}} .
	$$
	This result can be found in Sums of independent bounded random variables Section of \url{https://en.wikipedia.org/wiki/Chernoff_bound}.
\end{lemma}

\begin{lemma}[Bernstein’s Inequality]\label{lem:bernstein_ineq}
	Let $x_1,...,x_n$ be independent bounded random variables such that $\E[x_i]=0$ and $|x_i|\leq \xi$ with probability $1$. Let $\sigma^2 = \frac{1}{n}\sum_{i=1}^n \mathrm{Var}[x_i]$, then with probability $1-\delta$ we have 
	$$
	\frac{1}{n}\sum_{i=1}^n x_i\leq \sqrt{\frac{2\sigma^2\cdot\log(1/\delta)}{n}}+\frac{2\xi}{3n}\log(1/\delta)
	$$
\end{lemma}

\begin{lemma}[Empirical Bernstein’s Inequality \citep{maurer2009empirical}]\label{lem:empirical_bernstein_ineq}
	Let $x_1,...,x_n$ be i.i.d random variables such that $|x_i|\leq \xi$ with probability $1$. Let $\bar{x}=\frac{1}{n}\sum_{i=1}^nx_i$ and $\widehat{V}_n=\frac{1}{n}\sum_{i=1}^n(x_i-\bar{x})^2$, then with probability $1-\delta$ we have 
	$$
	\left|\frac{1}{n}\sum_{i=1}^n x_i-\E[x]\right|\leq \sqrt{\frac{2\widehat{V}_n\cdot\log(2/\delta)}{n}}+\frac{7\xi}{3n}\log(2/\delta).
	$$
\end{lemma}

\begin{lemma}[Generalized Fano's inequality]\label{lem:gen_Fano}
Let $L:\Theta\times\mathcal{A}\rightarrow \R_+$ be any loss function, and there exist $\theta_1,\ldots,\theta_m\in\Theta$ such that
\[
L(\theta_i,a)+L(\theta_j,a)\geq \Delta,\quad \forall i\neq j\in[m],a\in\mathcal{A}.
\]
Then it holds 
\[
\inf_{\widehat{\theta}}\sup_{\theta\in\Theta}\E_\theta L(\theta,\widehat{\theta})\geq \frac{\Delta}{2}\left(1-\frac{n\cdot\max_{i\neq j} \mathrm{KL}(\P_{\theta_i}||\P_{\theta_j})+\log2}{\log m}\right),
\]
where $n$ is the number of i.i.d. samples sampled from the distribution $\P_\theta$.
\end{lemma}

\begin{proof}[Proof of Lemma~\ref{lem:gen_Fano}]
The proof come from the combination of Lemma~1 and Lemma~3 of \cite{han}.
\end{proof}

\begin{lemma}[Chernoff Bound for Stochastic Shortest Path]\label{lem:Chern_SSP}
	Recall by definition
	\[
	n(s,a)=\sum_{i=1}^n\sum_{h=1}^{T_i} \mathbf{1}[s_h^{(i)}=s, a_h^{(i)}=a]=\sum_{i=1}^n\sum_{h=1}^\infty \mathbf{1}[s_h^{(i)}=s, a_h^{(i)}=a].
	\]
	Let ${T}_{\max}=\max_i T_i$ and recall $d_m:=\min\{\sum_{h=0}^\infty \xi^\mu_h(s,a):s.t. \sum_{h=0}^\infty \xi^\mu_h(s,a)>0\}$. When $n>C\cdot {T}_{\max}^2\log(SA/\delta)/d_m^2$, with probability $1-\delta$, for all $s,a\in\mathcal{S}\times\mathcal{A}$,
	\[
	\frac{1}{2}n\cdot\sum_{h=1}^\infty \xi_h^\mu(s,a) \leq n(s,a)\leq  \frac{3}{2}n\cdot\sum_{h=1}^\infty \xi_h^\mu(s,a).
	\]
\end{lemma}

\begin{proof}[Proof of Lemma~\ref{lem:Chern_SSP}]
Indeed, denote $n_t(s,a)=\sum_{i=1}^n\sum_{h=1}^t \mathbf{1}[s_h^{(i)}=s, a_h^{(i)}=a]$, then 
\[
\E[n_t(s,a)]=\sum_{i=1}^n\sum_{h=1}^t\E[ \mathbf{1}[s_h^{(i)}=s, a_h^{(i)}=a]]=\sum_{i=1}^n\sum_{h=1}^t \xi_h^\mu(s,a)=n\sum_{h=1}^t \xi_h^\mu(s,a).
\]
Now define $X_{i,t}=\sum_{h=1}^t \mathbf{1}[s_h^{(i)}=s, a_h^{(i)}=a]$, then by ${T}_{\max}=\max_i T_i$ we have $0\leq X_{i,t}\leq {T}_{\max}$ for all $i,t$ since ${T}_{\max}$ denotes the maximum length of trajectory. Then apply Lemma~\ref{lem:chernoff_multiplicative} (where we pick $\theta=\frac{1}{2}$) to $n_t(s,a)$ and $\sum_{h=1}^t \xi_h^\mu(s,a)$ and union bound over $s,a$, we have with probability $1-\delta$, for any fixed $t$,
\begin{align*}
	 \P\left[\frac{1}{2}n\cdot \sum_{h=1}^t \xi_h^\mu(s,a) \leq n_t(s,a)\leq  \frac{3}{2}n\cdot \sum_{h=1}^t \xi_h^\mu(s,a),\forall s,a\right]\geq 1-\delta
\end{align*}
Next note $n_t(s,a)\rightarrow n(s,a)$ almost surely, and $\sum_{h=1}^t \xi_h^\mu(s,a)\rightarrow \sum_{h=1}^\infty \xi_h^\mu(s,a)$ almost surely, and that a.s. convergence implies convergence in distribution, we have
\begin{align*}
 &\P\left[\frac{1}{2}n\cdot \sum_{h=1}^\infty \xi_h^\mu(s,a) \leq n(s,a)\leq  \frac{3}{2}n\cdot \sum_{h=1}^\infty\xi_h^\mu(s,a),\forall s,a\right]\\
 =&\lim_{t\rightarrow \infty}\P\left[\frac{1}{2}n\cdot \sum_{h=1}^t \xi_h^\mu(s,a) \leq n_t(s,a)\leq  \frac{3}{2}n\cdot \sum_{h=1}^t \xi_h^\mu(s,a),\forall s,a\right]\geq 1-\delta\\
\end{align*}
\end{proof}

\begin{lemma}
\label{lem:minAB}
    For any $a,b,c\in\R$, we have
    \begin{align}
        |\min\{a,b\}-\min\{a,c\}|\leq|b-c|.
    \end{align}
\end{lemma}
\begin{proof}
\begin{enumerate}
    \item Case I: $a\leq b$ and $a\leq c$, $|\min\{a,b\}-\min\{a,c\}|=0$.
    \item Case II: $a\geq b$ and $a\geq c$,  $|\min\{a,b\}-\min\{a,c\}|=|b-c|$.
    \item Case III: $b<a<c$ or $c<a<b$,  $|\min\{a,b\}-\min\{a,c\}|\leq\max\{|a-b|,|a-c|\}\leq|b-c|$.
\end{enumerate}
\end{proof}

\providecommand{\upGamma}{\Gamma}
\providecommand{\uppi}{\pi}

\end{document}